\documentclass[10pt,journal,compsoc]{IEEEtran}
\usepackage[nocompress]{cite}
\usepackage{xcolor}
\usepackage{amsmath}
\usepackage{amsfonts}
\usepackage{array}
\usepackage{dblfloatfix}
\usepackage{url}
\usepackage{tikz}
\usepackage{pgfplots}
\pgfplotsset{compat=1.17}
\usepackage{adjustbox}
\usepackage{bigstrut}
\usepackage{enumitem}
\usepackage{amsthm}

\usepackage{mathtools}
\definecolor{color-orange}{RGB}{230,90,50}
\definecolor{color-blue}{RGB}{39,108,158}
\definecolor{color-red}{RGB}{180,37,37}
\definecolor{color-pink}{RGB}{228,56,126}
\definecolor{color-yellow}{RGB}{230,167,87}
\definecolor{color-green}{RGB}{50,152,69}
\definecolor{color-purple}{RGB}{180,131,210}
\definecolor{color-cyan}{RGB}{60,180,180}

\definecolor{dsibcolor}{RGB}{252, 231, 178}
\definecolor{convcolor}{RGB}{233, 240, 222}
\usetikzlibrary{decorations.pathreplacing, arrows.meta, positioning, math, shapes.geometric, calc}
\usepackage{graphicx}
\usepackage{caption}
\usepackage{subcaption}
\usepackage{booktabs}
\usepackage{multirow}
\usepackage{wasysym}

\newtheorem{proposition}{Proposition}

\usepackage[hidelinks]{hyperref}

\begin{document}
\title{Principled Reflection Separation via Nonlinear Superposition and Feature Interaction}
\author{Qiming~Hu, Mingjia~Li, Yuntong~Li, and
        $\text{Xiaojie~Guo}^*$\thanks{All the authors are from College of Intelligence and Computing, Tianjin University, Tianjin 300350, China.\\ * Corresponding author.}
        \protect}

\markboth{}%
{Principled Reflection Separation via Nonlinear Superposition and Feature Interaction}
\IEEEtitleabstractindextext{%
\begin{abstract}

Single-image reflection separation is fundamentally challenged by the entanglement of transmission and reflection layers under complex image formation processes. Existing approaches largely rely on simplified assumptions or independent modeling, limiting their ability to handle real-world scenarios. In this work, we revisit the problem from a unified perspective and identify a key issue of existing approaches, \emph{i.e.}, the widely adopted linear composition model in the sRGB domain fails to capture the nonlinear coupling introduced by real-world image signal processing pipelines. To address this, we introduce a learnable nonlinear superposition model that more faithfully characterizes layer interactions and improves decomposition fidelity.
Building upon this formulation, we propose a generalized dual-stream interactive framework that explicitly models bidirectional dependencies between transmission and reflection through feature exchange. This framework unifies activation-, gating-, and attention-based interaction mechanisms, and is compatible with both CNN and Transformer backbones. Extensive experiments on diverse real-world benchmarks demonstrate that the proposed approach achieves superior performance with strong generalization capability. More importantly, our study reveals that \emph{reflection separation is not about undoing a linear mixture, but about learning nonlinear formation and interaction}, offering new insights into the design of principled image decomposition models. 
Code and models are publicly available at \url{https://mingcv.github.io/DIRS-Page}.

\end{abstract}

\begin{IEEEkeywords}
Reflection Separation, Nonlinear Superposition, Feature Interaction, Blind Source Separation.
\end{IEEEkeywords}}

\maketitle
\IEEEpeerreviewmaketitle

\IEEEraisesectionheading{\section{Introduction}\label{sec:introduction}}

\IEEEPARstart{R}eflections caused by light scattering and refraction on surfaces are fundamental to how both animals and imaging systems perceive the world \cite{born1965principles}, as shown in Fig.~\ref{fig:reflection_model}~(a). However, reflective surfaces are frequently transparent or semi-transparent. Consequently, the captured image $\mathbf{I}$ is a composite of two entangled components including a transmission layer $\mathbf{T}$ originating from the background and a reflection layer $\mathbf{R}$ formed by light bouncing off the surface~\cite{cvpr/FaridA99,schechner2000polarization}, as illustrated in Fig.~\ref{fig:reflection_model}~(b). This entanglement obscures the structure and semantics of each component, likely degrading visual quality and hindering downstream vision tasks~\cite{eccv/shizawa1992visual,cvpr/SimonP15,cvpr/qiu2023looking}. While specialized hardware such as polarization filters can partially mitigate reflections as shown in Fig.~\ref{fig:reflection_model}~(c), such solutions are often impractical for amateur use. This work concentrates on equipment-free single-image reflection separation.

Recovering transmission and reflection from a single observation is inherently ill-posed, as infinitely many decompositions can explain the same image. A widely-adopted assumption models the observed image as a linear combination of the two layers~\cite{cvpr/FaridA99}, which approximately holds in the RAW domain as 
    $\textbf{I}_\text{raw} = \textbf{T}_\text{raw} + \textbf{R}_\text{raw}$,
where $\textbf{I}_\text{raw}$, $\textbf{T}_\text{raw}$, and $\textbf{R}_\text{raw}$ denote the composite image, transmission layer, and reflection layer in the RAW sensor space, respectively. However, this assumption breaks down in the sRGB domain due to the nonlinear transformations introduced by the image signal processing (ISP) mapping (\emph{e.g.}, gamma correction and color mapping) $\mathcal{F}_{\text{ISP}}$~\cite{cvpr/WenT0LHH19}:
\begin{equation}
\begin{aligned}
    \mathbf{I} &= \mathcal{F}_{\text{ISP}}(\mathbf{T}_{\text{raw}} + \mathbf{R}_{\text{raw}}) \\
    &\neq \mathcal{F}_{\text{ISP}}(\mathbf{T}_{\text{raw}}) + \mathcal{F}_{\text{ISP}}(\mathbf{R}_{\text{raw}}) = \mathbf{T} + \mathbf{R}.
\end{aligned}
\end{equation}

This reveals an inherent mismatch between the commonly assumed linear superposition model and the actual image formation process in sRGB space. Consequently, methods relying on linear models in sRGB \cite{iccv/FanYHCW17, cvpr/ZhangNC18a, cvpr/WeiYFW019} often fail to generalize to real-world superimpositions.

\begin{figure*}[t]
    \centering
    \begin{subfigure}[b]{0.30\linewidth}
        \centering
        \includegraphics[width=\linewidth,height=100pt]{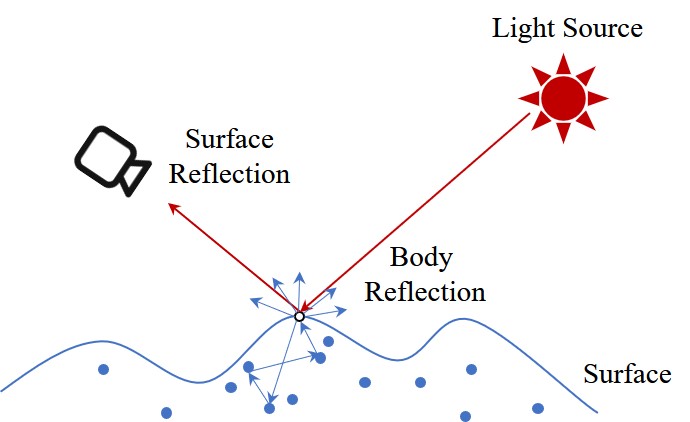}
        \vspace{1pt}
        \caption{Reflection at an opaque surface}
    \end{subfigure}
    \hfill
    \begin{subfigure}[b]{0.35\linewidth}
        \centering
        \includegraphics[width=\linewidth,height=100pt]{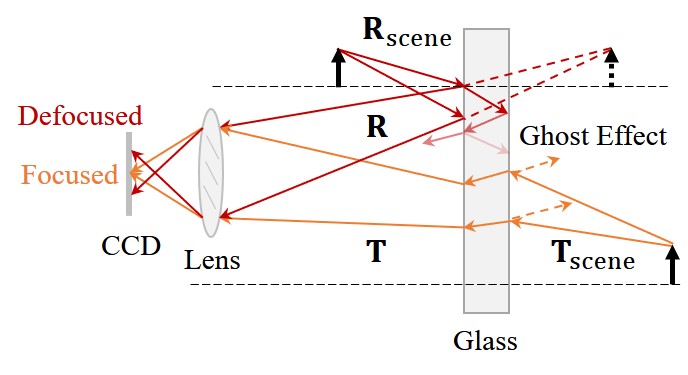}
        \caption{Reflection at a transparent surface}
    \end{subfigure}
    \hfill
    \begin{subfigure}[b]{0.28\linewidth}
        \centering
        \includegraphics[width=\linewidth,height=100pt]{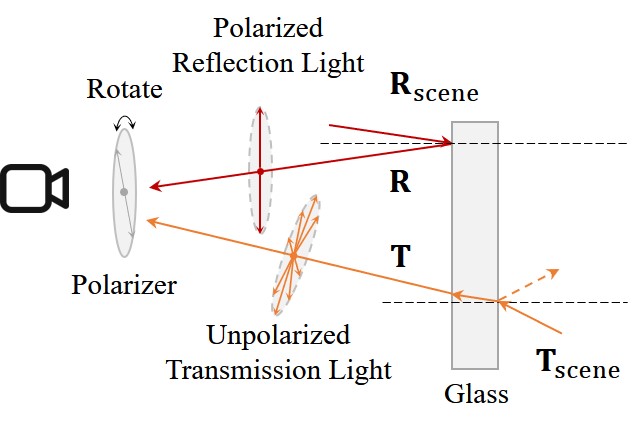}
        \caption{Polarization phenomenon}
    \end{subfigure}
    \caption{(a) Reflection of light from an opaque inhomogeneous surface. (b) Reflection superimposition in transparent media (such as glass and water surfaces) with relative smoothness and ghost effect. (c) Polarization effect of reflected light.}
    \label{fig:reflection_model}

    \centering
    \begin{subfigure}[b]{0.23\linewidth}
        \centering
        \includegraphics[width=\linewidth]{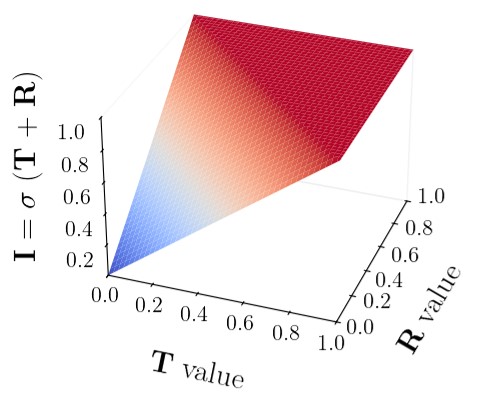}
        \caption{Linear model}
    \end{subfigure}
    \hfill
    \begin{subfigure}[b]{0.23\linewidth}
        \centering
        \includegraphics[width=\linewidth]{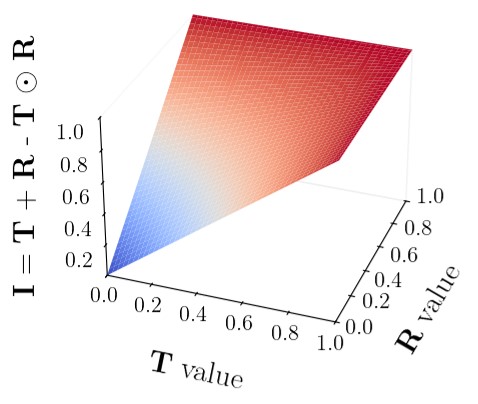}
        \caption{Screen model}
    \end{subfigure}
    \hfill
    \begin{subfigure}[b]{0.23\linewidth}
        \centering
        \includegraphics[width=\linewidth]{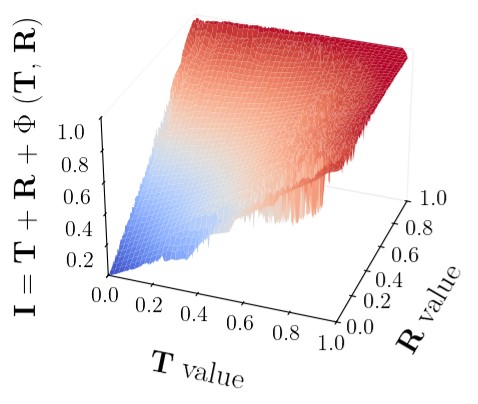}
        \caption{\makebox[100pt][l]{Learnable non-linear model}}
    \end{subfigure}
    \begin{subfigure}[b]{0.26\linewidth}
        \centering
        \includegraphics[width=\linewidth]{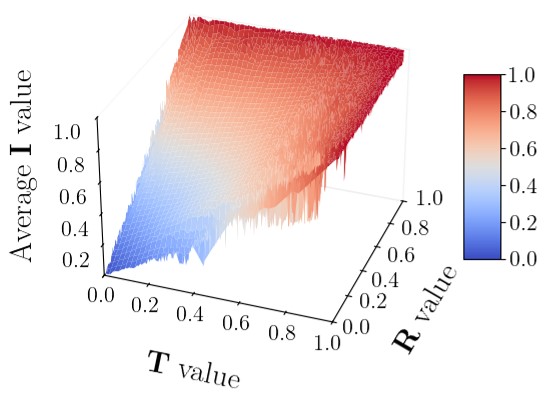}
        \caption{Ground-truth $\textbf{I}$}
    \end{subfigure}
    
    \caption{Illustration of how pixel values in $\mathbf{I}$ vary with $\mathbf{T}$ and $\mathbf{R}$ across different models: (a) a linear model with truncation $\sigma$~\cite{iccv/FanYHCW17}, (b) the screen blending model, (c) our high-order learnable non-linear model, and (d) real-world triplets from the $\text{SIR}^2$ dataset. As can be seen, reflective coupling demands higher-order inter-layer interactions beyond linear models.}
    \label{fig:prob}

    \centering
    \begin{subfigure}[b]{0.155\linewidth}
        \centering
        \includegraphics[width=\linewidth,height=75pt]{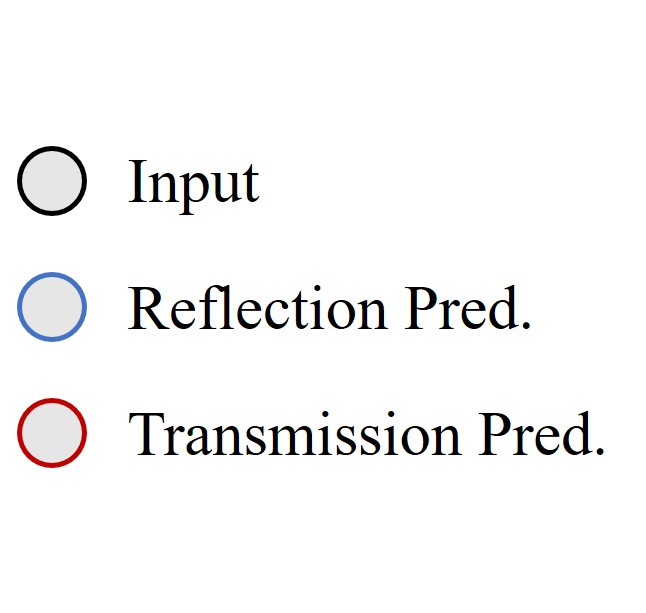}
        \caption*{}
    \end{subfigure}
    \hfill
    \begin{subfigure}[b]{0.19\linewidth}
        \centering
        \includegraphics[width=\linewidth,height=75pt]{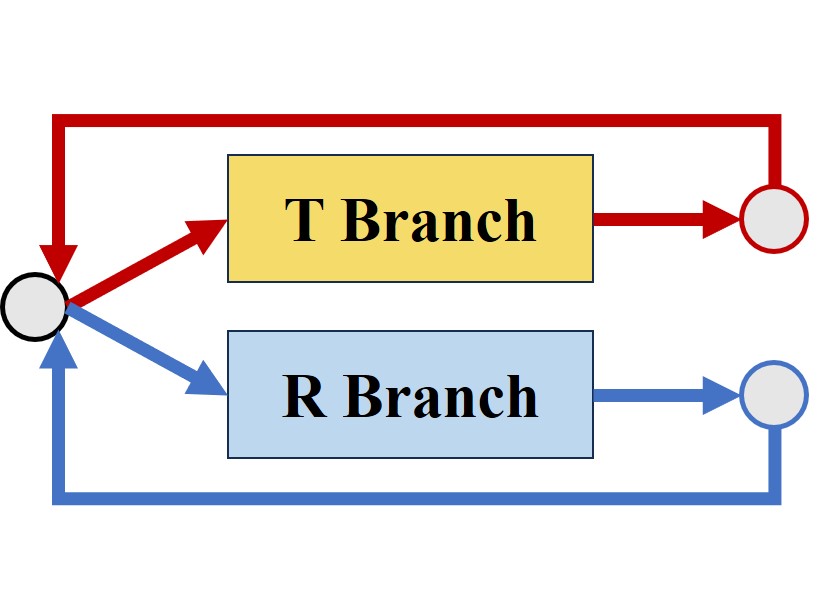}
        \caption{IBCLN \cite{cvpr/LiY0LH20}}
    \end{subfigure}
    \hfill
    \begin{subfigure}[b]{0.27\linewidth}
        \centering
        \includegraphics[width=\linewidth,height=75pt]{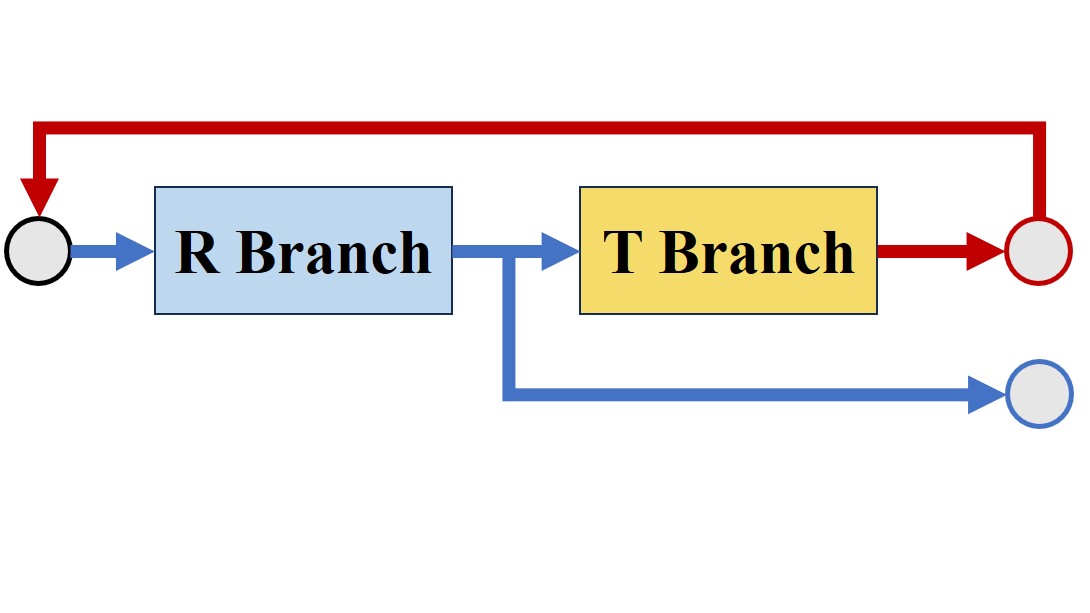}
        \caption{Dong \emph{et al.} \cite{iccv/Dong00BXL21}}
    \end{subfigure}
    \hfill
    \begin{subfigure}[b]{0.33\linewidth}
        \centering
        \includegraphics[width=\linewidth,height=75pt]{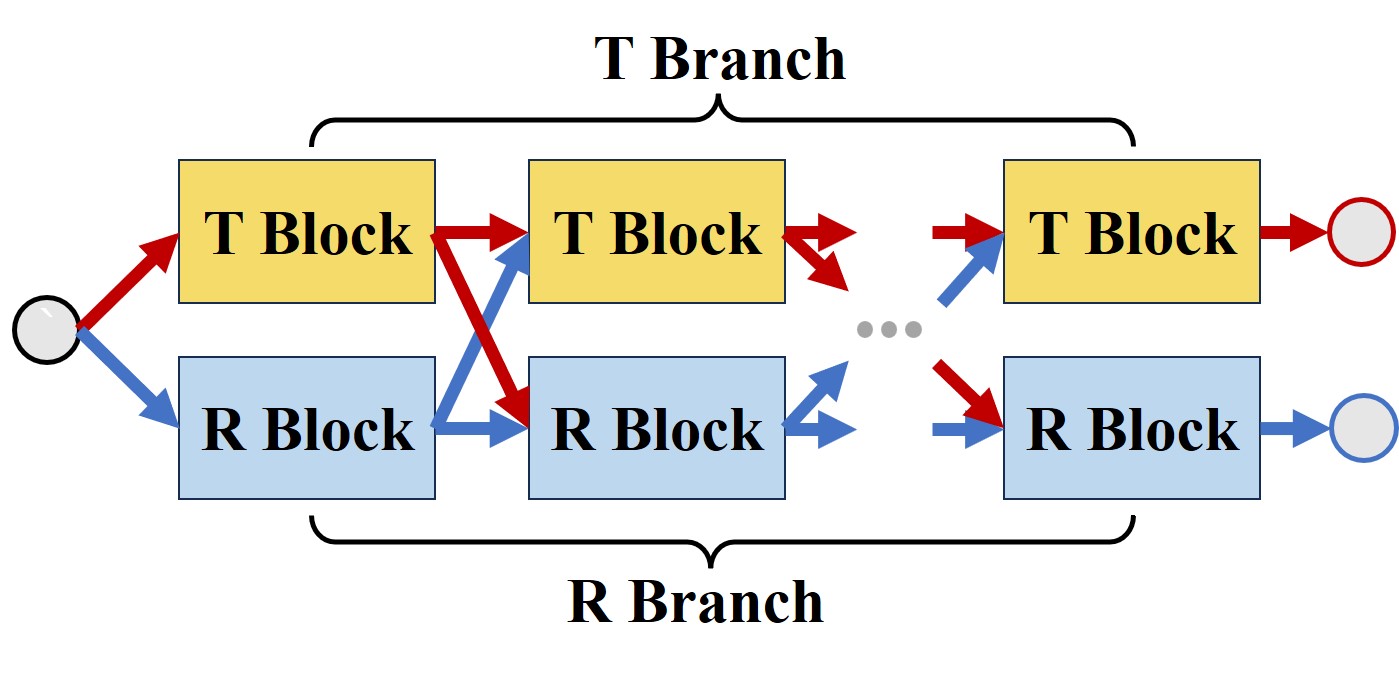}
        \caption{Ours}
    \end{subfigure}
   
    \caption{Comparison of our dual-stream interaction mechanism (c) with previous dual-branch approaches (a) and (b).} 
    \label{fig:head}
    \vspace{-5pt}
\end{figure*}

To further analyze this discrepancy, we examine the intensity distributions of ground-truth $\mathbf{I}$, $\mathbf{T}$, and $\mathbf{R}$ triplets from the real-world SIR$^2$ dataset \cite{iccv/WanSDTK17} in a 3D sRGB space (detailed in Sec.~\ref{sec:modeling}). As illustrated in Fig.~\ref{fig:prob}~(d), while $\mathbf{I}$ exhibits approximately linear behavior at low intensities, it undergoes pronounced saturation compression as layer intensities increase. This observation indicates that the interaction between transmission and reflection deviates significantly from linearity in practical imaging conditions. 
Consequently, the traditional linear truncation model (Fig.~\ref{fig:prob}~(a)) fails to capture this nonlinear roll-off behavior. The screen blending model (Fig.~\ref{fig:prob}~(b)) provides a closer approximation:
\begin{equation}
\mathbf{I} = \mathbf{T} + \mathbf{R} - \mathbf{T}\circ\mathbf{R}, \label{eq:screen}
\end{equation}
where $\circ$ denotes the Hadamard product. However, such fixed analytical formulations remain insufficient to model the spatially varying and context-dependent interactions observed in real-world data. 
To address this limitation, we introduce a {learnable nonlinear superposition model} that augments the linear formulation by:
\begin{equation}
\mathbf{I} = \mathbf{T} + \mathbf{R} + \Phi(\mathbf{T},\mathbf{R}) + \Psi. \label{eq:model_nonlinear}
\end{equation}
Here, $\Phi(\cdot)$ captures high-order, content-adaptive coupling between the two layers, while $\Psi$ accounts for global offsets induced by the ISP pipeline. As shown in Fig.~\ref{fig:prob}~(c), this formulation significantly narrows the gap between the modeled and real-world distributions.

Beyond modeling limitations, the ill-posed nature of reflection separation also necessitates effective utilization of structural priors. Existing methods typically impose priors \cite{nips/LevinZW02, cvpr/LiB14, cvpr/ZhangNC18a, cvpr/WeiYFW019}, such as mutual independence and relative smoothness, to constrain the solution space. However, these priors implicitly assume that the two layers can be recovered in a largely decoupled manner, overlooking their intrinsic dependencies. In practice, transmission and reflection exhibit strong structural and statistical correlations, making independent modeling insufficient for accurate disentanglement.
We argue that effective reflection separation requires jointly addressing two fundamental challenges, \emph{i.e.}, {accurate modeling of nonlinear layer superposition} and {explicit interaction between decomposed representations}. While the former ensures faithful reconstruction under realistic imaging conditions, the latter is essential for resolving the intrinsic ambiguity of the decomposition.
To this end, we propose a {generalized dual-stream interactive framework} that explicitly models bidirectional dependencies between transmission and reflection through multi-scale, high-dimensional feature exchange. Unlike prior dual-branch designs that either lack explicit interaction \cite{iccv/Dong00BXL21} or rely on limited post-hoc fusion \cite{cvpr/LiY0LH20} as shown in Fig.~\ref{fig:head}~(a) and (b), our framework (Fig.~\ref{fig:head}~(c)) enables multiple rounds of interaction within each forward pass, significantly enhancing the depth and semantic richness of information exchange. In this manuscript, we consolidate our previous interactive designs into a unified generalized mechanism that encompasses activation-based CNNs \cite{nips/HuG21}, gate-based CNNs \cite{iccv/Hu23}, and attention-based Transformers \cite{nips/hu2024single}. We conduct in-depth architectural explorations and comprehensive ablation studies to substantiate these design principles. 

The main contributions of this work are summarized as:
\begin{enumerate}
    \item We propose a learnable nonlinear reflection formulation that captures interlayer coupling in sRGB space, mitigating the limitations of conventional linear assumptions and improving layer separation fidelity;
    
    \item We design a generalized dual-stream interactive framework that enables explicit bidirectional information exchange between transmission and reflection via a family of feature interaction mechanisms; 
    
    \item We demonstrate that jointly modeling nonlinear superposition and feature interaction leads to state-of-the-art performance and strong generalization, offering new insights into principled image decomposition. 
\end{enumerate}

\begin{figure*}[t]
\centering
\begin{tikzpicture}[timeline/.style={draw, thick, -{Latex[length=3mm,width=2mm]}}, 
                    polarc/.style={fill=white, draw=color-blue,very thick},
                    polarl/.style={draw=color-blue,very thick},
                    motionc/.style={fill=white, draw=color-red,very thick},
                    motionl/.style={draw=color-red,very thick},
                    stereoc/.style={fill=white, draw=color-yellow,very thick},
                    stereol/.style={draw=color-yellow,very thick},
                    linearc/.style={fill=white, draw=color-green,very thick},
                    linearl/.style={draw=color-green,very thick},
                    flashc/.style={fill=white, draw=color-purple,very thick},
                    flashl/.style={draw=color-purple,very thick},
                    democ/.style={fill=white,draw=black,very thick},
                    oppoc/.style={fill=white, draw=color-cyan,very thick},
                    oppol/.style={draw=color-cyan,very thick},
                    cascaded/.style={diamond,minimum width=0.275cm,minimum height=0.275cm,inner sep=0},
                    single/.style={isosceles triangle,rotate=90,isosceles triangle apex angle=60, minimum size =0.2cm,inner sep=0},
                    year/.style={draw, black, thick},
                    item/.style={align=center, text width=3cm, inner sep=5pt, anchor=north},
                    every node/.style={font=\footnotesize}]

  \tikzmath{\hi = 5.0;};
  
  \draw[thick] (-3,\hi) -- (13.5,\hi);
  \draw[thick,dashed,-{Latex[length=3mm,width=2mm]}] (13.5,\hi) -- (14.5,\hi);

  \draw[year] (-2.75, \hi+0.1) -- (-2.75, \hi-0.1);
  \node[above=3pt] at (-2.75, \hi-0.1) {1989};

  \draw[polarl] (-2.25, \hi) -- (-2.25, \hi-0.25);
  \draw[polarc] (-2.25, \hi) circle (3pt);
  \node[item] at (-2.25, \hi-0.2) (1989wolff) {\textbf{Wolff} \cite{cvpr/Wolff89a} \\ Linear Cluster in Polarization Space};

  \draw[year] (-1.75, \hi+0.1) -- (-1.75, \hi-0.1);
  \node[above=3pt] at (-1.75, \hi-0.1) {1990};

  \draw[motionl] (-1.25, \hi) -- (-1.25, \hi+0.25);
  \draw[motionc] (-1.25, \hi) circle (3pt);
  \node[item] at (-1.25, \hi+1.5) (1990bergen) {\textbf{Bergen \emph{et al.}} \cite{eccv/BergenBHP90} \\ Dual Components Motion Estimation};
  
  \draw[year] (-0.75, \hi+0.1) -- (-0.75, \hi-0.1);
  \node[above=3pt] at (-0.75, \hi-0.1) {1991};

  \draw[year] (0, \hi+0.1) -- (0, \hi-0.1);
  \node[above=3pt] at (0, \hi-0.1) {1992};

  \draw[motionl] (0.5, \hi) -- (0.5, \hi-0.25);
  \draw[motionc] (0.5, \hi) circle (3pt);
  \node[item] at (0.5, \hi-0.2) (1992bergen) {\textbf{Bergen \emph{et al.}} \cite{pami/BergenBHP92} \\ Three-Frame Layer Motion Analysis};

  \draw[year] (1.0, \hi+0.1) -- (1.0, \hi-0.1);
  \node[above=3pt] at (1.0, \hi-0.1) {1993};
  
  \node at (1.55, \hi+0.2) {\ldots};

  \draw[year] (2, \hi+0.1) -- (2, \hi-0.1);
  \node[above=3pt] at (2, \hi-0.1) {1997};

  \draw[polarl] (2.5, \hi) -- (2.5, \hi+0.25);
  \draw[polarc] (2.5, \hi) circle (3pt);
  \node[item] at (2.25,\hi+1.5) (1997nayar) {\textbf{Nayar \emph{et al.}} \cite{ijcv/NayarFB97} \\ Color and \\ Polarization};
  
  \draw[year] (3, \hi+0.1) -- (3, \hi-0.1);
  \node[above=3pt] at (3, \hi-0.1) {1998};

  \draw[stereol] (3.5, \hi) -- (3.5, \hi-0.25);
  \draw[stereoc] (3.5, \hi) circle (3pt);
  \node[item] at (3.5, \hi-0.2) (1998schechner) {\textbf{Schechner \emph{et al.}} \cite{iccv/SchechnerKB98} \\ Extended Depth from Focus};
  
  \draw[year] (4, \hi+0.1) -- (4, \hi-0.1);
  \node[above=3pt] at (4, \hi-0.1) {1999};

  \draw[polarl] (5, \hi) -- (5, \hi+0.25);
  \draw[polarc] (5, \hi) circle (3pt);
  \node[item] at (4.9,\hi+1.5) (1999farid) {\textbf{Farid and Adelson} \cite{cvpr/FaridA99} \\ ICA on Polarized Linear Mixing};

  \draw[polarl] (6.25, \hi) -- (6.25, \hi-0.25);
  \draw[polarc] (6.25, \hi) circle (3pt);
  \node[item,text width=4cm] at (6.25, \hi-0.2) (1999farid) {\textbf{Schechner \emph{et al.}} \cite{schechner1999polarization} \\ Polarization-based Decorrelation Search};
  
  \draw[year] (7.25, \hi+0.1) -- (7.25, \hi-0.1);
  \node[above=3pt] at (7.25, \hi-0.1) {2000};

  \draw[motionl] (7.875, \hi) -- (7.875, \hi+0.25);
  \draw[motionc] (7.875, \hi) circle (3pt);
  \node[item,text width=4cm] at (7.75, \hi+1.5) (2000szeliski) {\textbf{Szeliski \emph{et al.}} \cite{cvpr/SzeliskiAA00} \\ Iterative Layer Estimation with Motions};

  \draw[year] (8.5, \hi+0.1) -- (8.5, \hi-0.1);
  \node[above=3pt] at (8.5, \hi-0.1) {2001};

  \draw[year] (8.875, \hi+0.1) -- (8.875, \hi-0.1);
  \draw[year] (9.25, \hi+0.1) -- (9.25, \hi-0.1);
  \node[above=3pt] at (9.25, \hi-0.1) {2003};

  \draw[stereol] (9.75, \hi) -- (9.75, \hi-0.25);
  \draw[stereoc] (9.75, \hi) circle (3pt);
  \node[item] at (9.5, \hi-0.2) (2003tsin) {\textbf{Tsin \emph{et al.}} \cite{cvpr/TsinKS03} \\ Iterative Component Color Estimation};

  \draw[year] (10.25,\hi+0.1) -- (10.25, \hi-0.1);
  \node[above=3pt] at (10.25, \hi-0.1) {2004};

  \draw[linearl] (10.75, \hi) -- (10.75, \hi+0.25);
  \draw[linearc] (10.75, \hi) circle (3pt);
  \node[item,text width=4cm] at (10.5, \hi+1.5) (2004sarel) {\textbf{Sarel and Irani} \cite{eccv/SarelI04} \\ Layer Exchange via Local Decorrelation};

  \draw[flashl] (12.25, \hi) -- (12.25,\hi-0.25);
  \draw[flashc] (12.25, \hi) circle (3pt);
  \node[item,text width=4cm] at (12.25, \hi-0.2) (2005agrawal) {\textbf{Agrawal \emph{et al.}} \cite{tog/AgrawalRNL05} \\ Gradient Projection of No-/Flash Pairs};

  \draw[year] (11.5,\hi+0.1) -- (11.5, \hi-0.1);
  \node[above=3pt] at (11.5, \hi-0.1) {2005};
  
  \draw[motionl] (13.25, \hi) -- (13.25, \hi+0.25);
  \draw[motionc] (13.25, \hi) circle (3pt);
  \node[item,text width=3cm] at (13.25, \hi+1.5) (2005sarel) {\textbf{Sarel and Irani} \cite{iccv/SarelI05} \\ Repetitive Dynamic Separation};
  
  \tikzmath{\hii = 2;};

  \draw[thick,dashed] (-3,\hii) -- (-2.5,\hii);
  \draw[thick] (-2.5,\hii) -- (13.5,\hii);
  \draw[thick,dashed,-{Latex[length=3mm,width=2mm]}] (13.5,\hii) -- (14.5,\hii);

  \draw[year] (-2.5, \hii+0.1) -- (-2.5, \hii-0.1);
  \node[above=3pt] at (-2.5, \hii-0.1) {2008};

  \draw[motionl] (-2, \hii) -- (-2, \hii+0.25);
  \draw[motionc] (-2, \hii) circle (3pt);
  \node[item,text width=3cm] at (-2, \hii+1.5) (2008gai) {\textbf{Gai \emph{et al.}} \cite{cvpr/GaiSZ08} \\ Separation via Spatial Shift Clues};

  \draw[year] (-1.5, 2.1) -- (-1.5, 1.9);
  \node[above=3pt] at (-1.5, 1.9) {2009};

  \draw[year] (-1, \hii+0.1) -- (-1, \hii-0.1);

  \draw[year] (-0.5, \hii+0.1) -- (-0.5, \hii-0.1);
  \node[above=3pt] at (-0.5, \hii-0.1) {2011};
  
  \draw[polarl] (0, \hii) -- (0, \hii-0.25);
  \draw[polarc] (0, \hii) circle (3pt);
  \node[item,text width=5cm] at (0, \hii-0.2) (2011kong) {\textbf{Kong \emph{et al.}} \cite{tip/KongTS11} \\ Gradient Variance Prior for Polarization};

  \draw[year] (0.5, \hii+0.1) -- (0.5, \hii-0.1);
  \node[above=3pt] at (0.5, \hii-0.1) {2012};
  
  \draw[stereol] (1, \hii) -- (1, \hii+0.25);
  \draw[stereoc] (1, \hii) circle (3pt);
  \node[item] at (1, \hii+1.5) (2011sinha) {\textbf{Sinha \emph{et al.}} \cite{tog/SinhaKGSS12} \\ Extended Semi-Global Matching};

  \draw[year] (1.5, \hii+0.1) -- (1.5, \hii-0.1);
  \node[above=3pt] at (1.5, \hii-0.1) {2013};
  
  \draw[motionl] (2.375, \hii) -- (2.375,\hii-0.25);
  \draw[motionc] (2.375, \hii) circle (3pt);
  \node[item] at (2.75, \hii-0.2) (2013li) {\textbf{Li and Brown} \cite{iccv/LiB13} \\ Gradient Magnitude Prior for Motions};

  \draw[year] (3.25, \hii+0.1) -- (3.25, \hii-0.1);
  \node[above=3pt] at (3.25, \hii-0.1) {2014};
  
  \draw[linearl] (4, \hii) -- (4, \hii+0.25);
  \draw[linearc] (4, \hii) circle (3pt);
  \node[item,text width=4cm] at (4, \hii+1.5) (2014guo) {\textbf{Guo \emph{et al.}} \cite{cvpr/GuoCM14} \\ Low Rank Prior for Motion Sequences};

  \draw[polarl] (5.75, \hii) -- (5.75, \hii-0.25);
  \draw[polarc] (5.75, \hii) circle (3pt);
  \node[item,text width=4cm] at (5.75,\hii-0.2) (2014kong) {\textbf{Kong \emph{et al.}} \cite{pami/KongTS14} \\ Error Suppression and Reflection Refinement};

  \draw[year] (6.5, \hii+0.1) -- (6.5, \hii-0.1);
  \node[above=3pt] at (6.5, \hii-0.1) {2015};

  \draw[motionl] (7,  \hii) -- (7, \hii+0.25);
  \draw[motionc] (7,  \hii) circle (3pt);
  \node[item,text width=3.5cm] at (7, \hii+1.5) (2015simon) {\textbf{Xue \emph{et al.}} \cite{cvpr/SimonP15} \\ Gradient-based Layer and Motion Estimation};

  \draw[motionl] (8.75, \hii) -- (8.75, \hii-0.25);
  \draw[motionc] (8.75, \hii) circle (3pt);
  \node[item] at (8.75, \hii-0.2) (2015simon) {\textbf{Simon and Park} \cite{cvpr/SimonP15} \\ Average Prior for Dashboard Cameras};
   
  \draw[year] (9.25, \hii+0.1) -- (9.25, \hii-0.1);
  \node[above=3pt] at (9.25, \hii-0.1) {2016};

  \draw[motionl] (10, \hii) -- (10, \hii+0.25);
  \draw[motionc] (10, \hii) circle (3pt);
  \node[item,text width=3cm] at (10, \hii+1.5) (2016sun) {\textbf{Sun \emph{et al.}} \cite{mm/SunLYZWL16} \\ SIFT-Flow Motion and Intensity Priors};

  \draw[motionl] (11.5, \hii) -- (11.5, \hii-0.25);
  \draw[motionc] (11.5, \hii) circle (3pt);
  \node[item,text width=3cm] at (11.5, \hii-0.2) (2016yang) {\textbf{Yang \emph{et al.}} \cite{cvpr/YangLDT16} \\ Double Layer Bright-ness Constancy};

  \draw[year] (12, \hii-0.1) -- (12, \hii+0.1);
  \node[above=3pt] at (12, \hii-0.1) {2017};

  \draw[motionl] (12.75, \hii) -- (12.75, \hii+0.25);
  \draw[motionc] (12.75, \hii) circle (3pt);
  \node[item] at (12.75, \hii+1.5) (2017han) {\textbf{Han and Sim} \cite{cvpr/HanS17} \\ Gradient Low Rank Matrix Completion};
  \tikzmath{\hiii = -1;};
  
  \draw[thick,dashed] (-3, \hiii) -- (-2.5, \hiii);
  \draw[thick,-{Latex[length=3mm,width=2mm]}] (-2.5, \hiii) -- (14.5, \hiii);

  \draw[year] (-2.5, \hiii-0.1) -- (-2.5, \hiii+0.1);
  \node[above=3pt] at (-2.5, \hiii-0.1) {2018};

  \draw[motionl] (-2, \hiii) -- (-2, \hiii+0.25);
  \draw[motionc] (-2, \hiii) circle (3pt);
  \node[item,text width=3cm] at (-2, \hiii+1.5) (2018han) {\textbf{Han and Sim} \cite{tip/HanS18} \\ Extension of \cite{cvpr/HanS17} with Co-Saliency};

  \draw[year] (-1.5, \hiii-0.1) -- (-1.5, \hiii+0.1);
  \node[above=3pt] at (-1.5, \hiii-0.1) {2019};
  
  \draw[polarl] (-1, \hiii) -- (-1, \hiii-0.25);
  \node[polarc,cascaded] at (-1, \hiii) {};
  \node[item,text width=3cm] at (-1, \hiii-0.2) (2019lyu) {\textbf{Lyu \emph{et al.}} \cite{nips/LyuCLPS19} \\ Polarization Param-eter Aware Network};

  \draw[motionl] (0.5, \hiii) -- (0.5, \hiii+0.25);
  \node[motionc,single] at (0.5, \hiii) {};
  \node[item] at (0.5, \hiii+1.5) (2019alayrac) {\textbf{Alayrac \emph{et al.}} \cite{cvpr/AlayracCZ19} \\ Video Separation via a 3D ConvNet};

  \draw[motionl] (1.75, \hiii) -- (1.75, \hiii-0.25);
  \draw[motionc] (1.75, \hiii) circle (3pt);
  \node[item,text width=3.5cm] at (1.75, \hiii-0.2) (2019punnappurath) {\textbf{Pun. and Brown} \cite{cvpr/PunnappurathB19} \\ Dual Pixel Sensor \\ of Mobile Devices};

  \draw[year] (2.25, \hiii+0.1) -- (2.25, \hiii-0.1);
  \node[above=3pt] at (2.25, \hiii-0.1) {2020};
  
  \draw[polarl] (3.25, \hiii) -- (3.25, \hiii+0.25);
  \node[polarc,cascaded] at (3.25, \hiii) {};
  \node[item,text width=3.5cm] at (3.25, \hiii+1.5) (2020li) {\textbf{Li \emph{et al.}} \cite{eccv/LiQZH20} \\ Polarization Simulation Engine};

  \draw[motionl] (4.25, \hiii) -- (4.25, \hiii-0.25);
  \node[motionc,cascaded] at (4.25, \hiii) {};
  \node[item,text width=4cm] at (4.25, \hiii-0.2) (2020liu) {\textbf{Liu \emph{et al.}} \cite{cvpr/LiuL0CH20} \\ Joint Flow and Layer Estimation};

  \draw[polarl] (5.5, \hiii) -- (5.5, \hiii+0.25);
  \node[polarc,cascaded] at (5.5, \hiii) {};
  \node[item,text width=3cm] at (5.5, \hiii+1.5) (2020lei) {\textbf{Lei \emph{et al.}} \cite{cvpr/LeiHZYSC20} \\ Polarization of RAW Image};
  
  \draw[year] (6, \hiii-0.1) -- (6, \hiii+0.1);
  \node[above=3pt] at (6, \hiii-0.1) {2021};

  \draw[flashl] (6.75, \hiii) -- (6.75, \hiii-0.25);
  \node[flashc,cascaded] at (6.75, \hiii) {};
  \node[item,text width=4cm] at (6.875, \hiii-0.2) (2021lei) {\textbf{Lei and Chen} \cite{cvpr/LeiC21} \\ Networks based on Flash/No-Flash Pairs};

  \draw[year] (7.5, \hiii+0.1) -- (7.5, \hiii-0.1);
  \node[above=3pt] at (7.5, \hiii-0.1) {2022};

  \draw[motionl] (8, \hiii) -- (8, \hiii+0.25);
  \node[motionc,cascaded] at (8, \hiii) {};
  \node[item,text width=3cm] at (8, \hiii+1.5) (2022liu) {\textbf{Liu \emph{et al.}} \cite{pami/LiuLYCH22} \\  Meta Learning and Upgrades of \cite{cvpr/LiuL0CH20}};

  \draw[year] (8.5, \hiii+0.1) -- (8.5, \hiii-0.1);
  \node[above=3pt] at (8.5, \hiii-0.1) {2023};
  
  \draw[flashl] (9.25, \hiii) -- (9.25, \hiii-0.25);
  \node[flashc,cascaded] at (9.25, \hiii) {};
  \node[item,text width=3cm] at (9.25, \hiii-0.2) (2023lei) {\textbf{Lei \emph{et al.}} \cite{pami/LeiJC23} \\  Journal Ver-\\sion of \cite{cvpr/LeiC21}};

  \draw[year] (10, \hiii+0.1) -- (10, \hiii-0.1);
  \node[above=3pt] at (10, \hiii-0.1) {2024};

  \draw[flashl] (10.5, \hiii) -- (10.5, \hiii+0.25);
  \node[flashc,cascaded] at (10.5, \hiii) {};
  \node[item,text width=3cm] at (10.5, \hiii+1.5) (2024hong) {\textbf{Hong \emph{et al.}} \cite{ijcv/hong2024light} \\ Light Flickering Guided Removal};
  
  \draw[year] (11, \hiii+0.1) -- (11, \hiii-0.1);
  \node[above=3pt] at (11, \hiii-0.1) {2025};

  \draw[polarl] (11.75, \hiii) -- (11.75, \hiii-0.25);
  \node[polarc,single] at (11.75, \hiii) {};
  \node[item,text width=3cm] at (11.75, \hiii-0.2) (2025yao) {\textbf{Yao \emph{et al.}} \cite{cvpr/yao2025polarfree} \\  PolaRGB Dataset and Diffusion-based Model};

  \draw[oppol] (13, \hiii) -- (13, \hiii+0.25);
  \node[oppoc,single] at (13, \hiii) {};
  \node[item,text width=3cm] at (13, \hiii+1.5) (2025kee) {\textbf{Kee \emph{et al.}} \cite{cvpr/kee2025removing} \\ RAW Domain with Opposite-View Input};
  
  \tikzmath{\hiv = -3;};
  \draw[democ] (1.75, \hiv) circle (3pt);
  \node[item] at (1.75, \hiv-0.1) () {Optimization-based Methods/Survey};
  
  \node[democ,single] at (5.75, \hiv-0.05) {};
  \node[item] at (5.75, \hiv-0.1) () {Single-Stream Network Structure};

  \node[democ,cascaded] at (9.75, \hiv-0.05) {};
  \node[item] at (9.75, \hiv-0.1) () {Cascaded/Iterative Network Structure};

\end{tikzpicture}
\caption{Roadmap of multiple image reflection removal/separation: \textcolor{color-blue}{blue} nodes represent polarization-based methods, \textcolor{color-red}{red} nodes are motion/shift-based methods, \textcolor{color-yellow}{yellow} nodes denote stereo-based methods, \textcolor{color-green}{green} nodes represent general multiple-image-based methods,  \textcolor{color-purple}{purple} nodes denote flash/no-flash-pair-based methods, and \textcolor{color-cyan}{cyan} nodes indicate the methods with paired reflection scene indicators. }
\label{fig:timeline1}
\end{figure*}
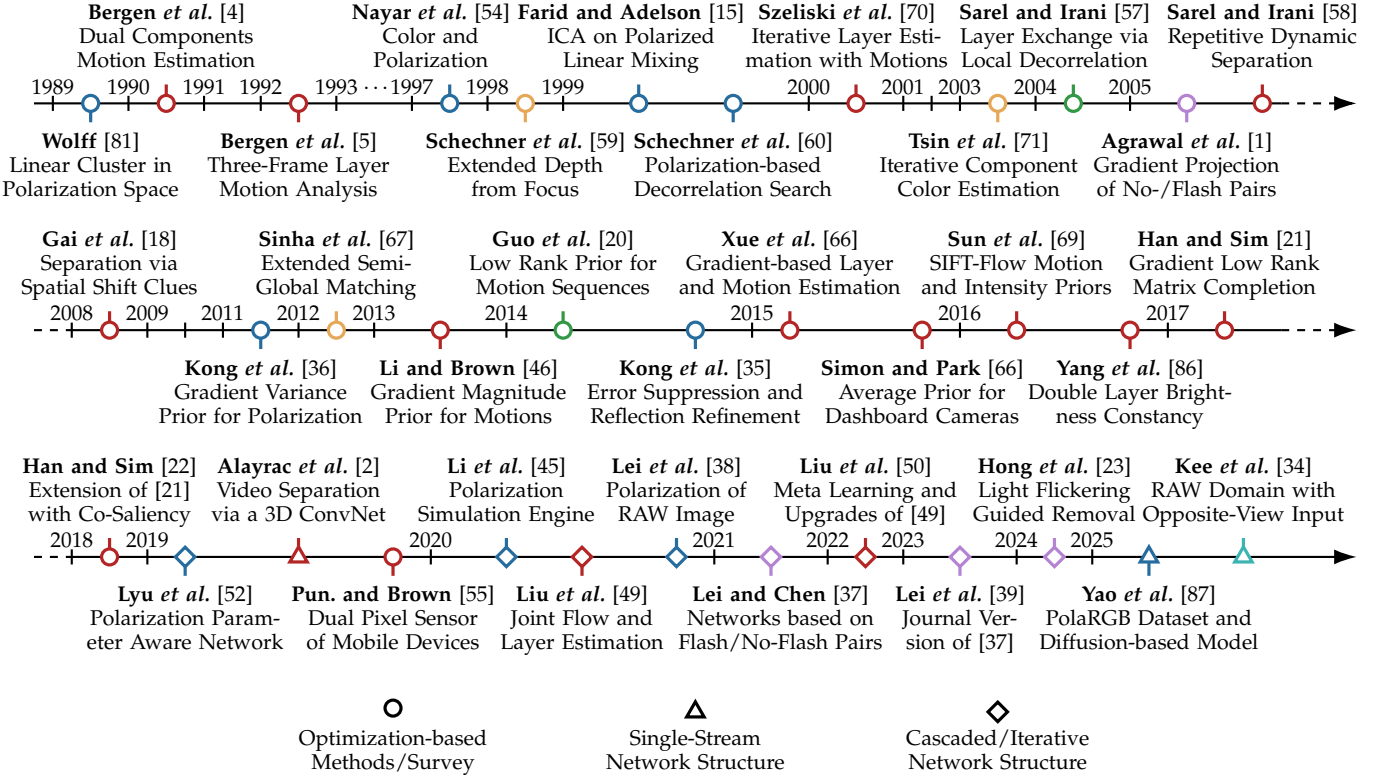

\section{A Roadmap-style Survey}
\label{sec:survey}

Reflection separation addresses the challenging task of decomposing an observed image into transmission and reflection layers under severe ill-posedness. Over the past two decades, the field has evolved significantly to mitigate this ambiguity, branching into multiple-image and single-image trajectories. As summarized by the comprehensive roadmaps in Fig.~\ref{fig:timeline1} and Fig.~\ref{fig:timeline2}, alongside an interactive version available on our project page, multiple-image methods typically leverage auxiliary physical cues to explicitly reduce layer entanglement. While these methods provide useful context, they often impose restrictive capture conditions. Consequently, the field has increasingly focused on the more practical single-image setting. This domain has witnessed a paradigm shift from optimization-based methods relying on hand-crafted visual priors to data-driven architectures that progressively emphasize structured modeling of layer interactions. In the following, we review representative works from both categories to contextualize the development of our proposed dual-stream interactive paradigm.

\subsection{Multiple Image Reflection Separation}
\label{sec:survey_mirs}

As categorized by the timeline in Fig.~\ref{fig:timeline1}, Multiple Image Reflection Separation (MIRS) methods exploit auxiliary observations to mitigate the fundamental ambiguity of layer decomposition. We review representative works from each primary category based on their employed physical or geometric cues in the following discussion.

\noindent\textbf{Polarization-based Methods.}
Polarization-based approaches leverage the differing polarization characteristics of the transmission and reflection layers. By rotating a linear polarizer in front of the camera, multiple images can be captured with varying reflection intensities, as discussed in the optical motivation of the main paper. This enables the disentanglement of layers based on their polarization responses, which is a principle rooted in physical optics. A series of methods have explored this strategy~\cite{cvpr/FaridA99,schechner1999polarization,tip/KongTS11,pami/KongTS14,nips/LyuCLPS19,eccv/LiQZH20,cvpr/LeiHZYSC20}, which are indicated by blue nodes in Fig.~\ref{fig:timeline1}. Early work by Farid and Adelson~\cite{cvpr/FaridA99} applied independent component analysis (ICA) to polarized image pairs to estimate the underlying layers. Kong \emph{et al.}~\cite{tip/KongTS11,pami/KongTS14} further introduced an optimization framework under the assumptions of mutual gradient exclusivity between polarization states. Wieschollek \emph{et al.}~\cite{eccv/WieschollekGGK18} enhanced the reconstruction fidelity using a canonical projection layer within a deep learning pipeline. Lyu \emph{et al.}~\cite{nips/LyuCLPS19} combined unpolarized and polarized images using semi-reflector orientation estimation. Li \emph{et al.}~\cite{eccv/LiQZH20} proposed a polarization-guided ray-tracing model to simulate physical reflection formation, while Lei \emph{et al.}~\cite{cvpr/LeiHZYSC20} addressed misalignment issues by designing a refined data acquisition process and a two-stage learning framework. More recently, Yao \emph{et al.}~\cite{cvpr/yao2025polarfree} proposed the PolaRGB dataset specifically curated for polarization-based reflection removal, and applied a diffusion-based technique to improve separation accuracy. 
{Despite their effectiveness, polarization-based methods require specialized hardware and are sensitive to viewing geometry, limiting their general applicability.}

\noindent\textbf{Motion-based Methods.}
The reflection layer typically exhibits geometric shifts with changes in the viewing angle due to refraction effects governed by Snell’s Law, whereas the transmission layer tends to vary more smoothly under stereo geometry.  Based on this observation, a large number of methods have explored the use of relative motion and stereo cues~\cite{eccv/BergenBHP90,pami/BergenBHP92,cvpr/SzeliskiAA00,iccv/SarelI05,pami/GaiSZ12,iccv/LiB13,tog/Freeman15,cvpr/SimonP15,mm/SunLYZWL16,cvpr/YangLDT16,cvpr/HanS17,nips/LyuCLPS19,cvpr/AlayracCZ19,cvpr/PunnappurathB19,cvpr/LiuL0CH20,pami/LiuLYCH22,iccv/SchechnerKB98,cvpr/TsinKS03,tog/SinhaKGSS12}, which are colored red in Fig.~\ref{fig:timeline1}. Early work by Bergen and Burt~\cite{eccv/BergenBHP90,pami/BergenBHP92} employed motion parallax to decouple layers, while Szeliski \emph{et al.}~\cite{cvpr/SzeliskiAA00} introduced video stabilization to isolate the transmission layer by suppressing temporally inconsistent reflections. Sarel and Irani~\cite{iccv/SarelI05} enhanced robustness using normalized cross-correlation, while Gai \emph{et al.}~\cite{pami/GaiSZ12} modeled motion fields for both layers explicitly. Li and Brown~\cite{iccv/LiB13} improved alignment via SIFT-flow. Xue \emph{et al.}~\cite{tog/Freeman15} further enforced sparsity priors to promote independence between layers. Later efforts by Simon and Pritch~\cite{cvpr/SimonP15} and Sun \emph{et al.}~\cite{mm/SunLYZWL16} incorporated temporal consistency cues, while Yang \emph{et al.}~\cite{cvpr/YangLDT16} and Han and Sim~\cite{cvpr/HanS17} improved optical flow estimation for transparent scenes.
With the emergence of deep learning, several works began addressing more complex and non-rigid motion patterns. Lyu \emph{et al.}~\cite{nips/LyuCLPS19} introduced a neural framework incorporating semi-reflector orientation estimation to guide layer decomposition. Alayrac \emph{et al.}~\cite{cvpr/AlayracCZ19} proposed a weakly-supervised learning strategy for generic video layer separation. Punnappurath and Brown~\cite{cvpr/PunnappurathB19} employed low-rank priors via matrix completion in conjunction with learning-based alignment. Liu \emph{et al.}~\cite{cvpr/LiuL0CH20,pami/LiuLYCH22} advanced the field by jointly learning motion estimation and layer separation through feature correlation, achieving improved robustness in real-world settings.
While motion-based methods avoid the need for specialized equipment, their effectiveness heavily relies on precise image alignment and degrades notably under non-rigid motion or large parallax, accompanied by substantial computational overhead.

\noindent\textbf{Stereo-based Methods.}
These approaches exploit depth differences captured by binocular or multi-view cameras to separate overlapping layers. Schechner \emph{et al.} \cite{iccv/SchechnerKB98} employed depth-from-focus by capturing images at varied focal lengths, using blur analysis to achieve initial separation and refinement. Tsin \emph{et al.} \cite{cvpr/TsinKS03} proposed a layered stereo matching model, introducing nested plane sweep and graph-cut optimization to jointly estimate depths and separate layers. Sinha \emph{et al.} \cite{tog/SinhaKGSS12} combined stereo disparity estimation with multi-depth reasoning, using boundary and gradient constraints to enhance separation in mirror-like and translucent scenes. These methods are represented by the blue markers in Fig.~\ref{fig:timeline1}.
{Although such methods alleviate the dependency on manual priors, their performance declines with low-texture regions, small disparities, or inaccurate depth estimation, often leading to unstable results.}

\noindent\textbf{Flash/No-flash Pair-based Methods.}
These methods capture both flash and no-flash images to exploit lighting and reflection differences for separation. Agrawal \emph{et al.} \cite{tog/AgrawalRNL05} introduced gradient projection and flash-exposure sampling to address flash reflection artifacts and imbalanced lighting, leveraging gradient coherence to remove reflections while maintaining image details. Lei and Chen \cite{cvpr/LeiC21} proposed subtracting the ambient image from the flash image to create a reflection-free flash-only image for improved reflection removal. Lei \emph{et al.} \cite{pami/LeiJC23} further developed a method to handle misalignment in handheld photography, using a misalignment synthesis pipeline and depth estimation to achieve state-of-the-art performance even with imperfect data. Hong \emph{et al.} \cite{ijcv/hong2024light} took a novel approach by leveraging the periodic flickering of artificial light to extract fluctuant and consistent components from reflective videos, aiding the separation of reflection and transmission scenes. {These methods face limitations, including reliance on two images, misalignment in dynamic scenes, flash artifacts, and reduced effectiveness in poorly lit or highly reflective environments.}

\noindent\textbf{General Multiple Image-based Methods.}
These methods exploit variations across sequences or videos to separate reflection and transmission layers. As marked by green symbols in Fig.~\ref{fig:timeline1}, Sarel and Irani \cite{iccv/SarelI05} proposed ``layer information exchange'', minimizing structural correlation across image mixtures at multiple scales to handle non-rigid transparent motion. Guo \emph{et al.} \cite{cvpr/GuoCM14} developed an Augmented Lagrangian Multiplier-based method, leveraging layer gradient sparsity and inter-image consistency to solve the decomposition task across both synthetic and real-world data.
{While general multiple-image-based methods perform well under dynamic conditions, they rely heavily on sufficient variation across frames and accurate alignment, limiting robustness when these assumptions are violated.}

\noindent\textbf{Summary.} Multiple image methods successfully mitigate the fundamental ambiguity of reflection separation by introducing physical or geometric constraints from auxiliary observations. However, requirements for specialized hardware or static scenes restrict their practical applicability, thereby motivating our focus on robust single image solutions in this work. Nevertheless, to demonstrate the versatility of our proposed dual stream paradigm, we further adapt our architecture to leverage auxiliary physical cues for polarized image reflection separation in Sec.~\ref{sec:extend_exps}.

\subsection{Single Image Reflection Separation}
\label{sec:survey_sirs}

\begin{figure*}[t]
\centering
\begin{tikzpicture}[timeline/.style={draw, thick, -{Latex[length=3mm,width=2mm]}}, 
                    gradc/.style={fill=white, draw=color-blue,very thick},
                    gradl/.style={draw=color-blue,very thick},
                    manualc/.style={fill=white, draw=color-purple,very thick},
                    manuall/.style={draw=color-purple,very thick},
                    smoothc/.style={fill=white, draw=color-yellow,very thick},
                    smoothl/.style={draw=color-yellow,very thick},
                    deepc/.style={fill=white, draw=color-red,very thick},
                    deepl/.style={draw=color-red,very thick},
                    panoc/.style={fill=white, draw=color-pink,very thick},
                    panol/.style={draw=color-pink,very thick},
                    gmmc/.style={fill=white, draw=color-orange,very thick},
                    gmml/.style={draw=color-orange,very thick},
                    datac/.style={fill=white, draw=color-green,very thick},
                    datal/.style={draw=color-green,very thick},
                    democ/.style={fill=white,draw=black,very thick},
                    cascaded/.style={diamond,minimum width=0.275cm,minimum height=0.275cm,inner sep=0},
                    single/.style={isosceles triangle,rotate=90,isosceles triangle apex angle=60, minimum size =0.2cm,inner sep=0},
                    dual-stream/.style={rectangle,minimum width=0.2cm,minimum height=0.2cm,inner sep=0},
                    dual-stream-i/.style={regular polygon,regular polygon sides=5,minimum width=0.275cm,minimum height=0.275cm,inner sep=0},
                    year/.style={draw, black, thick},
                    item/.style={align=center, text width=3cm, inner sep=5pt, anchor=north},
                    every node/.style={font=\footnotesize}]

  \draw[thick] (-2.5,5) -- (13.5,5);
  \draw[thick,dashed,-{Latex[length=3mm,width=2mm]}] (13.5,5) -- (14.5,5);

  \tikzmath{\hi = 5.0;};
  
  \draw[year] (-2, \hi+0.1) -- (-2, \hi-0.1);
  \node[above=3pt] at (-2, \hi-0.1) {2002};

  \draw[gradl] (-1.5, \hi) -- (-1.5, \hi-0.25);
  \draw[gradc] (-1.5, \hi) circle (3pt);
  \node[item,text width=3.5cm] at (-1.5, 4.8) (2002levin) {\textbf{Levin \emph{et al.}} \cite{nips/LevinZW02} \\ Most Probable Sepa-ration Search by Deri-vative and Corner};

  \draw[year] (-1, \hi+0.1) -- (-1, \hi-0.1);

  \draw[year] (-0.5, \hi+0.1) -- (-0.5, \hi-0.1);
  \node[above=3pt] at (-0.5, \hi-0.1) {2004};

  \draw[gradl] (0, \hi) -- (0, \hi+0.25);
  \draw[gradc] (0, \hi) circle (3pt);
  \node[item,text width=3.5cm] at (0, \hi+1.8) (2004levin) {\textbf{Levin \emph{et al.}} \cite{cvpr/LevinZW04} \\  Improved Cost Function and Patch Discretization Database of \cite{nips/LevinZW02}};
  
  \draw[manuall] (1.25, \hi) -- (1.25, \hi-0.25);
  \draw[manualc] (1.25, \hi) circle (3pt);
  \node[item,text width=3cm] at (1.25, 4.8) (2004levin) {\textbf{Levin and Weiss} \cite{eccv/LevinW04} \\ Separation Based on Manually Marked Gradients};

  \draw[year] (2, \hi+0.1) -- (2, \hi-0.1);
  \node[above=3pt] at (2, \hi-0.1) {2005};

  \node at (2.53, 5.2) {\ldots};
  \draw[year] (3, \hi+0.1) -- (3, \hi-0.1);
  \node[above=3pt] at (3, \hi-0.1) {2009};
  
  \draw[smoothl] (3.5, \hi) -- (3.5, \hi+0.25);
  \draw[smoothc] (3.5, \hi) circle (3pt);
  \node[item,text width=4cm] at (3.5, \hi+1.8) (2009chung) {\textbf{Chung \emph{et al.}} \cite{wacv/ChungCWC09} \\  Fuzzy Classification of Edges of Layers involving Blur Measurement};
  
  \draw[year] (4, \hi+0.1) -- (4, \hi-0.1);
  \node[above=3pt] at (4, \hi-0.1) {2010};
  \node at (4.53, 5.2) {\ldots};
  \draw[year] (5, \hi+0.1) -- (5, \hi-0.1);
  \node[above=3pt] at (5, \hi-0.1) {2013};
  
  \draw[smoothl] (5.5, \hi) -- (5.5, \hi-0.25);
  \draw[smoothc] (5.5, \hi) circle (3pt);
  \node[item,text width=4cm] at (5.5,\hi-0.2) (2013yan) {\textbf{Yan \emph{et al.}} \cite{iscas/YanXY13} \\ Edge Classification via Gradient Profile Sharpness Validated by NGC \cite{eccv/SarelI04}};
  
  \draw[year] (6, \hi+0.1) -- (6, \hi-0.1);
  \node[above=3pt] at (6, \hi-0.1) {2014};

  \draw[smoothl] (7, \hi) -- (7, \hi+0.25);
  \draw[smoothc] (7, \hi) circle (3pt);
  \node[item] at (7, \hi+1.8) (2014li) {\textbf{Li and Brown} \cite{cvpr/LiB14} \\ Smoother Reflection Separation via Stronger Gradient Penalty};
  
  \draw[year] (7.5, \hi+0.1) -- (7.5, \hi-0.1);
  \node[above=3pt] at (7.5, \hi-0.1) {2015};
  
  \draw[gmml] (8.75, \hi) -- (8.75, \hi-0.25);
  \draw[gmmc] (8.75, \hi) circle (3pt);
  \node[item,text width=3.5cm] at (8.75, \hi-0.2) (2015shih) {\textbf{Shih \emph{et al.}} \cite{cvpr/ShihKDF15} \\ Separation based on GMM and Ghosting Kernel Estimation};

  \draw[year] (9.25, \hi+0.1) -- (9.25, \hi-0.1);
  \node[above=3pt] at (9.25, \hi-0.1) {2016};

  \tikzmath{\hloc = 10;};
  \draw[smoothl] (\hloc, \hi) -- (\hloc, \hi+0.25);
  \draw[smoothc] (\hloc, \hi) circle (3pt);
  \node[item] at (\hloc, \hi+1.8) (2016wan) {\textbf{Wan \emph{et al.}} \cite{icip/WanSHK16} \\ Layer Edge Selection based on Depth of Field Pyramid};

  \tikzmath{\hloc = 10.75;};
  \draw[year] (\hloc, \hi+0.1) -- (\hloc, \hi-0.1);
  \node[above=3pt] at (\hloc, \hi-0.1) {2017};
  
  \tikzmath{\hloc = 11.5;};
  \draw[datal] (\hloc, \hi) -- (\hloc, \hi-0.25);
  \draw[datac] (\hloc, \hi) circle (3pt);
  \node[item] at (\hloc, \hi-0.2) (2017wan) {\textbf{Wan \emph{et al.}} \cite{iccv/WanSDTK17} \\ $\text{SIR}^2$ Benchmark Dataset with Real-world Paired Layers};

  \tikzmath{\hloc = 13;};
  \draw[gradl] (\hloc, \hi) -- (\hloc, \hi+0.25);
  \node[gradc,cascaded] at (\hloc, \hi) {};
  \node[item,text width=3.5cm] at (\hloc, \hi+1.8) (2017fan) {\textbf{Fan \emph{et al.}} \cite{iccv/FanYHCW17} \\ Cascaded Edge Detection and Guided Separation Networks};

  \tikzmath{\hii = 1.5;};
  
  \draw[thick,dashed] (-2.5, \hii) -- (-2, \hii);
  \draw[thick] (-2,\hii) -- (13.5,\hii);
  \draw[thick,dashed,-{Latex[length=3mm,width=2mm]}] (13.5,\hii) -- (14.5,\hii);

  \tikzmath{\hloc = -2;};
  \draw[year] (\hloc, \hii+0.1) -- (\hloc, \hii-0.1);
  \node[above=3pt] at (\hloc, \hii-0.1) {2018};

  \tikzmath{\hloc = -1.5;};
  \draw[gradl] (\hloc, \hii) -- (\hloc, \hii-0.25);
  \node[gradc,cascaded] at (\hloc, \hii) {};
  \node[item,text width=3cm] at (\hloc, \hii-0.1) (2018wan) {\textbf{Wan \emph{et al.}} \cite{cvpr/WanSDTK18} \\ Multi-Scale Edge-Feature-Guided Network};

  \tikzmath{\hloc = -1;};
  \draw[deepl] (\hloc, \hii) -- (\hloc, \hii+0.25);
  \node[deepc,single] at (\hloc, \hii)  {};
  \node[item,text width=4cm] at (\hloc, \hii+1.8) (2018zhang) {\textbf{Zhang \emph{et al.}} \cite{cvpr/ZhangNC18a} \\ Single Network with Perceptual, Adversarial and Exclusion Losses};

  \tikzmath{\hloc = 1.1;};
  \draw[deepl] (\hloc, \hii) -- (\hloc, \hii-0.25);
  \node[deepc,cascaded] at (\hloc, \hii) {};
  \node[item] at (\hloc, \hii-0.1) (2018yang) {\textbf{Yang \emph{et al.}} \cite{eccv/YangGLS18} \\ Network with Three Stages, Alternated to Estimate the Layers};

  \tikzmath{\hloc = 1.6;};
  \draw[year] (\hloc, \hii+0.1) -- (\hloc, \hii-0.1);
  \node[above=3pt] at (\hloc, \hii-0.1) {2019};

  \tikzmath{\hloc = 2;};
  \draw[deepl] (\hloc, \hii) -- (\hloc, \hii+0.25);
  \node[deepc,single] at (\hloc, \hii)  {};
  \node[item,text width=4cm] at (\hloc, \hii+1.8) (2019wei) {\textbf{Wei \emph{et al.}} \cite{cvpr/WeiYFW019} \\ Network with Con-textual Blocks Exploit-ing Misaligned Pairs};

  \tikzmath{\hloc = 2.5;};
  \draw[year] (\hloc, \hii+0.1) -- (\hloc, \hii-0.1);
  \node[above=3pt] at (\hloc, \hii-0.1) {2020};
  
  \tikzmath{\hloc = 4;};
  \draw[deepl] (\hloc, \hii) -- (\hloc, \hii-0.25);
  \node[deepc,cascaded] at (\hloc, \hii) {};
  \node[item,text width=4cm] at (\hloc, \hii-0.1) (2020wan) {\textbf{Wan \emph{et al.}} \cite{cvpr/WanSLDK20} \\ Cascaded Separation and Reflection Enhancement Networks};

  \tikzmath{\hloc = 5;};
  \draw[deepl] (\hloc, \hii) -- (\hloc, \hii+0.25);
  \node[deepc,dual-stream] at (\hloc, \hii) {};
  \node[item] at (\hloc, \hii+1.8) (2020li) {\textbf{Li \emph{et al.}} \cite{cvpr/LiY0LH20} \\ Dual-Branch Iterative Separation Network with LSTM Units};
  
  \tikzmath{\hloc = 6.85;};
  \draw[datal] (\hloc, \hii) -- (\hloc, \hii-0.25);
  \node[datac,single] at (\hloc, \hii) {};
  \node[item] at (\hloc, \hii-0.1) (2019wen) {\textbf{Wen \emph{et al.}} \cite{cvpr/WenT0LHH19} \\ GAN-based Reflection Superimposition Data Synthesis};

  \tikzmath{\hloc = 7.25;};
  \draw[year] (\hloc, \hii+0.1) -- (\hloc, \hii-0.1);
  \node[above=3pt] at (\hloc, \hii-0.1) {2021};

  \tikzmath{\hloc = 7.75;};
  \draw[deepl] (\hloc, \hii) -- (\hloc, \hii+0.25);
  \node[deepc,dual-stream] at (\hloc, \hii) {};
  \node[item] at (\hloc, \hii+1.8) (2021feng) {\textbf{Feng \emph{et al.}} \cite{tip/FengPJCZL21} \\ Dual-Branch Sepa-ration and Further Refinement Network};     

  \tikzmath{\hloc = 9.75;};
  \draw[deepl] (\hloc, \hii) -- (\hloc, \hii-0.25);
  \node[deepc,cascaded] at (\hloc, \hii) {};
  \node[item,text width=4cm] at (\hloc,\hii-0.1)  (2021li) {\textbf{Li \emph{et al.}} \cite{li2023two} \\ Multi-Scale Reflection Guided Transmission Estimation Network};

  \tikzmath{\hloc = 10.5;};
  \draw[deepl] (\hloc, \hii) -- (\hloc, \hii+0.25);
  \node[deepc,cascaded] at (\hloc, \hii) {};
  \node[item] at (\hloc, \hii+1.8) (2021zheng) {\textbf{Zheng \emph{et al.}} \cite{cvpr/ZhengSCJDK21} \\ Two-Stage Network with Absorption Factor Estimation};

  \tikzmath{\hloc = 12.75;};
  \draw[deepl] (\hloc, \hii) -- (\hloc, \hii-0.25);
  \node[deepc,dual-stream] at (\hloc, \hii) {};
  \node[item, text width = 4cm] at (\hloc, \hii-0.1) (2021feng2) {\textbf{Feng \emph{et al.}} \cite{icme/FengJJP0L21} \\ Two-Branch Network with Contrastive Fea-ture and Loss Designs};

  \tikzmath{\hloc = 13.25;};
  \draw[panol] (\hloc, \hii) -- (\hloc, \hii+0.25);
  \node[panoc,dual-stream] at (\hloc, \hii) {};
  \node[item, text width = 4cm] at (\hloc, \hii+1.8) (2021hong) {\textbf{Hong \emph{et al.}} \cite{cvpr/HongZZJKS21} \\ Reflection Retrieval and Refinement based on Panoramic Images};

  \tikzmath{\hiii = -2;};
  \draw[thick,dashed] (-2.5,\hiii) -- (-2,\hiii);
  \draw[thick,-{Latex[length=3mm,width=2mm]}] (-2,\hiii) -- (14.5,\hiii);

  \tikzmath{\hloc = -1.5;};
  \draw[deepl] (\hloc, \hiii) -- (\hloc, \hiii+0.25);
  \node[deepc,dual-stream-i] at (\hloc, \hiii) {};
  \node[item, text width = 3cm] at (\hloc, \hiii+1.8) (2021hu) {\textbf{Hu and Guo} \cite{nips/HuG21} \\  Dual-Stream Feature Interaction through Paired Activators};

  \tikzmath{\hloc = -1;};
  \draw[deepl] (\hloc, \hiii) -- (\hloc, \hiii-0.25);
  \node[deepc,cascaded] at (\hloc, \hiii) {};
  \node[item, text width = 4cm] at (\hloc, \hiii-0.1) (2021dong) {\textbf{Dong \emph{et al.}} \cite{iccv/Dong00BXL21} \\ Recurrent Separation Network with Multi-Scale Laplacian Prior};

  \tikzmath{\hloc = -0.5;};
  \draw[year] (\hloc, \hiii-0.1) -- (\hloc, \hiii+0.1);
  \node[above=3pt] at (\hloc, \hiii-0.1) {2022};

  \tikzmath{\hloc = 1.2;};
  \draw[deepl] (\hloc, \hiii) -- (\hloc, \hiii+0.25);
  \node[deepc,cascaded] at (\hloc, \hiii) {};
  \node[item,text width=3cm] at (\hloc, \hiii+1.8) (2022zhang) {\textbf{Zhang \emph{et al.}} \cite{mm/ZhangSL22} \\ Deep Unfolding Network with Multi-Scale Dictionary};

  \tikzmath{\hloc = 1.7;};
  \draw[year] (\hloc, \hiii+0.1) -- (\hloc, \hiii-0.1);
  \node[above=3pt] at (\hloc, \hiii-0.1) {2023};

  \tikzmath{\hloc = 2.2;};
  \draw[deepl] (\hloc, \hiii) -- (\hloc, \hiii-0.25);
  \node[deepc,single] at (\hloc, \hiii) {};
  \node[item,text width=3cm] at (\hloc, \hiii+-0.1) (2023song) {\textbf{Song \emph{et al.}} \cite{cvpr/song2023robust} \\ Robust Structure and Training against Adversarial Attacks};

  \tikzmath{\hloc = 4;};
  \draw[deepl] (\hloc, \hiii) -- (\hloc, \hiii+0.25);
  \node[deepc,dual-stream-i] at (\hloc, \hiii) {};
  \node[item,text width=3cm] at (\hloc, \hiii+1.8) (2023hu) {\textbf{Hu and Guo} \cite{iccv/Hu23} \\Dual-Stream Net-\\work with Mutually-Gated Interactions};

  \tikzmath{\hloc = 4.5;};
  \draw[year] (\hloc, \hiii-0.1) -- (\hloc, \hiii+0.1);
  \node[above=3pt] at (\hloc, \hiii-0.1) {2024};

  \tikzmath{\hloc = 5;};
  \draw[datal] (\hloc, \hiii) -- (\hloc, \hiii-0.25);
  \node[datac,cascaded] at (\hloc, \hiii) {};
  \node[item,text width=3cm] at (\hloc, \hiii-0.1) (2024zhu) {\textbf{Zhu \emph{et al.}} \cite{cvpr/zhu2024revisiting} \\ Two-Stage model with Enhanced Data Pipeline};

  \tikzmath{\hloc = 6.8;};
  \draw[manuall] (\hloc, \hiii) -- (\hloc, \hiii+0.25);
  \node[manualc,cascaded] at (\hloc, \hiii) {};
  \node[item,text width=3.5cm] at (\hloc, \hiii+1.8) (2024zhong) {\textbf{Zhong \emph{et al.}} \cite{cvpr/ZhongHWLS24} \\ Language Driven Separation Based on Cross Attention};

  \tikzmath{\hloc = 7.3;};
  \draw[manuall] (\hloc, \hiii) -- (\hloc, \hiii-0.25);
  \node[manualc,single] at (\hloc, \hiii) {};
  \node[item,text width=2.8cm] at (\hloc, \hiii-0.1) (2024hong) {\textbf{Hong \emph{et al.}} \cite{eccv/hong2024differ} \\ Iterative Language-Based Diffusion Model};

  \tikzmath{\hloc = 9.8;};
  \draw[deepl] (\hloc, \hiii) -- (\hloc, \hiii+0.25);
  \node[deepc,dual-stream-i] at (\hloc, \hiii) {};
  \node[item,text width=3.2cm] at (\hloc, \hiii+1.8) (2024hu) {\textbf{Hu and Guo} \cite{nips/hu2024single} \\ Dual-Stream Trans-former with Dual-Attention Interactions};

 \tikzmath{\hloc = 10.3;};
  \draw[year] (\hloc, \hiii-0.1) -- (\hloc, \hiii+0.1);
  \node[above=3pt] at (\hloc, \hiii-0.1) {2025};

  \tikzmath{\hloc = 10.8;};
  \draw[deepl] (\hloc, \hiii) -- (\hloc, \hiii-0.25);
  \node[deepc,dual-stream-i] at (\hloc, \hiii) {};
  \node[item,text width=3.2cm] at (\hloc-0.65, \hiii-0.1) (2025huang) {\textbf{Huang \emph{et al.}} \cite{pami/huang2025lightweight} \\  Dual-Stream Net-\\work with Unfold-\\ing Exclusion Blocks};

  \tikzmath{\hloc = 12.5;};
  \draw[deepl] (\hloc, \hiii) -- (\hloc, \hiii+0.25);
  \node[deepc,dual-stream-i] at (\hloc, \hiii) {};
  \node[item,text width=3.0cm] at (\hloc, \hiii+1.8) (2025zhao) {\textbf{Zhao \emph{et al.}} \cite{cvpr/zhao2025reversible} \\ Dual-Stream Reversible De-\\coupling Network};

  \tikzmath{\hloc = 13;};
  \draw[year] (\hloc, \hiii-0.1) -- (\hloc, \hiii+0.1);
  \node[above=3pt] at (\hloc, \hiii-0.1) {2026};

  \tikzmath{\hloc = 13.5;}; 
  \draw[deepl] (\hloc, \hiii) -- (\hloc, \hiii-0.25); \node[deepc,single] at (\hloc, \hiii) {}; 
  \node[item] at (\hloc-0.5, \hiii-0.1) {\textbf{Hu \emph{et al.}} \cite{aaai/hu2026dereflection} \\ One-Step Diffu-\\sion Priors with Larger Dataset};

  \tikzmath{\hiv = -4;};
  \draw[democ] (0.25, \hiv) circle (3pt);
  \node[item] at (0.25, \hiv-0.1) () {Optimization-based Methods/Survey};
  
  \node[democ,single] at (3, \hiv-0.05) {};
  \node[item] at (3, \hiv-0.1) () {Single-Stream Network Structure};

  \node[democ,cascaded] at (5.75, \hiv-0.05) {};
  \node[item] at (5.75, \hiv-0.1) () {Cascaded/Iterative Network Structure};
   
  \node[democ,dual-stream] at (8.5, \hiv-0.05) {};
  \node[item] at (8.5, \hiv-0.1) () {Dual-Stream Network Structure};

  \node[democ,dual-stream-i] at (12, \hiv-0.05) {};
  \node[item,text width=4cm] at (12, \hiv-0.1) () {Dual-Stream Network Structure with Feature Interactions};

  
\end{tikzpicture}
\caption{Roadmap of single-image reflection removal/separation: \textcolor{color-blue}{blue} nodes represent methods mainly based on gradient/edge/corner priors, \textcolor{color-purple}{purple} nodes denote manual annotation/language guided methods, \textcolor{color-yellow}{yellow} nodes are methods using blur/relative smoothness/depth of field priors, \textcolor{color-red}{red} nodes denote perceptual information aided methods with pretrained backbone/perceptual loss/adversarial training, \textcolor{color-orange}{orange} nodes show methods with ghosting clues, \textcolor{color-green}{green} nodes denote dataset contributions, and \textcolor{color-pink}{pink} ones represent dependency on panoramic images}
\label{fig:timeline2}
\end{figure*}
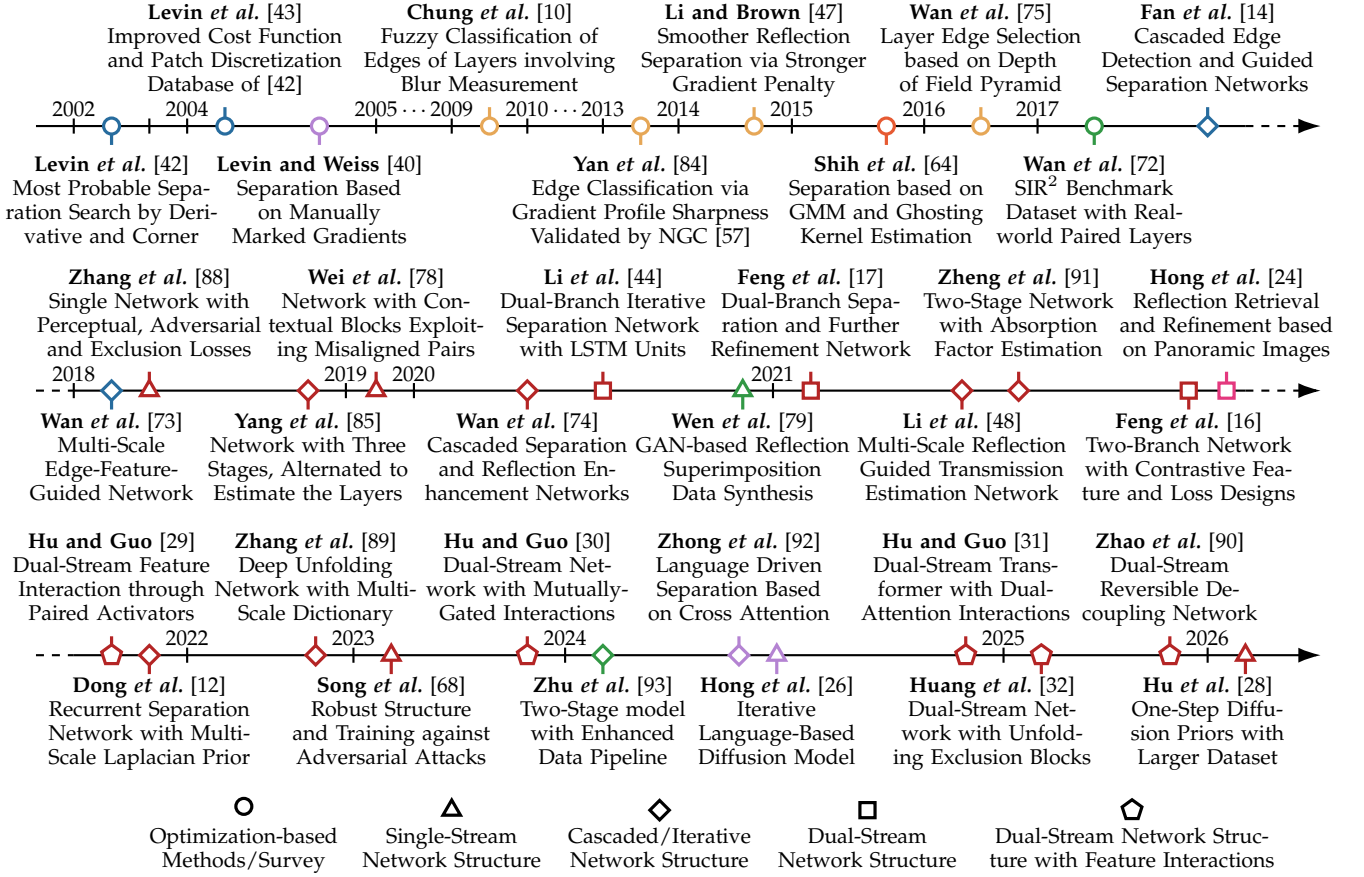

As summarized in Fig.~\ref{fig:timeline2}, the single-image setting has undergone a continuous evolution. Early approaches were dominated by optimization-based methods relying on strong visual priors, such as gradient sparsity \cite{nips/LevinZW02}. Later, researchers explored diverse physical assumptions, including relative smoothness \cite{cvpr/LiB14} and ghosting effects \cite{cvpr/ShihKDF15}. With the advent of deep learning, a paradigm shift toward data-driven approaches emerged, where single-stream and cascaded architectures leveraged large-scale datasets for improved performance. More recently, the field has transitioned toward structured modeling of layer interactions. Dual-stream architectures \cite{cvpr/LiY0LH20} explicitly model transmission and reflection in parallel, evolving into interaction-driven designs \cite{nips/HuG21, nips/hu2024single} that enable deeper cross-layer information exchange. In parallel, multimodal and language-guided approaches \cite{cvpr/ZhongHWLS24} have emerged for resolving semantic ambiguity. Based on this evolution, we categorize existing single-image methods according to their underlying priors and architectural designs as follows:

\noindent\textbf{Gradient/Edge Prior-based Methods.}
Emerging in the early 2000s, as marked by the blue nodes in the roadmap, these approaches exploit the sparsity of image gradients. Levin \emph{et al.} \cite{nips/LevinZW02,cvpr/LevinZW04} pioneered this era using sparse gradient priors and iterative optimization. Advancing into the late 2010s, researchers integrated these priors into deep neural architectures. CEILNet \cite{iccv/FanYHCW17} and Wan \emph{et al.} \cite{cvpr/WanSDTK18} developed edge-guided networks to capture structural boundaries. However, these edge-driven methods frequently struggled when both layers exhibited dense, overlapping textures, which eventually catalyzed the field's shift toward more comprehensive perceptual models.

\noindent\textbf{Blur/Smoothness Prior-based Methods.}
Flourishing in the mid-2010s (yellow nodes), these methods assume reflection layers are inherently blurrier due to depth-of-field discrepancies. Early implementations utilized total variation \cite{wacv/ChungCWC09} or measured gradient profile sharpness \cite{iscas/YanXY13}. Li and Brown \cite{cvpr/LiB14} formally introduced the relative smoothness prior by penalizing reflection gradients, while Wan \emph{et al.} \cite{icip/WanSHK16} leveraged depth-of-field pyramids. Although occasionally misclassifying layers when the background was also defocused, this relative smoothness assumption profoundly influenced the field by pioneering a foundational data-synthesis strategy widely adopted in subsequent deep learning models.

\noindent\textbf{Ghosting and Panoramic Methods.}
To tackle specific physical phenomena, mid-2010s methods explored ghosting effects (orange nodes). Shih \emph{et al.} \cite{cvpr/ShihKDF15} explicitly modeled multi-surface glass reflections using a double-impulse kernel and a Gaussian Mixture Model. Yet, its spatially invariant assumption restricted its application in wide-angle scenes. To resolve extreme ambiguities, the early 2020s saw the rise of panoramic methods (pink nodes). Approaches like PAR2Net \cite{cvpr/HongZZJKS21,pami/HongZZJKS23} aligned reflection scenes with the contaminated image to guide neural recovery. Despite their high fidelity, heavy reliance on precise geometric alignment and substantial computational overhead limited their practical deployment in real-world environments.

\noindent\textbf{Annotation and Language Guidance-based Methods.}
Represented by purple nodes, user-assisted separation first appeared in the early 2000s \cite{eccv/LevinW04,pami/LevinW07} relying on manual edge annotations. Following a long hiatus, the mid-2020s witnessed a strong revival of this category driven by multimodal learning. Textual prompts \cite{cvpr/ZhongHWLS24}, language-driven diffusion \cite{eccv/hong2024differ}, and contrastive masks unifying visual-textual cues \cite{aaai/chen2025firm} emerged as powerful alternatives for resolving overlapping semantics. While achieving remarkable accuracy aligned with human intent, their inherent reliance on external inputs restricts automation and scalability.

\noindent\textbf{Dataset Contributions.}
High-quality benchmarks (green nodes) have been pivotal since the deep learning boom. The SIR$^2$ dataset \cite{iccv/WanSDTK17} established a standard in 2017 by capturing diverse real-world scenes, while Wen \emph{et al.} \cite{cvpr/WenT0LHH19} reduced the synthetic-to-real gap via nonlinear alpha blending. More recently, Zhu \emph{et al.} \cite{cvpr/zhu2024revisiting} released the large-scale RRW dataset with over 14,950 high-resolution pairs. Hu \emph{et al.} \cite{aaai/hu2026dereflection} further introduced DRR, a diversified 4K dataset collected by rotating reflective media to vary reflection angles and intensities. In addition, synthetic and real-world subsets \cite{iccv/FanYHCW17,cvpr/ZhangNC18a,cvpr/WanSLDK20} remain widely used for evaluating generalization.

\noindent\textbf{Perceptual Information-aided Methods.}
Marked by the red nodes in Fig.~\ref{fig:timeline2}, deep learning methods have become a major force since the late 2010s by learning separation cues beyond hand-crafted assumptions. As summarized by the node shapes, these methods can be roughly grouped into triangular single-stream networks, diamond cascaded/iterative structures, and rectangular or pentagonal dual-stream designs. We next review them from this structural perspective.

\noindent\textbullet \;\; \textbf{Single-stream Structures.}
As marked by the triangular nodes, single-stream CNNs became representative around 2018--2019, when supervised learning began to replace purely hand-crafted optimization pipelines. Zhang \emph{et al.}~\cite{cvpr/ZhangNC18a} introduced perceptual, adversarial, and gradient exclusion losses to suppress reflection edges while recovering plausible transmission content. ERRNet~\cite{cvpr/WeiYFW019} further exploited misaligned real pairs and enlarged contextual perception, reducing the dependence on perfectly registered supervision. Later, Transformer-based designs such as Song \emph{et al.}~\cite{cvpr/song2023robust} improved long-range modeling and robustness to severe reflection patterns. These triangular-node methods established the basic learning-based paradigm, but they mostly allocate the entire network to transmission recovery alone. Without explicit reflection prediction, they cannot fully exploit the mutual constraints between $\mathbf{T}$ and $\mathbf{R}$ in highly ambiguous regions. Recent diffusion-based methods revisit this single-stream paradigm with stronger generative priors. Hu \emph{et al.}~\cite{aaai/hu2026dereflection} proposed DAI, using the diversified DRR dataset and a one-step diffusion framework with progressive training and reflection-invariant fine-tuning to improve in-the-wild robustness. Such a method improves semantic plausibility and restoration quality in difficult scenes, yet it still mainly focus on transmission output and often introduce higher computational cost and larger datasets, motivating more layer-aware architectures.

\noindent\textbullet \;\; \textbf{Cascaded/Iterative Structures.} Peaking around 2018--2022, as indicated by the diamond nodes, cascaded designs explicitly model layer interdependence by decomposing the problem into sequential subproblems. They typically estimate intermediate reflection cues, such as coarse layers, absorption maps, or Laplacian confidence maps, to condition subsequent transmission recovery. For example, BDN~\cite{eccv/YangGLS18} alternates between transmission and reflection estimation, Zheng \emph{et al.}~\cite{cvpr/ZhengSCJDK21} uses absorption-effect prediction to guide the second-stage recovery, Dong \emph{et al.}~\cite{iccv/Dong00BXL21} recurrently estimates reflection-related confidence with Laplacian priors before refining transmission, and Zhang \emph{et al.}~\cite{mm/ZhangSL22} formulates the process as an optimization-unfolding pipeline. While capturing dependencies better than single branch networks, these designs face an optimization dilemma. Reusing one stage for iterations yields suboptimal results since a solitary module struggles to accommodate varying refinement states. Conversely, jointly training multiple stages introduces massive encoding and decoding overhead. Furthermore, exchanging only decoded images prevents effective signal communication in high-dimensional spaces. This bottleneck motivates replacing cascaded subnetworks with stacked interaction modules, allowing both layers to refine and exchange information within the feature space.

\noindent\textbullet \;\; \textbf{Dual-stream Structures.}
From 2020 onwards, dual-stream structures, shown as rectangular nodes in the roadmap, began to estimate transmission and reflection in parallel. IBCLN~\cite{cvpr/LiY0LH20} introduced a dual-branch LSTM framework for separate layer reconstruction. Feng \emph{et al.}~\cite{tip/FengPJCZL21,icme/FengJJP0L21} further explored dual-branch refinement with reflection-guided transmission recovery and contrastive feature supervision. These designs alleviate the single-output limitation and reduce cascaded error accumulation, but the two branches are still mainly maintained side by side, with limited explicit feature-level communication.

\noindent\textbullet \;\; \textbf{Dual-stream Interactive Structures.}
Recognizing the need for deeper information exchange, dual-stream interactive structures, marked by pentagonal nodes, shift the paradigm from isolated parallel pathways to explicit inter-stream communication. YTMT~\cite{nips/HuG21} first routes complementary cues through paired activations, DSRNet~\cite{iccv/Hu23} strengthens this paradigm with mutually gated interaction and a learnable non-linear term, and DSIT~\cite{nips/hu2024single} extends it to Transformers via parallel attention interaction. Subsequent methods further validate this direction from different perspectives: Huang \emph{et al.}~\cite{pami/huang2025lightweight} introduce optimization-based mutual exclusion between streams and Zhao \emph{et al.}~\cite{cvpr/zhao2025reversible} enhance information preservation with reversible encoders. 

\noindent\textbf{Summary.} The past two decades of reflection separation have shown a clear shift from isolated physical priors toward data-driven and increasingly interactive deep architectures. Collectively, this trajectory underscores the importance of explicitly communicating and disentangling deep features to more effectively tackle real-world layer entanglement, motivating our subsequent investigation into learnable reflection modeling and dual-stream interaction.

\begin{table*}[t]
\centering
\caption{Quantitative analysis of reflection models in sRGB domain. We performed patch-based regression on real-world triplets. MSE, AIC, and the coefficient of determination $R^2$ measure the fitting quality. Max Coeff $\sigma$ denotes the maximum standard deviation among learned coefficients, serving as an indicator of numerical stability.}
\label{tab:math_modeling}
\resizebox{\textwidth}{!}{%
\begin{tabular}{llccccc}
\toprule
\textbf{Model Category} & \textbf{Mathematical Formulation} & \textbf{Params} & \textbf{MSE} ($\downarrow$) & \textbf{$R^2$} ($\uparrow$) & \textbf{AIC} ($\downarrow$) & \textbf{Max Coeff $\sigma$} \\ \midrule
\multicolumn{7}{l}{\textit{\textbf{1. Linear Blending}}} \\
Standard Linear & $\mathbf{I} = \mathbf{T} + \mathbf{R}$ & 0 & $1.84 \times 10^{-2}$ & -0.693 & -57,237 & - \\
Alpha Blending & $\mathbf{I} = (1-a)\mathbf{T} + a\mathbf{R}$ & 1 & $4.64 \times 10^{-3}$ & 0.615 & -76,853 & 0.29 \\
Weighted Linear & $\mathbf{I} = a\mathbf{T} + b\mathbf{R}$ & 2 & $3.47 \times 10^{-4}$ & 0.978 & -104,668 & 0.22 \\
Weighted Bias & $\mathbf{I} = a\mathbf{T} + b\mathbf{R} + z$ & 3 & $2.54 \times 10^{-4}$ & 0.984 & -107,562 & 0.23 \\ \midrule
\multicolumn{7}{l}{\textit{\textbf{2. Multiplicative Modeling}}} \\
Coupled (w/o Bias) & $\mathbf{I} = a\mathbf{T} + b\mathbf{R} + ab(\mathbf{T} \circ \mathbf{R})$ & 3 & $5.72 \times 10^{-4}$ & 0.968 & -100,236 & 0.22 \\
Coupled (w/ Bias) & $\mathbf{I} = a\mathbf{T} + b\mathbf{R} + ab(\mathbf{T} \circ \mathbf{R}) + z$ & 3 & $3.04 \times 10^{-4}$ & 0.982 & -105,949 & 0.20 \\
Independent & $\mathbf{I} = a\mathbf{T} + b\mathbf{R} + c(\mathbf{T} \circ \mathbf{R}) + z$ & 4 & $1.79 \times 10^{-4}$ & 0.987 & -110,773 & 1.40 \\ \midrule
\multicolumn{7}{l}{\textit{\textbf{3. High-Order Polynomials}}} \\
Second-Order & $\mathbf{I} = a\mathbf{T} + b\mathbf{R} + c(\mathbf{T} \circ \mathbf{R}) + d\mathbf{T}^2 + e\mathbf{R}^2 + z$ & 6 & $1.54 \times 10^{-4}$ & 0.988 & -112,658 & 21.2 \\
Third-Order & $\mathbf{I} = \text{2nd-Order} + f\mathbf{T}^3 + g\mathbf{T}^2\circ\mathbf{R} + \dots + z$ & 10 & $1.46 \times 10^{-4}$ & 0.989 & -113,415 & $2.9 \times 10^3$ \\
Fourth-Order & $\mathbf{I} = \text{3rd-Order} + j\mathbf{T}^4 + k\mathbf{T}^3\circ\mathbf{R} + \dots + z$ & 15 & $1.42 \times 10^{-4}$ & 0.989 & -113,874 & $1.8 \times 10^6$ \\ \midrule
\multicolumn{7}{l}{\textit{\textbf{4. Physical Approximation}}} \\
Gamma Approx. & $\mathbf{I} = (a\mathbf{T}^{2.2} + b\mathbf{R}^{2.2} + z)^{1/2.2}$ & 3 & $3.98 \times 10^{-4}$ & 0.970 & -102,541 & 0.18 \\ 
sRGB Physical & $\mathbf{I} = \mathcal{P}_{\text{sRGB}}( a \mathcal{P}^{-1}(\mathbf{T}) + b \mathcal{P}^{-1}(\mathbf{R}) + z)$ & 3 & $3.68 \times 10^{-4}$ & 0.974 & -103,674 & 0.18 \\
\bottomrule
\end{tabular}%
}
\vspace{-10pt}
\end{table*}

\section{Learnable Nonlinear Modeling}
\label{sec:modeling}

Unlike the RAW domain where photon accumulation is (approximately) linear, sRGB images undergo a complex, non-linear ISP processing (\emph{e.g.}, gamma correction, tone mapping, and clipping). Consequently, the superimposed image $\textbf{I}$ is no longer a simple summation of the transmission layer $\textbf{T}$ and the reflection layer $\textbf{R}$, but further incorporates a bias term and higher-order non-linear residuals.

\begin{proposition}[Linear Superposition Bias in sRGB Space]
Let the RAW-domain reflection formation be additive:
\begin{equation}
    \mathbf{I}_\textup{raw} = \mathbf{T}_\textup{raw} + \mathbf{R}_\textup{raw}, \nonumber
\end{equation}
and let the observed sRGB images be produced by an ISP mapping function $\mathcal{F}_{\textup{ISP}}$:
\begin{equation}
   \mathbf{I} = \mathcal{F}_{\textup{ISP}}(\mathbf{T}_\textup{raw} + \mathbf{R}_\textup{raw}),\ \
\mathbf{T} = \mathcal{F}_{\textup{ISP}}(\mathbf{T}_\textup{raw}),\ \
\mathbf{R} = \mathcal{F}_{\textup{ISP}}(\mathbf{R}_\textup{raw}).
\nonumber
\end{equation}
If $\mathcal{F}_{\textup{ISP}}$ is nonlinear, then, in general, there do not exist fixed constants \(a,b,z\) such that
$\mathbf{I} = a\mathbf{T}+b\mathbf{R}+z$
holds for all admissible pairs $(\mathbf{T}_\textup{raw}, \mathbf{R}_\textup{raw})$.
Therefore, a fixed linear superposition model in sRGB space is intrinsically biased.
\label{prop1}
\end{proposition}

\begin{proof}
The proof is given in Appendix~\ref{sec:append_proofs}.
\end{proof}

\subsection{Revisiting Reflection Superimposition Models}
To mathematically formulate the sRGB superimposition relationship, we analyze a collection of localized, spatially aligned patch triplets $\{(\mathbf{T}_k, \mathbf{R}_k, \mathbf{I}_k)\}_{k=1}^N$, with normalized intensities in $[0,1]$. In real-world scenarios, the physical blending factors (\emph{e.g.}, reflectance and transmittance) vary spatially across the image due to changing incident angles and complex scene geometries~\cite{pami/KongTS14}. To render the modeling tractable, we introduce a local stationarity assumption: within a sufficiently small spatial window, these coefficients can be approximated as constants. Under this premise, the regional superimposition can be abstracted as a parameterized bivariate mapping $\hat{\mathbf{I}} = \mathcal{I}(\mathbf{T}, \mathbf{R}; \Theta)$, where $\Theta$ encapsulates the locally stationary coefficients. To identify the optimal analytical form of $\mathcal{I}$, we systematically categorize and evaluate four families of mathematical models:

\noindent\textbf{Linear Blending.} This category assumes the superposition remains linear despite ISP processing. It investigates three variants: the vanilla Standard Linear summation $\mathbf{I} = \mathbf{T} + \mathbf{R}$, the constraint-based Alpha Blending $\mathbf{I} = (1-a)\mathbf{T} + a\mathbf{R}$, and the generalized Weighted Linear model with an optional bias  $\mathbf{I} = a\mathbf{T} + b\mathbf{R} + z$.
    
\noindent\textbf{Multiplicative Modeling.} To address non-linearity, we introduce a second-order multiplicative term $\mathbf{T} \circ \mathbf{R}$. We evaluate a Coupled form where the multiplicative coefficient is tied to linear weights $ab$, and an Independent form $\mathbf{I} = a\mathbf{T} + b\mathbf{R} + c(\mathbf{T} \circ \mathbf{R}) + z$ with an unconstrained coefficient $c$ for the multiplicative term.
    
\noindent\textbf{High-Order Polynomials.} We extend the expansion to the full 2nd, 3rd, and 4th orders (\emph{e.g.}, including quadratic terms $\mathbf{T}^2, \mathbf{R}^2$, cubic interactions  like $\mathbf{T}^2\circ\mathbf{R}$, and quartic  $\mathbf{T}^2\circ\mathbf{R}^2$), serving as a theoretical upper bound to probe polynomial fitting and numerical stability.
    
\noindent\textbf{Physical Approximation.} These models attempt to reverse the non-linearity using Gamma curves. We test both a simplified power-law approximation ($\mathbf{I} = (a\mathbf{T}^{\gamma} + b\mathbf{R}^{\gamma} + z)^{1/\gamma}$) and the standard sRGB transfer function ($\mathcal{P}_{\text{sRGB}}$) to verify whether such explicit inverse mappings can sufficiently restore the linear superposition property.

\subsection{Empirical Regression Analysis}
To validate these hypotheses, we conducted the patch-based regression analysis on the $\textrm{SIR}^2$ dataset, containing real I-T-R triplets. We evaluated the fitting quality using Mean Squared Error (MSE), Akaike Information Criterion (AIC), the Coefficient of Determination $R^2$. The Max Coefficient Standard Deviation $\sigma$ is also monitored to detect numerical instability. The quantitative results are presented in Table~\ref{tab:math_modeling}.

\noindent\textbf{Observation 1: Limitations of Linear Models.} 
As reported in Table~\ref{tab:math_modeling}, the {Standard Linear} model exhibits the highest Mean Squared Error (MSE) of $1.84 \times 10^{-2}$, revealing a clear discrepancy between the idealistic additive assumption and actual sRGB observations. 
Notably, this model yields a negative $R^2$ value (-0.693), which measures the proportion of variance in the data that is predictable from the model, indicating  a significant deviation from the sRGB data distribution.
Even with learnable scaling and offsets, the {Weighted Bias} model remains inferior to non-linear alternatives.
\textit{These observations imply that ISP transformations break the linear assumption, prompting us to incorporate higher-order interactions to represent the complex layer coupling effects.}

\noindent\textbf{Observation 2: The Necessity of Interaction.}
As reported in Table~\ref{tab:math_modeling}, the {Independent} model (MSE $1.79 \times 10^{-4}$) not only surpasses linear baselines but also significantly outperforms the {sRGB Physical} model ($3.68 \times 10^{-4}$). This superiority suggests that the sRGB non-linearity is driven by the synergistic coupling term $\mathbf{T} \circ \mathbf{R}$ rather than a global Gamma mapping. 
Furthermore, the superiority of the independent coefficient $c$ over the coupled $ab$ indicates that this multiplicative interaction behaves as a distinct physical factor rather than a derivative of the linear base.
\textit{These findings reveal the inherent non-linearity of the reflection superimposition model in the sRGB space and justify the independent relationship between the multiplicative coefficient and the linear weights.}

\begin{figure}[t]
    \begin{subfigure}[b]{0.2424\linewidth}
        \centering
        \includegraphics[width=\linewidth]{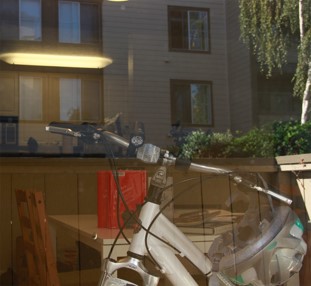}
        \\[2pt]
        \includegraphics[width=\linewidth]{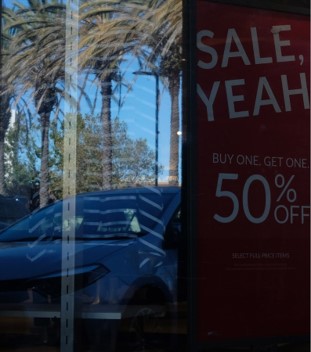}
        \caption*{$\textbf{I}$}
    \end{subfigure}
    \hfill
    \begin{subfigure}[b]{0.2424\linewidth}
        \centering
        \includegraphics[width=\linewidth]{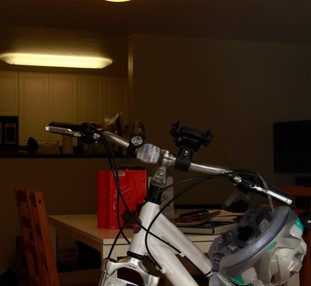}
        \\[2pt]
        \includegraphics[width=\linewidth]{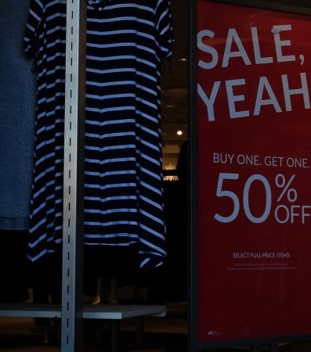}
        \caption*{$\textbf{T}$}
    \end{subfigure}
    \hfill
    \begin{subfigure}[b]{0.2424\linewidth}
        \centering
        \includegraphics[width=\linewidth]{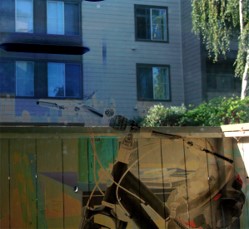}
        \\[2pt]
        \includegraphics[width=\linewidth]{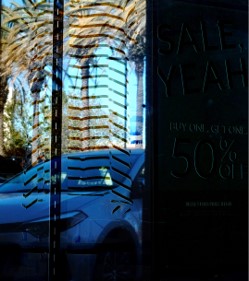}
        \caption*{$\textbf{I}-\textbf{T}$}
    \end{subfigure}
    \hfill
    \begin{subfigure}[b]{0.2424\linewidth}
        \centering
        \includegraphics[width=\linewidth]{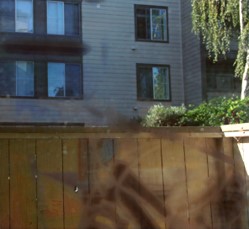}
        \\[2pt]
        \includegraphics[width=\linewidth]{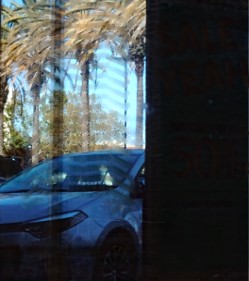}
        \caption*{\smash{$\mathcal{F}_C(\mathbf{I}, \mathbf{T})$}}
    \end{subfigure}
 
  \caption{Qualitative results of reflection ground truth completion. The direct subtraction $\mathbf{I} - \mathbf{T}$ introduces severe structural degradation. By decoupling non-linearities and biases, our method $\mathcal{F}_C(\mathbf{I}, \mathbf{T})$ yields a purer reflection layer.}
  \label{fig:refcomp}
  \vspace{-10pt}
\end{figure}

\noindent\textbf{Observation 3: The Instability of High-Order Polynomials.}
Extending the analytical expansion to higher orders in Table~\ref{tab:math_modeling} reveals a severe trade-off between fitting fidelity and numerical robustness. Specifically, while fidelity metrics such as MSE and AIC exhibit marginal improvements up to the 4th-order, these minor gains are accompanied by an astronomical explosion in coefficient variance ($\sigma = 1.8 \times 10^6$). This demonstrates that simply escalating the polynomial order to achieve more accurate fitting is unviable, as it inevitably leads to catastrophic numerical instability. By comparison, the second-order \textit{Independent} model serves as a conservative and stable baseline, validating its use for reliable data synthesis.
\textit{These results indicate that static high-order templates are insufficient for the spatially varying complexity of sRGB images, necessitating a learnable, adaptive approach to harness high-order dynamics without numerical collapse.}

\subsection{From Analysis to Modeling}
\label{sec:anal2model}

The above analysis highlights a critical challenge, \emph{i.e.}, existing linear synthesis pipelines \cite{iccv/FanYHCW17, cvpr/WenT0LHH19} fail to capture inter-layer coupling, and thus suffer from severe synthetic-to-real domain gaps. To address this, we propose a unified Synthesis–Network–Real (S–N–R) framework that bridges analytical modeling and data-driven learning:
\begin{equation}
\renewcommand{\arraystretch}{2.2} 
\begin{array}{r@{\hskip 4pt}c@{\hskip 2pt}c@{\hskip 4pt}c@{\hskip 4pt}c@{\hskip 4pt}c@{\hskip 4pt}c@{\hskip 4pt}c}
    \textbf{Syn:} & \mathbf{I}_{\textrm{syn}} & = & \overbrace{a\mathbf{T} + b\mathbf{R}}^{\text{Weighted Base}} & + & \underbrace{c(\mathbf{T} \circ \mathbf{R})}_{\text{Interaction}} & + & \underbrace{z}_{\text{Bias}}\,, \\
    \textbf{Net:} & \hat{\mathbf{I}}_{\textrm{net}} & = & \underbrace{\mathbf{T} + \mathbf{R}}_{\text{Unit Base}} & + & \boldsymbol{\Phi}(\mathbf{T}, \mathbf{R}) & + & \boldsymbol{\Psi}\,, \\
    \textbf{Real:} & \mathbf{I}_\textrm{real} & \approx & \underbrace{w_{10}\mathbf{T} + w_{01}\mathbf{R}}_{\text{Weighted Base}} & + & \overbrace{\sum_{i,j} w_{ij}\mathbf{T}^i \circ \mathbf{R}^j}^{\text{High-order Terms}} & + & \overbrace{\epsilon}^{\text{Offset}} \,.
\end{array}
\label{eq:unified_model}
\end{equation}

\noindent\textbf{Synthesis Strategy (Syn):} We adopt the \textit{Independent} multiplicative model for data synthesis, enabling stable generation of plausible triplets without numerical collapse.

\noindent\textbf{Real-world Complexity (Real):} In practice, ISP pipelines introduce high-order nonlinearities that can be expressed as a series expansion $\sum w_{ij} \mathbf{T}^i \circ \mathbf{R}^j$. As indicated by the regression analysis, directly modeling such expansions is infeasible due to instability and variability, leading to a notable synthetic-to-real gap.

\noindent\textbf{Network Architecture (Net):} To bridge this gap within a single unified model, we introduce the Learnable Offset-Residual Superposition (LORS) model. Built upon a linear unit base, LORS incorporates two learnable terms, say the Learnable Non-linear Residual (LNR) $\boldsymbol{\Phi}(\mathbf{T}, \mathbf{R})$ to absorb complex synthetic and real-world non-linearities, and the Learnable Zero-order Offset (LZO) $\boldsymbol{\Psi}$ to explicitly model ambient biases. By isolating these couplings, LORS mitigates the ISP-induced domain gap and improves generalization from synthetic training data to real-world observations.

\begin{figure*}[t]
  \includegraphics[width=\textwidth]{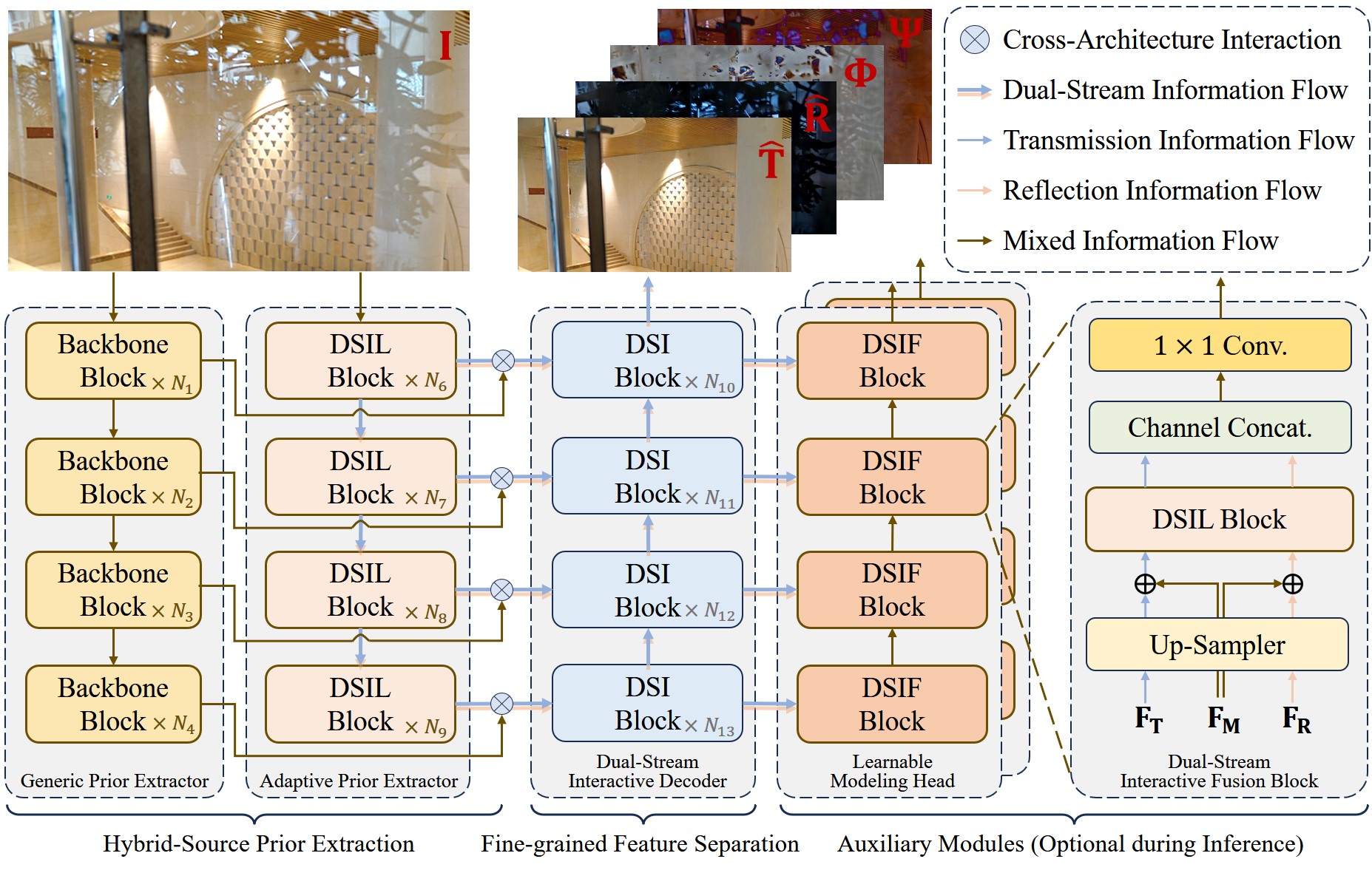}
  \caption{The proposed unified Dual-stream Interactive Reflection Separation (DIRS) architecture, embracing (a) Hybrid-Source Prior Extraction via a pretrained Generic Prior Extractor (GPE) and a task-specific Adaptive Prior Extractor (APE); (b) Fine-grained Feature Separation using a Dual-Stream Interactive Decoder equipped with customizable DSI Blocks; and (c) Auxiliary Modules comprising DSIF Blocks for high-order residual $\boldsymbol{\Phi}$ and zero-order offset $\boldsymbol{\Psi}$ modeling.}
  \vspace{-10pt}
  \label{fig:arch}
\end{figure*}

\begin{figure*}[t]
  \includegraphics[width=\textwidth]{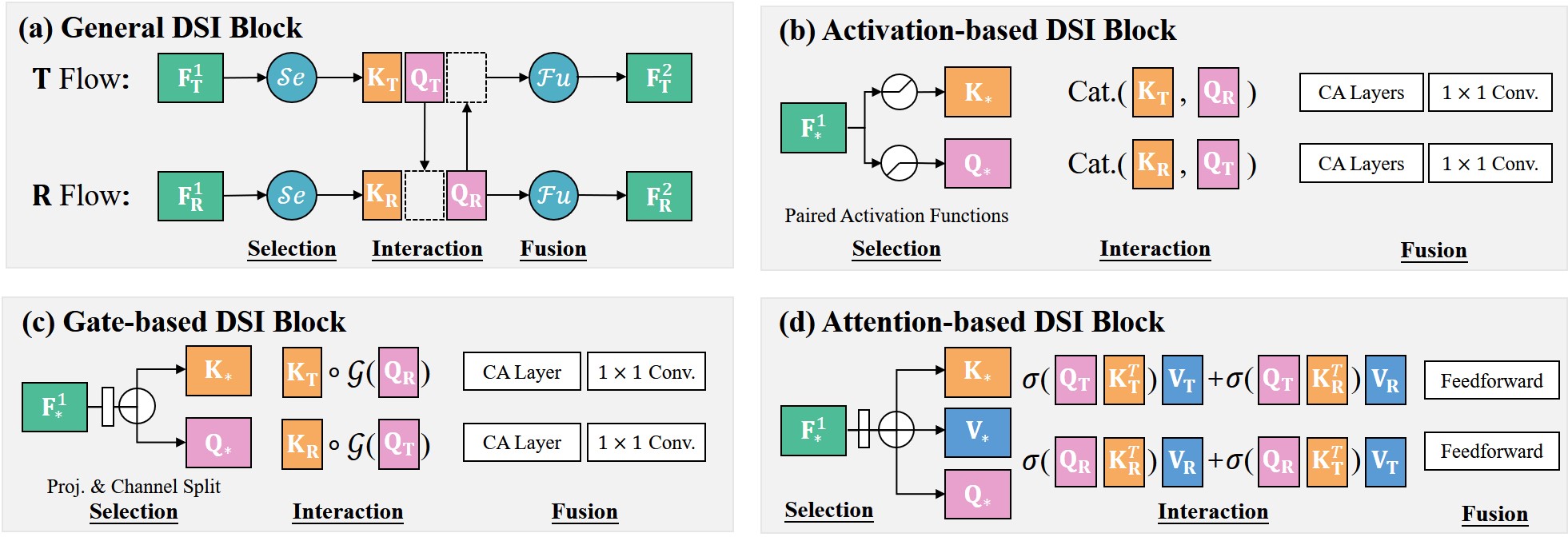}
  \caption{Our proposed Dual-Stream Interactive Blocks (DSI Blocks). (a) General formulation abstracting interactive modules into a unified Selection, Interaction, and Fusion (SIF) paradigm to update the dual-stream features. (b)-(d) Three specific variants based on activation, gate, and attention mechanisms, all adhering to this SIF paradigm.}
  \label{fig:fib}
\end{figure*}

\subsection{Reflection Ground Truth Completion}
\label{sec:ref_comp}

Due to the immense difficulty of isolating reflections during physical acquisition, mainstream real-world datasets~\cite{cvpr/ZhangNC18a,cvpr/LiY0LH20,cvpr/zhu2024revisiting,aaai/hu2026dereflection} provide only bipartite pairs $(\mathbf{I}_\textrm{real}, \mathbf{T})$ rather than complete triplets. To circumvent this, previous dual-layer methods conventionally adopt the residual $\mathbf{I}_\textrm{real} - \mathbf{T}$ as pseudo-supervision~\cite{cvpr/ZhangNC18a,cvpr/LiY0LH20,cvpr/song2023robust}. However, this forces the predicted reflection layer to implicitly absorb high-order couplings and biases. Such entanglement severely corrupts the structural integrity of the reflection, leading to local degradation and gradient artifacts.

To resolve this dilemma, an independent pre-training phase is required to learn and predict the pseudo-reflection layers before formal training. Let $\mathcal{F}_C: (\mathbf{I}, \mathbf{T}) \rightarrow (\hat{\mathbf{T}}, \hat{\mathbf{R}}_\textrm{comp})$ denote the reflection completion mapping, where $\hat{\mathbf{T}}$ is jointly predicted strictly to leverage dual-stream constraints. We utilize our unified \textbf{Net} model as a physics-informed bridge to elegantly resolve this entanglement by minimizing the following reconstruction objective:
\begin{equation}
    \mathcal{L}_\textrm{comp} = \left\| \mathbf{I}_\textrm{real} - \Big(\hat{\mathbf{T}} + \hat{\mathbf{R}}_\textrm{comp} + \boldsymbol{\Phi}(\hat{\mathbf{T}}, \hat{\mathbf{R}}_\textrm{comp}) + \boldsymbol{\Psi}\Big) \right\|_1.
\end{equation}
Since the ground-truth $\mathbf{T}$ is provided as input for both synthetic and real data, the network anchors $\hat{\mathbf{T}} \rightarrow \mathbf{T}$. Crucially, because the synthetic dataset provides structurally complete reflection supervision, jointly learning the mappings $(\mathbf{I}_{\textrm{syn}}, \mathbf{T}) \rightarrow (\mathbf{T}, \mathbf{R})$ and $(\mathbf{I}_\textrm{real}, \mathbf{T}) \rightarrow (\mathbf{T}, \mathbf{I}_\textrm{real} - \mathbf{T})$ establishes a strong linear prior. Driven by this prior, the LNR module $\boldsymbol{\Phi}$ and LZO module $\boldsymbol{\Psi}$ actively absorb the non-linearities and biases. Consequently, the mathematical expectation of the completed reflection layer shifts to:
\begin{equation}
    \mathbb{E}[\hat{\mathbf{R}}_\textrm{comp}] \approx \mathbf{R}^* + \frac{1}{\eta}\Big(\sum_{i,j} w_{ij}\mathbf{T}^i \circ (\mathbf{R}^*)^j + \epsilon\Big)
\end{equation}
where $\mathbf{R}^*$ is the latent true reflection, and $\eta > 1$ serves as an attenuation factor governed by the ratio of synthetic to real images during training. This mechanism ensures that we obtain a purer reflection prediction on real-world data, as visually compared in Fig.~\ref{fig:refcomp}. By attenuating the non-linear terms and biases, this physics-guided disentanglement generates structurally plausible and cleaner pseudo-triplets $(\mathbf{I}_\textrm{real}, \mathbf{T}, \hat{\mathbf{R}}_\textrm{comp})$ for real-world datasets, establishing a consistent supervisory foundation for the subsequent formal training phase of the proposed architecture.

\section{Dual-stream Interactive Designs} 
Based on the physical formulation established in Sec.~\ref{sec:modeling}, we here formally motivate our dual-stream network design for reflection separation, denoted by $\mathcal{F}_S$.

\begin{proposition}[Nonlinear Formation Necessitates Coupled Estimation]
Let the observed sRGB image be formed through a real-world Image Signal Processing pipeline $\mathcal{F}$ applied to raw physical irradiances: $\mathbf{I} = \mathcal{F}(\mathbf{T}_{\textup{raw}} + \mathbf{R}_{\textup{raw}})$. Modeling this formation equivalently in the sRGB domain yields:
\begin{equation}
    \mathbf{I} =\mathbf{T} +\mathbf{R} + \boldsymbol{\Phi}(\mathbf{T},\mathbf{R}) + \boldsymbol{\Psi}, \nonumber
\end{equation}
where $\mathbf{T}$ and $\mathbf{R}$ denote the transmission and reflection layers, $\boldsymbol{\Phi}(\mathbf{T},\mathbf{R})$ is a nonlinear interlayer residual term, and $\boldsymbol{\Psi}$ is a zero-order offset. For typical nonlinear ISP functions with non-zero mixed responses, the term $\boldsymbol{\Phi}$ is generally nonseparable (namely, it cannot be strictly decomposed into independent functions $\boldsymbol{\Phi}_\mathbf{T}$ and $\boldsymbol{\Phi}_\mathbf{R}$ such that $\boldsymbol{\Phi}(\mathbf{T},\mathbf{R})=\boldsymbol{\Phi}_\mathbf{T}(\mathbf{T})+\boldsymbol{\Phi}_{\mathbf{R}}(\mathbf{R})$). From an optimization perspective, this nonseparability creates a coupled loss landscape. Consequently, the estimation of $\mathbf{T}$ and $\mathbf{R}$ cannot be reduced to independent subproblems; explicit interlayer feature interaction is required to approximate the joint inference process.
\end{proposition}

\begin{proof}
The proof is given in Appendix~\ref{sec:append_proofs}.
\end{proof}

\subsection{Dual-Stream Interactive Architecture}
\label{sec:ds_arch}
Figure~\ref{fig:arch} depicts the generalized Dual-stream Interactive Reflection Separation (DIRS) architecture, which consolidates our previous designs (YTMT~\cite{nips/HuG21}, DSRNet~\cite{iccv/Hu23}, and DSIT~\cite{nips/hu2024single}) into a unified framework. The overall architecture is structured into three main stages:

\noindent\textbf{Hybrid-source Prior Extraction.} To capture both deep semantic priors and fine-grained spatial details, we employ an asymmetric dual-branch encoder mechanism. The first branch is a Generic Prior Extractor (GPE), which utilizes a frozen, high-capacity backbone pretrained on large-scale datasets (\emph{e.g.}, ImageNet or Object365~\cite{iccv/shao2019objects365}) to extract universal, semantically rich mixed information flows, denoted as $\{\mathbf{F}_{\mathbf{M}}^l\}_{l=1}^L$. Simultaneously, the second branch acts as a trainable Adaptive Prior Extractor (APE), which maintains separate transmission and reflection information flows optimized specifically for the separation task. 
To enrich the task-specific representations with deep semantics, we introduce a Cross-Architecture Interaction mechanism ($\otimes$ in Fig.~\ref{fig:arch}). At each level $l$, the generic mixed feature $\mathbf{F}_{\mathbf{M}}^l$ is explicitly injected into the APE, yielding the semantically enriched, dual-stream encoded features: $\{\mathbf{F}_{\mathbf{T}, \textrm{Enc}}^l, \mathbf{F}_{\mathbf{R}, \textrm{Enc}}^l\}_{l=1}^L$.

\noindent\textbf{Fine-grained Feature Separation.} In the second stage, the encoded dual-stream features are fed into a Dual-Stream Interactive Decoder ($\mathcal{F}_{\textrm{Dec}}$). This decoder performs progressive spatial reconstruction and component disentanglement from the coarsest to the finest level. Specifically, at each level $l$, we first fuse the encoded features with the upsampled decoded features from the previous level, denoted as $\tilde{\mathbf{F}}_{\mathbf{T}}^l = \mathbf{F}_{\mathbf{T}, \textrm{Enc}}^l \oplus \textrm{Up}(\mathbf{F}_{\mathbf{T}, \textrm{Dec}}^{l-1})$ and $\tilde{\mathbf{F}}_{\mathbf{R}}^l = \mathbf{F}_{\mathbf{R}, \textrm{Enc}}^l \oplus \textrm{Up}(\mathbf{F}_{\mathbf{R}, \textrm{Dec}}^{l-1})$. Then, the component disentanglement is formulated as:
\begin{equation}
    \mathbf{F}^{l}_{\mathbf{T}, \textrm{Dec}}, \mathbf{F}^{l}_{\mathbf{R}, \textrm{Dec}} = \textrm{DSI-Block}(\tilde{\mathbf{F}}_{\mathbf{T}}^l, \tilde{\mathbf{F}}_{\mathbf{R}}^l),
\end{equation}
where $\textrm{Up}(\cdot)$ denotes upsampling, $\oplus$ denotes element-wise addition, and DSI-Block represents the task-specific Dual-Stream Interactive Block (detailed in Sec.~\ref{sec:ds_blocks}). For the coarsest level ($l=1$), the separation is performed directly using the encoded features without cross-level fusion. At the highest resolution ($l=L$), simple projection layers are applied to generate the clean predictions, $\hat{\mathbf{T}}$ and $\hat{\mathbf{R}}$.

\noindent\textbf{Auxiliary Modules for LORS Modeling.} To explicitly model the non-linear residual $\boldsymbol{\Phi}$ and zero-order offset $\boldsymbol{\Psi}$, two auxiliary Learnable Modeling Heads ($\mathcal{F}_{\textrm{Aux}}$) are placed parallel to the main decoder. These heads are constructed by stacking Dual-Stream Interactive Fusion (DSIF) Blocks, as detailed in the rightmost panel of Fig.~\ref{fig:arch}. At every scale, each head integrates the accumulated auxiliary features from the previous scale ($\mathbf{F}_{\mathbf{M}}$) with the current-scale transmission ($\mathbf{F}_{\mathbf{T}}$) and reflection ($\mathbf{F}_{\mathbf{R}}$) representations. Through upsampling, interactive fusion via DSIL blocks, channel-wise concatenation, and $1\times1$ convolutional projection, this module continuously extracts and fuses the target information to generate the respective predictions for $\boldsymbol{\Phi}$ and $\boldsymbol{\Psi}$.  Notably, because these auxiliary modules are designed to act as physical absorbers during training, they are \textit{entirely optional and can be safely discarded during inference without performance loss}, ensuring an efficient deployment.

Depending on deployment requirements and hardware constraints, the DIRS framework is flexibly configured into three network designs: (1) an \textit{activation-based CNN}, which employs a VGGNet-based GPE alongside YTMT blocks for both encoding (DSIL) and decoding (DSI) to fuse pretrained priors with task-specific representations; (2) a \textit{gate-based CNN}, which inherits the CNN GPE but upgrades to MuGI blocks throughout, enabling explicit multiplicative interactions and unimpeded information flow; and (3) an \textit{attention-based Transformer}, which upgrades the GPE to a Swin Transformer for global priors, integrating MuGI blocks during encoding to extract local details and PAIR blocks during decoding for non-local interactions.

In summary, the DIRS framework provides a structural template for different interaction mechanisms, allowing one to flexibly select the architectural configuration based on target performance and hardware bottlenecks. The formulations of these mechanisms are detailed in Sec.~\ref{sec:ds_blocks}.

\subsection{Dual-Stream Interactive Blocks}
\label{sec:ds_blocks}

To formally unify the foundational paradigms, we introduce the General Dual-Stream Interactive Block (General DSI Block). As illustrated in Fig.~\ref{fig:fib}~(a), it abstracts feature interaction into three standardized operations: Selection, Interaction, and Fusion (SIF). Specifically, the Selection operator explicitly decouples input features $\mathbf{F}_{*}^{1}$ ($* \in \{\mathbf{T}, \mathbf{R}\}$) into intra-stream retained features $\mathbf{K}_{*}$ and cross-stream interactive features $\mathbf{Q}_{*}$. Subsequently, the $\mathbf{K}$ components are preserved locally, while the $\mathbf{Q}$ components are exchanged and processed via interaction operators to model inter-stream dependencies. Finally, a Fusion operator integrates these recomposited features to yield the updated outputs $\mathbf{F}_{*}^{2}$. We detail three specific variants below.

\begin{figure*}[ht]
\centering
\begin{subfigure}[b]{0.32\textwidth}
    \centering
    \begin{minipage}[c][3.8cm][c]{\linewidth}
        \centering
        \definecolor{pafcolor}{HTML}{CC0000}
        \definecolor{nafcolor}{HTML}{4472C4}
        
        \begin{tikzpicture}
            \begin{axis}[
                width=2.4cm, height=2.4cm, scale only axis, 
                xlabel={$x$}, ylabel={$y$}, axis lines=middle, 
                domain=-8:8, samples=100,
                xtick=\empty, ytick=\empty,
                xmin=-9, xmax=9, ymin=-9, ymax=9, 
                clip=false 
            ]
                \addplot[line width=0.5mm, color=pafcolor] {max(0,x)};
            \end{axis}
            \node[below] at (1.2,-0.5) {\footnotesize PAF: $\text{ReLU}(x)$}; 
        \end{tikzpicture}
        \hspace{0.1cm}
        \begin{tikzpicture}
            \begin{axis}[
                width=2.4cm, height=2.4cm, scale only axis,
                xlabel={$x$}, ylabel={$y$}, axis lines=middle, 
                domain=-8:8, samples=100,
                xtick=\empty, ytick=\empty,
                xmin=-9, xmax=9, ymin=-9, ymax=9, 
                clip=false 
            ]
                \addplot[line width=0.5mm, color=nafcolor] {x-max(0,x)};
            \end{axis}
            \node[below] at (1.2,-0.5) {\footnotesize NAF: $x - \text{ReLU}(x)$};
        \end{tikzpicture}
    \end{minipage}
    \caption{Visualizations of Paired Activations}
    \label{fig:ytmt_vis}
\end{subfigure}
\hfill
\begin{subfigure}[b]{0.32\textwidth}
    \centering
    \begin{minipage}[c][3.8cm][c]{\linewidth}
        \centering
        \includegraphics[width=0.85\linewidth]{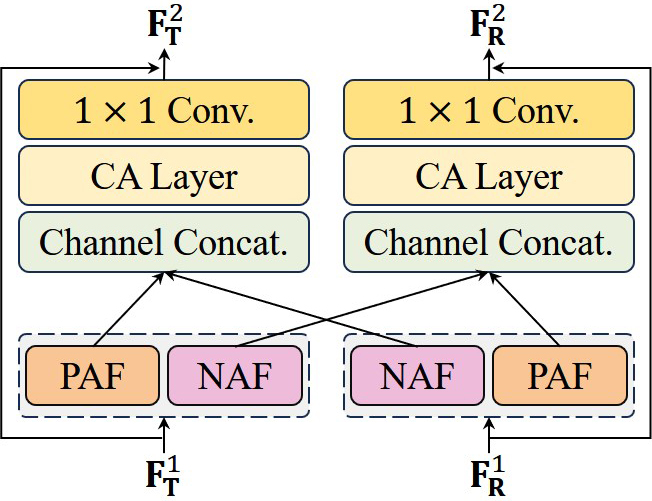}
    \end{minipage}
    \caption{Architecture of the YTMT Block}
    \label{fig:ytmt_arch}
\end{subfigure}
\hfill
\begin{subfigure}[b]{0.32\textwidth}
    \centering
    \begin{minipage}[c][3.8cm][c]{\linewidth}
        \centering
        \begin{tabular}{lcc} 
            \toprule
            PAF ($\delta^+$) & PSNR & SSIM \\
            \midrule
            ReLU   & \textbf{24.94} & \textbf{0.902} \\
            PReLU  & 24.66 & 0.892  \\
            SiLU   & 24.86 & 0.897  \\
            GeLU   & 24.74 & 0.900 \\
            Hswish & 24.88 & 0.902 \\
            \bottomrule
        \end{tabular}
    \end{minipage}
    \caption{Ablation on PAF choices}
    \label{fig:ytmt_table}
\end{subfigure}
\caption{(a) The ReLU pair (PAF and NAF). (b) Illustration of the YTMT Block. (c) Ablation on various PAF ($\delta^+$) choices evaluated on Real20 and SIR$^2$ datasets.} 
\label{fig:activation_results}
\end{figure*}

\begin{figure*}[ht]
\centering

\begin{subfigure}[b]{0.32\textwidth}
    \centering
    \begin{minipage}[c][4.8cm][c]{\linewidth}
        \centering
        \includegraphics[width=0.95\linewidth]{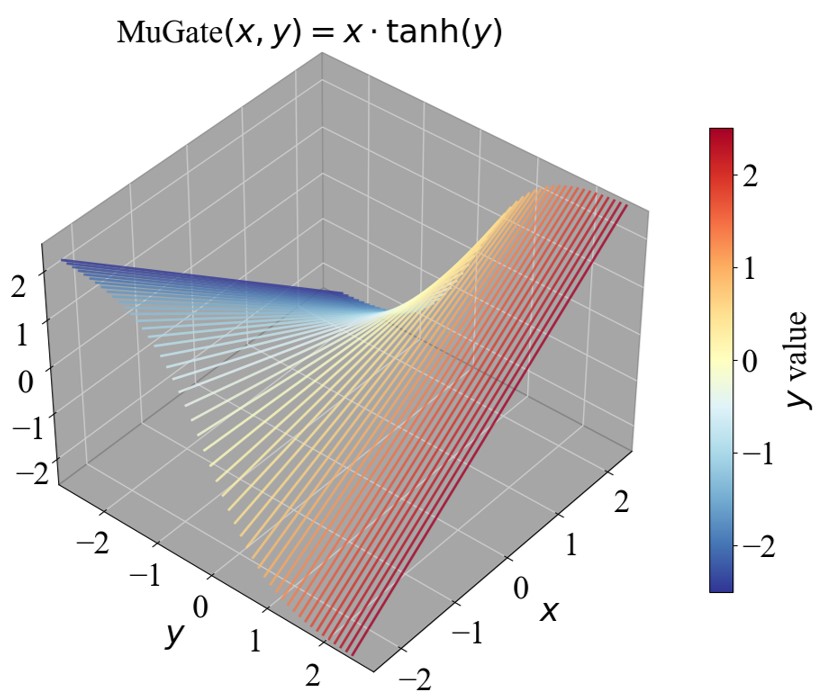}
    \end{minipage}
    \caption{Visualizations of gating functions}
    \label{fig:mugi_vis}
\end{subfigure}
\hfill
\begin{subfigure}[b]{0.32\textwidth}
    \centering
    \begin{minipage}[c][4.8cm][c]{\linewidth}
        \centering
        \includegraphics[height=4.8cm, keepaspectratio]{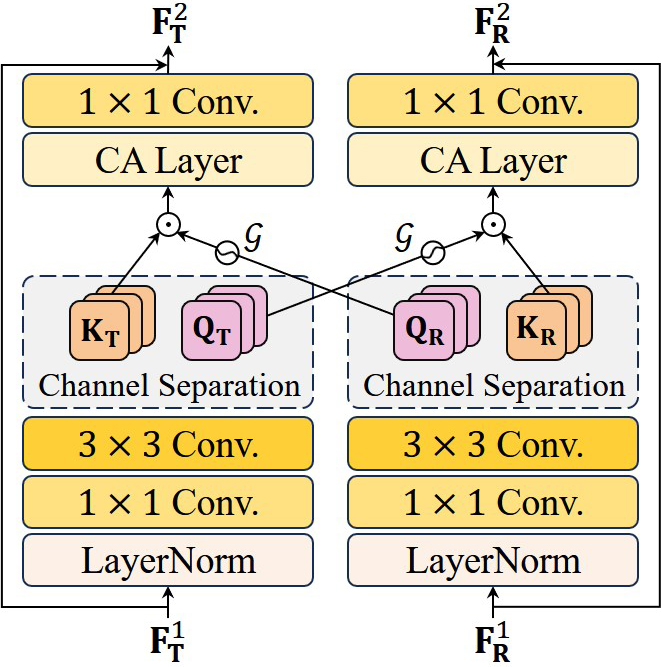} 
    \end{minipage}
    \caption{Architecture of the MuGI Block}
    \label{fig:mugi_arch}
\end{subfigure}
\hfill
\begin{subfigure}[b]{0.32\textwidth}
    \centering
    \begin{minipage}[c][4.8cm][c]{\linewidth}
        \centering
        \begin{tabular}{lcc}
            \toprule
            Gate ($\mathcal{G}$) & PSNR & SSIM \\
            \midrule
            Identity & 25.40 & 0.905 \\
            Sign      & 25.06 & 0.902  \\
            Tanh     & 25.50 & 0.907 \\
            GLU      & 25.19 & 0.905  \\
            GeLU     & 25.41 & 0.912  \\
            ReGLU     & \textbf{25.67} & 0.907 \\
            GeGLU     & 25.46 & 0.908 \\
            SwiGLU    & 25.63 & \textbf{0.913} \\
            \bottomrule
        \end{tabular}
    \end{minipage}
    \caption{Ablation on gate choices}
    \label{fig:mugi_table}
\end{subfigure}

\caption{(a) The gating functions in 3D feature space. (b) Illustration of the MuGI Block. (c) Ablation on various Gate ($\mathcal{G}$) choices evaluated on Real20 and SIR$^2$ datasets.}
\label{fig:mugi_gates}
\end{figure*}

\vspace{1mm}\noindent\textbf{Activation-based DSI Block.}
Standard activations suppress negative feature domains, causing severe information loss for the inherently weak reflection branch~\cite{icmlw/maas2013rectifier}. Since our LORS formulation safely bypasses complex high-order semantics and biases outside the decoder, we can abstract layer decomposition within DSI Blocks as an approximate dichotomous separation. Thus, we introduce a paired activation strategy. As illustrated in Fig.~\ref{fig:fib}~(b), the Selection operator utilizes a Positively Activated Function (PAF, $\delta^+$) to retain components $\mathbf{K}_* = \delta^+(\mathbf{F}^1_*)$, and an origin-symmetric Negatively Activated Function (NAF, $\delta^-$) to extract suppressed components $\mathbf{Q}_* = \delta^-(\mathbf{F}^1_*)$ as visualized in Fig.~\ref{fig:activation_results}~(a). Following the SIF paradigm, the $\mathbf{Q}$ components are exchanged and concatenated with the preserved $\mathbf{K}$ components of the opposite stream. The fusion operator, comprising Channel Attention (CA) and a $1\times 1$ convolution, integrates these recomposited features. Formally, the Activation-based DSI Block is expressed as: 
\begin{equation}
\begin{aligned}
    \mathbf{F}^2_{\mathbf{T}} &= \text{Conv}_{1\times 1}(\text{CA}([\delta^+(\mathbf{F}^1_{\mathbf{T}}), \delta^-(\mathbf{F}^1_{\mathbf{R}})])), \\
    \mathbf{F}^2_{\mathbf{R}} &= \text{Conv}_{1\times 1}(\text{CA}([\delta^+(\mathbf{F}^1_{\mathbf{R}}), \delta^-(\mathbf{F}^1_{\mathbf{T}})])), 
\end{aligned}
\end{equation}
where $[\cdot, \cdot]$ denotes channel-wise concatenation. This architecture of the block is visualized in Fig.~\ref{fig:activation_results}~(b).
Unlike vanilla CNNs that permanently discard negative-domain features, this block ensures that information suppressed by one branch is fully recycled by the other. We term this the ``Your Trash is My Treasure'' (YTMT) strategy, which maximizes feature utilization efficiency. While various activation functions can instantiate the YTMT block, our comparison in Fig.~\ref{fig:activation_results}~(c) reveals that the basic ReLU pairing (\emph{e.g.}, $\max(0, x)$ and $x - \max(0, x)$) demonstrates the most effective performance. This indicates that feature separation with a strict, hard boundary best aligns with the mutually complementary nature of the dual-stream features in activation-based networks.

\vspace{1mm}\noindent\textbf{Gate-based DSI Block.}
While YTMT establishes a clean dichotomous boundary, real-world sRGB reflection superimposition exhibits spatially non-uniform coupling and multiplicative ISP interactions, as established in Section~\ref{sec:modeling}. Facing these non-linearities, naively stacking activation blocks falls short. Gating mechanisms~\cite{dauphin2017language} inherently offer the multiplicative modeling required here. In standard single-stream models, a gating mechanism acts as a self-modulator, typically formulated as $\text{Gate}(x) = x \circ \mathcal{G}(x)$, where $x$ represents the controlled component and $\mathcal{G}(\cdot)$ assigns spatial attention-like weights to $x$ (\emph{e.g.}, Sigmoid in GLU  or ReLU in ReGLU~\cite{dauphin2017language}).

Motivated by this, we extend the self-gating concept into a cross-stream paradigm tailored for dual-stream architectures. We introduce the Mutual Gate (MuGate), defined as:
\begin{equation}
    \text{MuGate}(x, y) = x \circ \mathcal{G}(y).
\end{equation}
Here, the activation of one feature stream $x$ is explicitly modulated by the semantics of the other stream $y$. In the context of layer decomposition, this effectively models the photometric interaction: a high-intensity semantic region in the transmission stream can actively suppress the corresponding activation in the reflection stream, and vice versa.

\begin{figure*}[t]
\centering
\begin{subfigure}[b]{0.35\textwidth}
    \centering
    \begin{minipage}[c][4.8cm][c]{\linewidth}
        \centering
        \begin{minipage}[b]{0.246\linewidth}
            \centering
            \includegraphics[width=\linewidth]{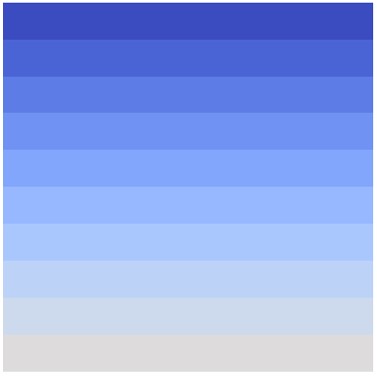}\\[0.5mm]
            {\footnotesize \textbf{T}}\\[2mm]
            \includegraphics[width=\linewidth]{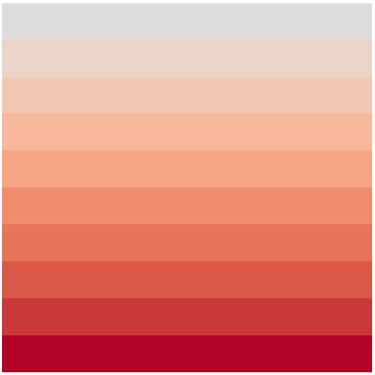}\\[0.5mm]
            {\footnotesize \textbf{R}}
        \end{minipage}
        \hfill
        \begin{minipage}[b]{0.68\linewidth}
            \centering
            \includegraphics[width=\linewidth]{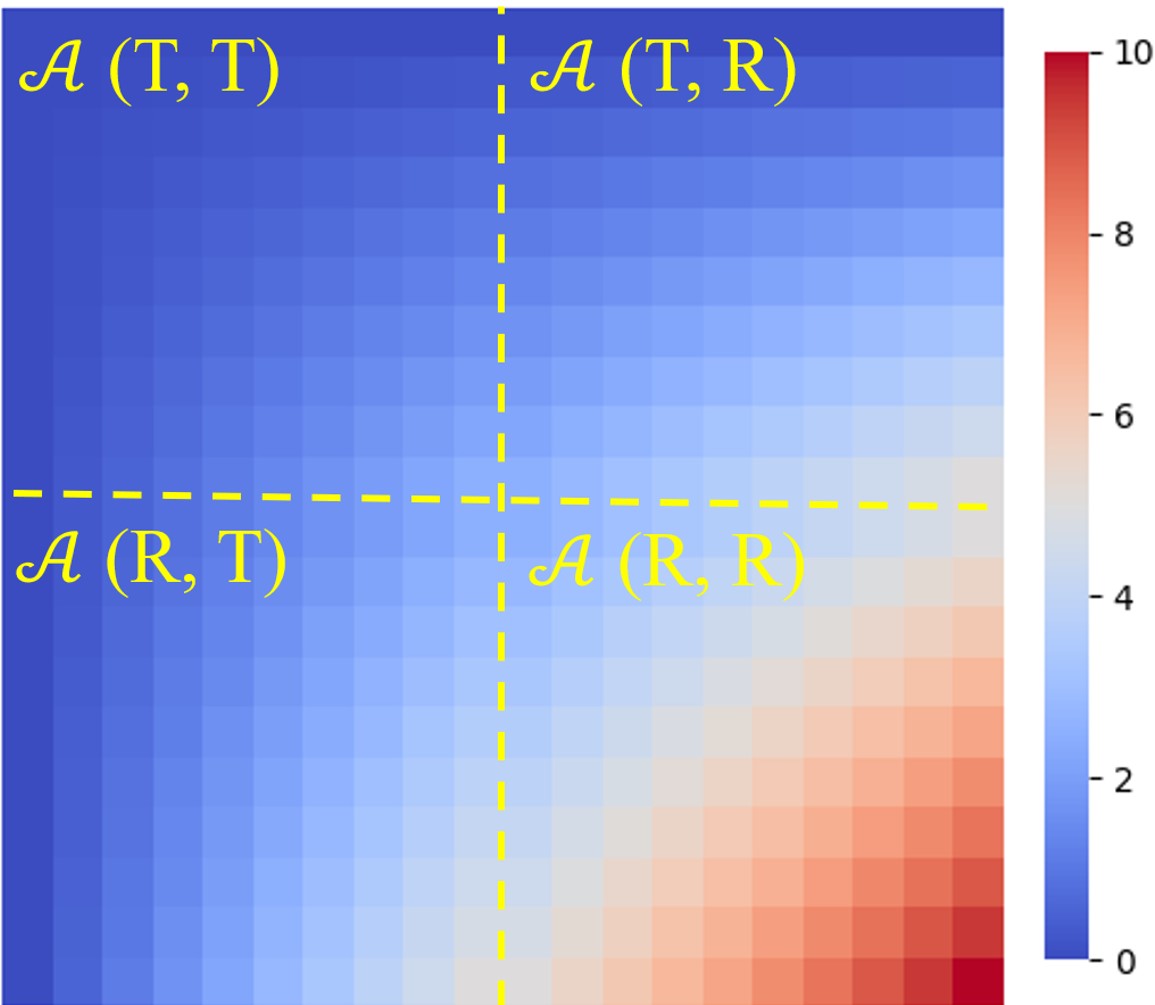}\\[1mm]
            {\footnotesize Attention map of DS-JA(T, R)}
        \end{minipage}
    \end{minipage}
    \caption{Visualizations of DS-JA}
    \label{fig:attn_vis_merged}
\end{subfigure}
\hfill
\begin{subfigure}[b]{0.28\textwidth}
    \centering
    \begin{minipage}[c][4.8cm][c]{\linewidth}
        \centering
        \includegraphics[height=4.8cm, keepaspectratio]{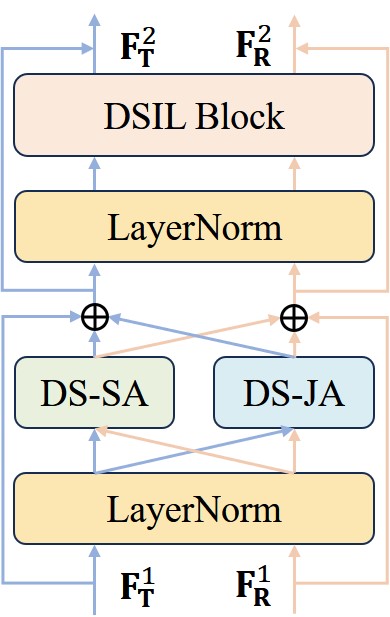}
    \end{minipage}
    \caption{Architecture of the PAIR Block}
    \label{fig:attn_arch}
\end{subfigure}
\hfill
\begin{subfigure}[b]{0.32\textwidth}
    \centering
    \begin{minipage}[c][4.8cm][c]{\linewidth}
        \centering
        \begin{tabular}{lcc}
            \toprule
            Attention ($\mathcal{A}$) & PSNR & SSIM \\
            \midrule
            Mamba2             & 25.05 & 0.904 \\
            PoolingAttn        & 24.41 & 0.892 \\
            TransposedAttn      & 25.60 & 0.912 \\
            NeighborhoodAttn   & 25.74 & 0.914 \\
            WindowAttn         & \textbf{26.37} & \textbf{0.918} \\
            \bottomrule
        \end{tabular}
    \end{minipage}
    \caption{Ablation on attention choices}
    \label{fig:attn_table}
\end{subfigure}

\caption{(a) The inputs (\textbf{T}, \textbf{R}) and the Dual-Stream Joint-Attention (DS-JA) mechanism. (b) Illustration of the PAIR Block. (c) Ablation on various attention mechanisms ($\mathcal{A}$) for DS-SA and DS-JA evaluated on Real20 and SIR$^2$ datasets.}
\label{fig:attention_compare}
\end{figure*}

As illustrated in Fig.~\ref{fig:mugi_gates}~(b), the Mutually-Gated Interaction (MuGI) block realizes the Gate-based DSI variant within the SIF paradigm (Fig.~\ref{fig:fib}~(c)) by leveraging the proposed MuGate operator and depthwise separable convolutions \cite{arxiv/howard2017mobilenets}. Specifically, the Selection operator processes the input features $\mathbf{F}^1_*$ ($* \in \{\mathbf{T}, \mathbf{R}\}$) through a sequence of Layer Normalization (LN), a $1 \times 1$ convolution, and a $3 \times 3$ convolution, followed by a channel-wise split: 
\begin{equation}
    \mathbf{K}_*, \mathbf{Q}_* = \text{Split}\left(\text{Conv}_{3 \times 3}\left(\text{Conv}_{1 \times 1}\left(\text{LN}(\mathbf{F}^1_*)\right)\right)\right),
\end{equation}
where $\text{LN}(\cdot)$ normalizes the feature distributions, which guarantees numerical comparability between the dual streams. The $\text{Conv}_{1 \times 1}(\cdot)$ then elevates the feature dimension by doubling the channels to provide sufficient capacity for the  feature separation. The following $\text{Conv}_{3 \times 3}(\cdot)$ guarantees the information diversity between the intra-stream controlled component $\mathbf{K}_*$ and the cross-stream controlling component $\mathbf{Q}_*$, which are subsequently generated by the $\text{Split}(\cdot)$ operator. During the cross-stream interaction, the controlling components are passed through the gating function to generate the modulation maps $\mathcal{G}(\mathbf{Q}_{\mathbf{R}})$ and $\mathcal{G}(\mathbf{Q}_{\mathbf{T}})$. These maps are subsequently applied to the retained components $\mathbf{K}_{\mathbf{T}}$ and $\mathbf{K}_{\mathbf{R}}$ via the MuGate operator.
Finally, the Fusion operator, comprising Channel Attention (CA) and a $1\times 1$ convolution, integrates these mutually interacted features. The above operations are expressed as: 
\begin{equation}
\begin{aligned}
    \mathbf{F}^2_{\mathbf{T}} &= \text{Conv}_{1\times 1}(\text{CA}(\mathbf{K}_{\mathbf{T}} \circ \mathcal{G}(\mathbf{Q}_{\mathbf{R}}))), \\
    \mathbf{F}^2_{\mathbf{R}} &= \text{Conv}_{1\times 1}(\text{CA}(\mathbf{K}_{\mathbf{R}} \circ \mathcal{G}(\mathbf{Q}_{\mathbf{T}}))). 
\end{aligned}
\end{equation}

As visualized in Fig.~\ref{fig:mugi_gates}~(a), the cross-stream control signal $y$ effectively modulates the functional mapping between the input $x$ and the gate output, explicitly illustrating the dynamic mutual modulation mechanism between the two streams. Furthermore, the quantitative evaluation in Fig.~\ref{fig:mugi_gates}~(c) demonstrates that the specific choice of the gating function significantly impacts the efficiency of this cross-stream interaction. Among the evaluated variants, employing the ReGLU gate ($\mathcal{G}(x) = \max(0, x)$) yields the optimal separation performance, effectively capturing the mutually inhibitive nature of the transmission and reflection layers. 

\vspace{1mm}\noindent\textbf{Attention Based DSI Block.} 
While lightweight, the convolutional designs (YTMT and MuGI) lack nonlocal receptive fields and explicit interstream similarity assessment. They passively aggregate all routed features, risking information scrambling when contextual interaction is unnecessary. 

To overcome this, we propose the Parallel Attention Interaction and Refinement (PAIR) block (Fig.~\ref{fig:fib}~(d)), leveraging dot product attention for nonlocal explicit assessment~\cite{nips/XieWYAAL21,cvpr/Liu0LYXWN000WG22,iccv/ZongS023}. Naively applying standard Cross Attention forces the query to compute Softmax weights exclusively over the opposing stream. This inherently lacks the ability to reject irrelevant information, forcibly introducing noisy features even when interstream similarity is low.

We resolve this via Dual Stream Joint Attention (DS-JA). Given features $\textbf{T}, \textbf{R} \in \mathbb{R}^{N \times C}$, DS-JA concatenates them along the sequence dimension into $\textbf{X} = [\textbf{T}^T, \textbf{R}^T]^T \in \mathbb{R}^{2N \times C}$ to compute joint queries $\textbf{Q}$, keys $\textbf{K}$, and values $\textbf{V}$. The joint attention map $\textbf{A}$ (Fig.~\ref{fig:attention_compare}~(a)) naturally partitions into four quadrants:
\begin{equation}
\begin{aligned}
    \textbf{A} &= \text{Softmax}(\frac{\textbf{Q}\textbf{K}^T}{\sqrt{D}}) = \text{Softmax}(\frac{1}{\sqrt{D}}\begin{bmatrix}
    \textbf{T} \\
    \textbf{R}
\end{bmatrix}\textbf{W}_q\textbf{W}_k^T
\begin{bmatrix}
    \textbf{T}^T \, \textbf{R}^T
\end{bmatrix}) \\
&= \text{Softmax}(\frac{1}{\sqrt{D}} \begin{bmatrix}
    \textbf{T}\textbf{W}_q\textbf{W}_k^T\textbf{T}^T & \textbf{T}\textbf{W}_q\textbf{W}_k^T\textbf{R}^T \\
    \textbf{R}\textbf{W}_q\textbf{W}_k^T\textbf{T}^T & \textbf{R}\textbf{W}_q\textbf{W}_k^T\textbf{R}^T
\end{bmatrix} ).
\end{aligned}
\end{equation}

Defining partial attention $\mathcal{A}^p(\textbf{B}_1,\textbf{B}_2) = \text{Softmax}(\frac{1}{\sqrt{D}}\textbf{B}_1\textbf{W}_q\textbf{W}_k^T\textbf{B}_2^T)\textbf{B}_2\textbf{W}_v$, the updated representation $\textbf{Y} = \textbf{A}\textbf{V}$ decomposes into:
\begin{equation}
\textbf{Y} =  \begin{bmatrix}
\textbf{T}_o \\
\textbf{R}_o
\end{bmatrix}
= \begin{bmatrix}
\mathcal{A}^1(\textbf{T},\textbf{T}) + \mathcal{A}^1(\textbf{T},\textbf{R}) \\
\mathcal{A}^2(\textbf{R},\textbf{T}) + \mathcal{A}^2(\textbf{R},\textbf{R})
\end{bmatrix},
\end{equation}
where terms sharing superscript $p$ use identical Softmax denominators. If interstream correlations are weak, the probability mass naturally falls back to intrastream components (\emph{e.g.}, $\mathcal{A}^1(\textbf{T},\textbf{T})$). This prevents forced aggregation of noisy features, safeguarding information purity.

To complement this, PAIR concurrently employs parallel Dual Stream Self Attention (DS-SA) (Fig.~\ref{fig:attention_compare}~(b)). While DS-JA modulates cross stream features, it risks diluting structural focus if queries are dominated by opposing high response tokens. DS-SA, executing independent self attention, establishes an undiluted baseline for intrastream coherence. Their elementwise addition yields features balancing pure spatial context with dynamic modulation. Notably, PAIR is compatible with various operators $\mathcal{A}(\cdot)$; our ablation (Fig.~\ref{fig:attention_compare}~(c)) shows Window Attention (WindowAttn) optimal for balancing receptive fields and spatial fidelity.

In summary, the three DIRS architectures introduced in Sec.~\ref{sec:ds_arch}, which are instantiated by the interactive blocks detailed above, offer an architectural spectrum tailored to diverse deployment requirements. Specifically, DIRS-YTMT and DIRS-MuGI serve as highly efficient, CNN-based variants suitable for resource-constrained hardware, maintaining low computational footprints while delivering competitive results. Conversely, DIRS-PAIR leverages non-local Transformer modeling to maximize restoration quality despite higher computational demands. To thoroughly push the performance upper bound, we select DIRS-PAIR as the primary configuration to benchmark against current state-of-the-art methods in the subsequent experimental evaluations (Sec.~\ref{sec:performance}), while deferring a detailed complexity and trade-off analysis of all variants to Sec.~\ref{sec:variants_analysis}.

\begin{table*}[t]
  \caption{Quantitative results on four real-world datasets. The best results are in \textbf{bold}, while the second-best are \underline{underlined}. $\dagger$ means training under data setting II. * represents additional prompts are used. $\triangle$ reflects extra data pairs are involved.}
  \label{tab:qcomp}
  \centering
  \begin{tabular}{lccccccccccc}
  \toprule[1pt]
  \multirow{2}{*}{Methods} & \multirow{2}{*}{Venue} & \multicolumn{2}{c}{Real20 (20)} & \multicolumn{2}{c}{Objects (200)} & \multicolumn{2}{c}{Postcard (199)} &\multicolumn{2}{c}{Wild (55)} & \multicolumn{2}{c}{Average} \\ 
  \addlinespace[1mm]
  \cline{3-12}
  \addlinespace[1mm]
  & & PSNR & SSIM & PSNR  & SSIM & PSNR & SSIM & PSNR & SSIM & PSNR & SSIM \\ \midrule
  Zhang \emph{et al.} \cite{cvpr/ZhangNC18a} & CVPR, 2018 & 22.55 & 0.788 & 22.68 & 0.879 & 16.81 & 0.797 & 21.52 & 0.832 & 20.08 & 0.835 \\
  BDN \cite{eccv/YangGLS18} & ECCV, 2018  & 18.41 & 0.726 & 22.72 & 0.856 & 20.71 & 0.859 & 22.36 & 0.830 & 21.65 & 0.849 \\
  ERRNet \cite{cvpr/WeiYFW019} & CVPR, 2019 & 22.89 & 0.803 & 24.87 & 0.896 & 22.04 & 0.876 & 24.25 & 0.853 & 23.53 & 0.879 \\
  IBCLN \cite{cvpr/LiY0LH20} & CVPR, 2020& 21.86 & 0.762 & 24.87 & 0.893 & 23.39 & 0.875 & 24.71 & 0.886 & 24.10 & 0.879 \\
  DMGN \cite{tip/FengPJCZL21} & TIP, 2021 & 20.71 & 0.770 & 24.98 & 0.899 & 22.92 & 0.877 & 23.81 & 0.835 & 23.80 & 0.877 \\
  MoG-SIRR \cite{mma/shao2021model} & MMAsia, 2021 & 21.63 & 0.814 & 24.57 & 0.911 & 22.78 & 0.892 & 24.13 & 0.890 & 23.64 & 0.896 \\
  Zheng \emph{et al.} \cite{cvpr/ZhengSCJDK21} & CVPR, 2021 & 20.17 & 0.755 & 25.20 & 0.880 & 23.26 & 0.905 & 25.39 & 0.878 & 24.20 & 0.885 \\
  RobustSIRR \cite{cvpr/song2023robust} & CVPR, 2023  & 23.30 & 0.827 & 24.90 & \underline{0.917} & 19.91 & 0.868 & 23.67 & 0.884 & 22.59 & 0.889 \\ 
  SRNet \cite{chen2024closer} & TIP, 2024  & 23.58 & 0.803 & \underline{26.96} & 0.912 & 23.85 & 0.892 & 25.63 & 0.894 & 25.36 & 0.897 \\ 
  RDNet \cite{cvpr/zhao2025reversible} & CVPR, 2025 & \underline{24.43} & \underline{0.835} & 25.76 & 0.905 & \textbf{25.95} & \textbf{0.920} & \textbf{27.20} & \underline{0.910} & \underline{25.95} & \underline{0.909} \\
  DIRS-PAIR (Ours) & - & \textbf{25.02} & \textbf{0.836} & \textbf{27.40} & \textbf{0.927} & \underline{25.47} & \underline{0.919} & \underline{26.41} & \textbf{0.911} & \textbf{26.37} & \textbf{0.918} \\
  \midrule
  ${\text{Dong \emph{et al.}}}^\dagger$ \cite{iccv/Dong00BXL21} & ICCV, 2021 & 23.34 & 0.812 & 24.36 & 0.898 & 23.72 & 0.903 & 25.73 & 0.902 & 24.21 & 0.897 \\
  $\text{PNACR}^\dagger$ \cite{mm/wang2023personalized} & MM, 2023 & 22.57 & 0.806 & 24.73 & 0.897 & 23.11 & 0.890 &25.69 & 0.903 & 24.07 & 0.891 \\
  $\text{RRW}^{\triangle}$ \cite{cvpr/zhu2024revisiting} & CVPR, 2024 & 21.83 & 0.801 & 26.67 & \textbf{0.931} & 24.04 & 0.903 & 26.49 & 0.915 & 25.34 & 0.912 \\
  $\text{Zhong \emph{et al.}}^{\dagger\triangle\ast}$ \cite{cvpr/ZhongHWLS24} & CVPR, 2024 & 24.05 & 0.824 & 26.51 & 0.927 & 25.02 & 0.915 & 26.23 & \textbf{0.925} & 25.75 & 0.917 \\
  $\text{L-DiffER}^{\dagger\triangle\ast}$ \cite{eccv/hong2024differ} & ECCV, 2024 & 23.77 & 0.821 & 25.75 & 0.918 & 24.35 & 0.905 & 26.11 & 0.909 & 25.12 & 0.907 \\
  $\text{DExNet}^\dagger$ \cite{pami/huang2025lightweight} & TPAMI, 2025  & 23.50 & 0.817 & 26.38 & 0.916 & 25.52 & 0.918 & 26.95 & 0.908 & 25.96 & 0.912 \\ 
  $\text{RDNet}^\dagger$ \cite{cvpr/zhao2025reversible} & CVPR, 2025 & \textbf{25.58} & \textbf{0.846} & \underline{26.78} & 0.921 & \underline{26.33} & \underline{0.922} & \underline{27.70} & 0.915 & \underline{26.65} & \underline{0.918} \\
  $\text{DIRS-PAIR}^\dagger$ (Ours) & - & \underline{25.19} & \underline{0.834} & \textbf{26.87} & \underline{0.926} & \textbf{26.38} & \textbf{0.925} & \textbf{27.90} & \underline{0.920} & \textbf{26.71} & \textbf{0.921} \\
  \bottomrule[1pt]
  \end{tabular}
\end{table*}

\begin{table*}[t]
  \centering
  \caption{Quantitative results on the Nature test set for models trained under Data Setting II.}
  \label{tab:nature}
  \setlength{\tabcolsep}{8pt}
  \begin{tabular}{ccccccc}
  \toprule[1pt]
  Metrics &  Zhang \emph{et al.} \cite{cvpr/ZhangNC18a} & BDN-F \cite{eccv/YangGLS18}  & ERRNet-F \cite{cvpr/WeiYFW019} & IBCLN \cite{cvpr/LiY0LH20} & Dong \emph{et al.} \cite{iccv/Dong00BXL21} & PNACR \cite{mm/wang2023personalized} \\ 
  \midrule
  PSNR    & 19.56 & 18.92 & 22.18 & 23.57 & 23.45 & 23.92  \\
  SSIM    & 0.736 & 0.737 & 0.756 & 0.783 & 0.808 & 0.807  \\ 
  \midrule
  Metrics & RRW \cite{cvpr/zhu2024revisiting} &  Zhong \emph{et al.} \cite{cvpr/ZhongHWLS24} & L-DiffER \cite{eccv/hong2024differ} & DExNet \cite{pami/huang2025lightweight} & RDNet \cite{cvpr/zhao2025reversible} & $\text{DIRS-PAIR}^\dagger$ (Ours)\\ 
  \midrule
  PSNR   & 26.04 & 23.87 & 23.95 & 24.69 & \underline{26.21} & \textbf{26.67} \\
  SSIM   & \underline{0.846} & 0.812 & 0.831 & 0.841 & 0.842 & \textbf{0.847} \\ 
  \bottomrule[1pt]
  \end{tabular}
\end{table*}

\subsection{Dual-Stream Learning Objective}
\label{sec:ds_obj}


\noindent\textbf{Pixel reconstruction loss.} To enforce consistency between predicted layers $(\hat{\mathbf{T}}, \hat{\mathbf{R}})$ and ground truths $(\mathbf{T}, \mathbf{R})$, while constraining nonlinear superposition under the LORS formulation, we define:
\begin{equation}
\begin{aligned}
    \mathcal{L}_{\text{pix}} &:= \| \hat{\textbf{T}} - \textbf{T} \|^2_2 + \|\hat{\textbf{R}} - \textbf{R}\|^2_2 \\
    &\quad + \alpha\|\textbf{I} - (\hat{\textbf{T}} + \hat{\textbf{R}} + \boldsymbol{\Phi}(\hat{\textbf{T}},\hat{\textbf{R}}) + \boldsymbol{\Psi})\|_1 + \beta\|\nabla \boldsymbol{\Psi}\|_1,
\end{aligned}
\end{equation}
where $\|\cdot\|_2$ and $\|\cdot\|_1$ denote the $\ell_2$ and $\ell_1$ norms, respectively, and $\nabla$ represents the spatial gradient operator. The total variation penalty $\|\nabla \boldsymbol{\Psi}\|_1$ encourages the zero-order offset to be spatially smooth, preventing it from absorbing high-frequency details. The coefficients $\alpha$ and $\beta$ are hyperparameters balancing the LORS and smoothness penalties, empirically set to $0.2$ and $10^{-3}$, respectively.


\noindent\textbf{Gradient disentanglement loss.} To preserve high-frequency details and enforce structural independence between layers~\cite{pami/LevinW07,iccv/FanYHCW17,cvpr/WanSDTK18,cvpr/WeiYFW019}, we adopt:
\begin{equation}
\begin{aligned}
     \mathcal{L}_{\text{grad}} := \|\nabla \hat{\textbf{T}}-\nabla \textbf{T}\|_{1}&+\|\nabla \hat{\textbf{R}}-\nabla \textbf{R}\|_{1} \\
     &+ \frac{1}{N} \sum_{n=0}^{N-1} \| \mathcal{D}(\hat{\textbf{T}}^{\downarrow n}, \hat{\textbf{R}}^{\downarrow n})\|_2^2, \\
    \text{with} \ \;\mathcal{D}(\hat{\textbf{T}}, \hat{\textbf{R}}) :=  \tanh & \left(\xi_{1}|\nabla \hat{\textbf{T}}|\right) \circ  \tanh \left(\xi_{2}|\nabla \hat{\textbf{R}}|\right),
\end{aligned}
\end{equation}
where $\nabla$ denotes the first-order derivative operator. $\hat{\textbf{T}}^{\downarrow n}$ and $\hat{\textbf{R}}^{\downarrow n}$ are the $2^n$ down-sampled versions of $\hat{\textbf{T}}$ and $\hat{\textbf{R}}$, and $\xi_1$ and $\xi_2$ act as normalization factors. The exclusion term $\mathcal{D}(\hat{\textbf{T}}, \hat{\textbf{R}})$ ensures the multi-scale mutual exclusion of the two layers in the gradient domain~\cite{cvpr/ZhangNC18a}.

\noindent\textbf{Feature reconstruction loss.} To promote the perceptual visual quality of the decoupled layers, we harness the standard feature reconstruction loss:
\begin{equation}
    \mathcal{L}_{\text{fea}} := \sum_{i} \omega_{i}\|\phi_{i}(\hat{\textbf{T}})-\phi_{i}(\textbf{T})\|_{1},
\end{equation}
where $\phi_i(\cdot)$ extracts the intermediate feature maps from a pre-trained VGG-19 network, with $i\in\{2, 7, 12, 21, 30\}$ denoting the specific layer indices. Furthermore, $\omega_i$ represents the weight balancing different hierarchical levels. 

\noindent\textbf{Total loss.} The final training objective is a weighted combination of the aforementioned three domains of losses:
\begin{equation}
    \mathcal{L}_{\text{tot}} := \mathcal{L}_{\text{pix}} + \lambda_1\mathcal{L}_{\text{grad}} + \lambda_2\mathcal{L}_{\text{fea}},
\end{equation}
where the coefficients $\lambda_1 = 1$ and $\lambda_2 = 0.01$ are empirically set to balance the respective loss terms.

\begin{figure*}[t]
     \centering{
     \begin{subfigure}{0.19\linewidth}
          \includegraphics[width=1\linewidth,height=65pt]{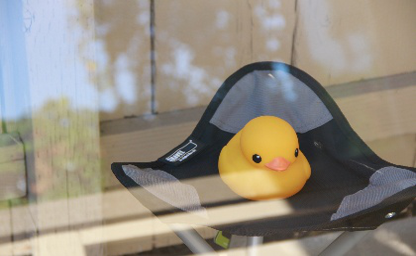}
          \subcaption*{Input}
          \vspace{2pt}
     \end{subfigure}
     \begin{subfigure}{0.19\linewidth}
          \includegraphics[width=1\linewidth,height=65pt]{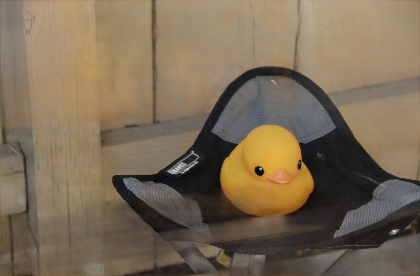}
          \subcaption*{Zhang \emph{et al.} \cite{cvpr/ZhangNC18a}}
          \vspace{2pt}
    \end{subfigure}
     \begin{subfigure}{0.19\linewidth}
          \includegraphics[width=1\linewidth,height=65pt]{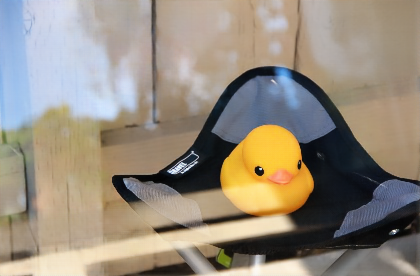}
          \subcaption*{BDN \cite{eccv/YangGLS18}}
          \vspace{2pt}
    \end{subfigure}
     \begin{subfigure}{0.19\linewidth}
          \includegraphics[width=1\linewidth,height=65pt]{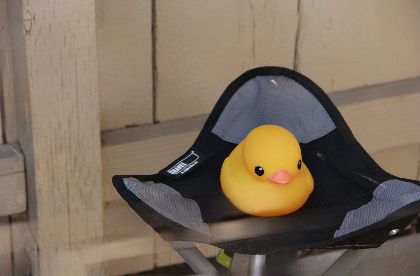}
          \subcaption*{ERRNet \cite{cvpr/WeiYFW019}}
          \vspace{2pt}
    \end{subfigure}
     \begin{subfigure}{0.19\linewidth}
          \includegraphics[width=1\linewidth,height=65pt]{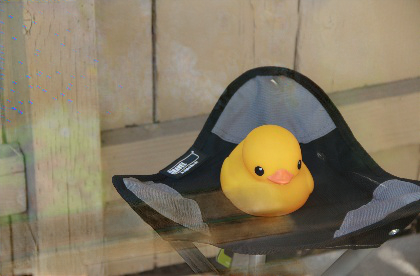}
          \subcaption*{IBCLN \cite{cvpr/LiY0LH20}}
          \vspace{2pt}
    \end{subfigure}
    \begin{subfigure}{0.19\linewidth}
          \includegraphics[width=1\linewidth,height=65pt]{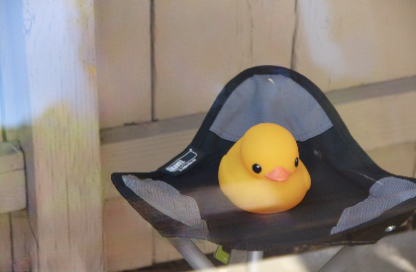}
          \subcaption*{Dong \emph{et al.} \cite{iccv/Dong00BXL21}}
    \end{subfigure}
    \begin{subfigure}{0.19\linewidth}
          \includegraphics[width=1\linewidth,height=65pt]{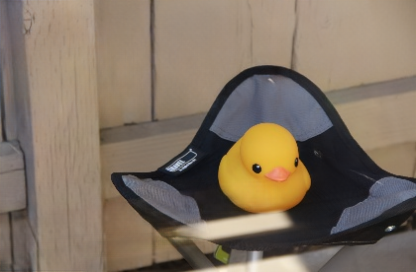}
          \subcaption*{RRW \cite{cvpr/zhu2024revisiting} }
    \end{subfigure}
    \begin{subfigure}{0.19\linewidth}
          \includegraphics[width=1\linewidth,height=65pt]{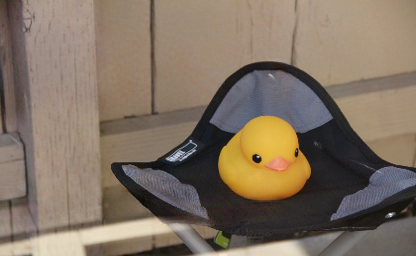}
          \subcaption*{RDNet \cite{cvpr/zhao2025reversible} }
    \end{subfigure}
    \begin{subfigure}{0.19\linewidth}
          \includegraphics[width=1\linewidth,height=65pt]{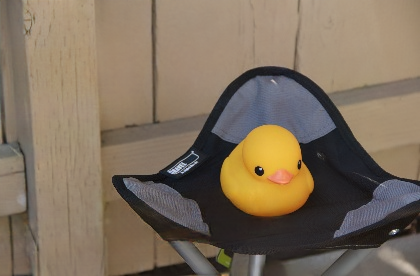}
          \subcaption*{DIRS-PAIR (Ours)}
    \end{subfigure}
     \begin{subfigure}{0.19\linewidth}
          \includegraphics[width=1\linewidth,height=65pt]{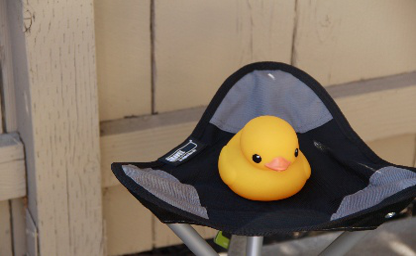}
          \subcaption*{GT}
    \end{subfigure}}
     \caption{Visual comparison of transmission layer predictions on a sample from the Real20 dataset.}
     \label{fig:visual_comp}
     \vspace{-5pt}
\end{figure*}
\begin{figure*}[t]
     \centering{
     \begin{subfigure}{0.16\linewidth}
          \includegraphics[width=1\linewidth,height=115pt]{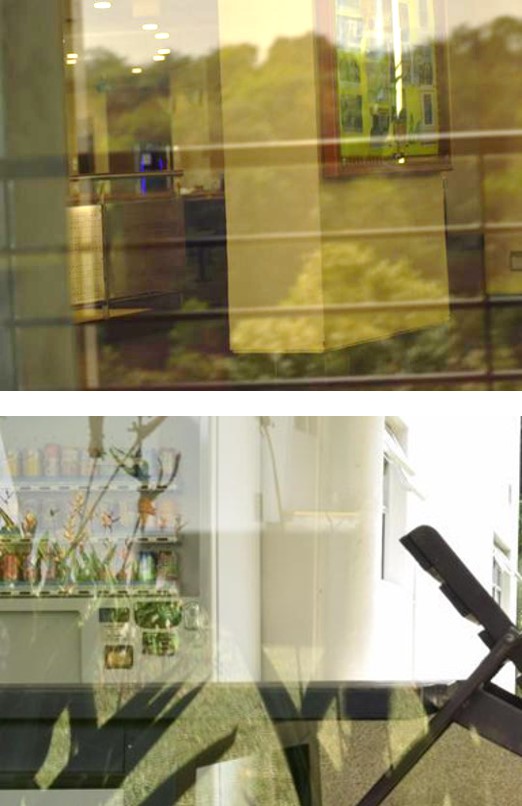}
          \subcaption*{Input}
     \end{subfigure}
     \begin{subfigure}{0.16\linewidth}
          \includegraphics[width=1\linewidth,height=115pt]{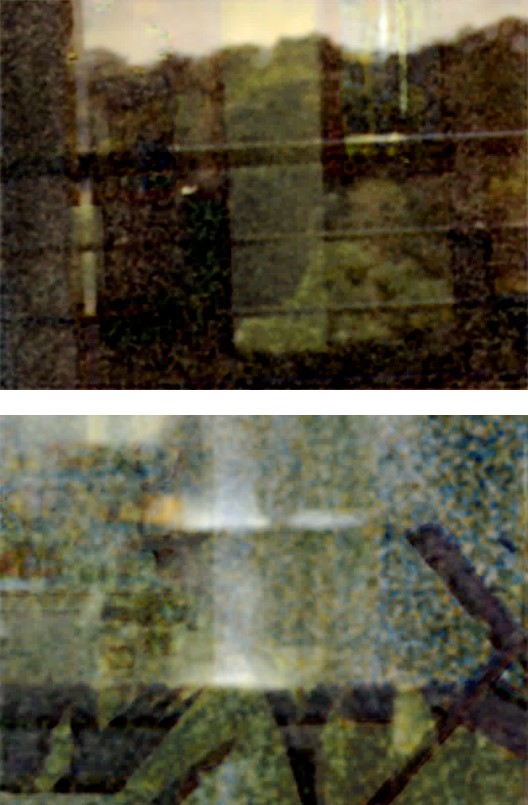}
          \subcaption*{IBCLN~\cite{cvpr/LiY0LH20}}
    \end{subfigure}
    \begin{subfigure}{0.16\linewidth}
          \includegraphics[width=1\linewidth,height=115pt]{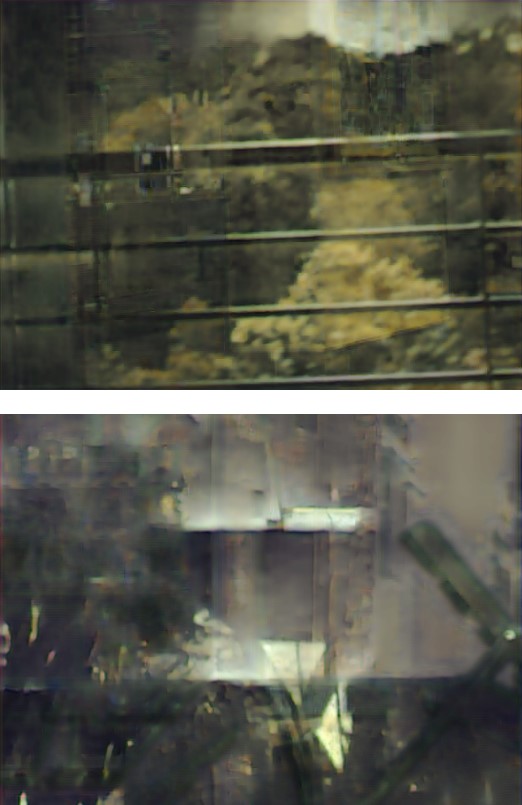}
          \subcaption*{Dong \emph{et al.}~\cite{iccv/Dong00BXL21}}
    \end{subfigure}
    \begin{subfigure}{0.16\linewidth}
          \includegraphics[width=1\linewidth,height=115pt]{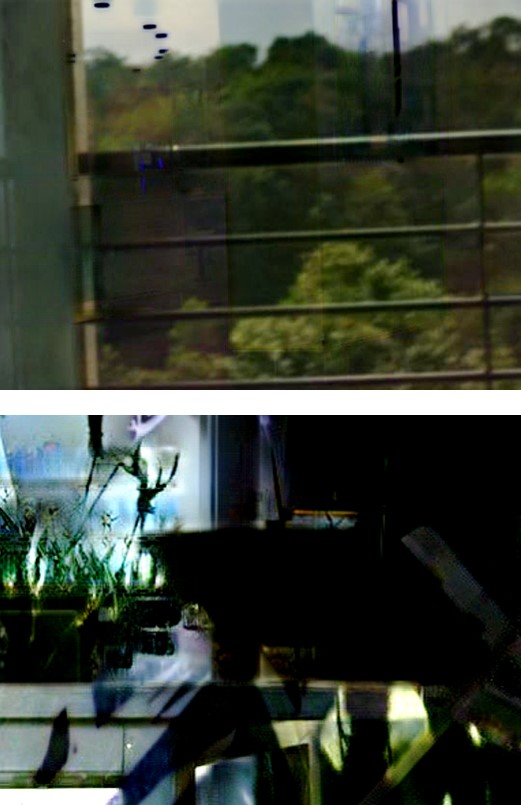}
          \subcaption*{RDNet~\cite{cvpr/zhao2025reversible}}
    \end{subfigure}
    \begin{subfigure}{0.16\linewidth}
          \includegraphics[width=1\linewidth,height=115pt]{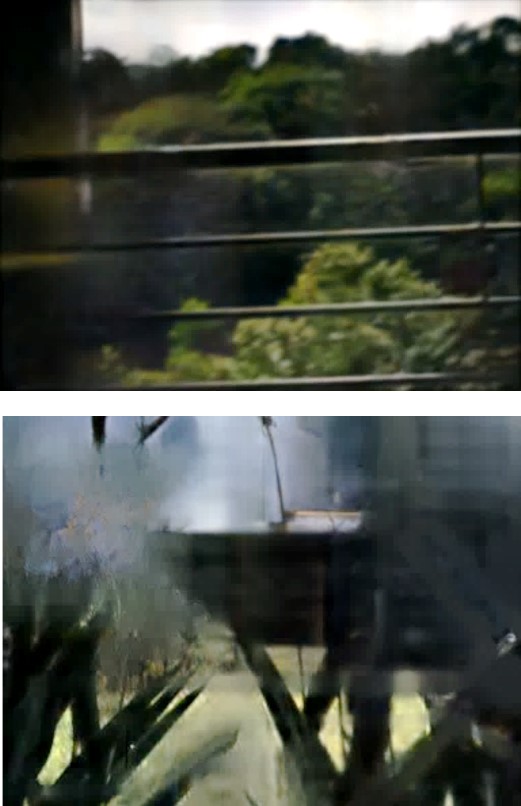}
          \subcaption*{DIRS-PAIR (Ours)}
    \end{subfigure}
     \begin{subfigure}{0.16\linewidth}
          \includegraphics[width=1\linewidth,height=115pt]{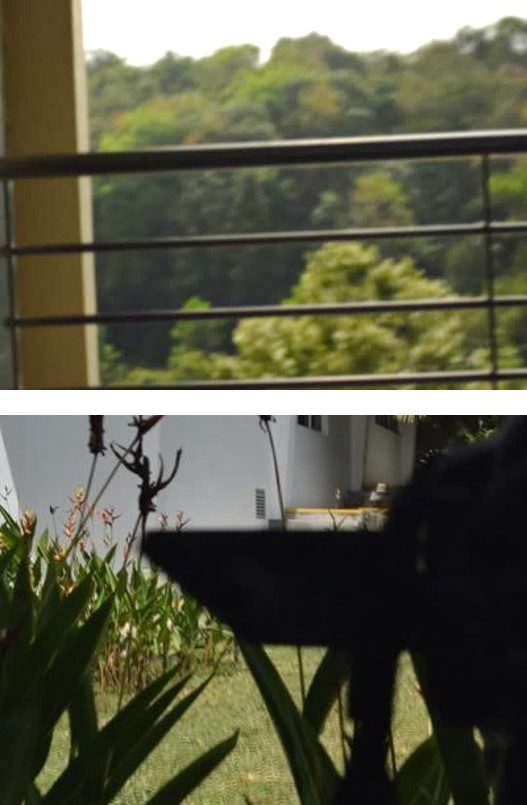}
          \subcaption*{GT}
    \end{subfigure}}
     \caption{Visual comparison of reflection layer predictions on samples from the $\text{SIR}^2$ dataset. 
     }
     \vspace{-10pt}
     \label{fig:visual_comp_r}
\end{figure*}

\begin{figure*}[t]
    \centering{
    \begin{subfigure}{0.19\linewidth}
        \includegraphics[width=1\linewidth,height=65pt]{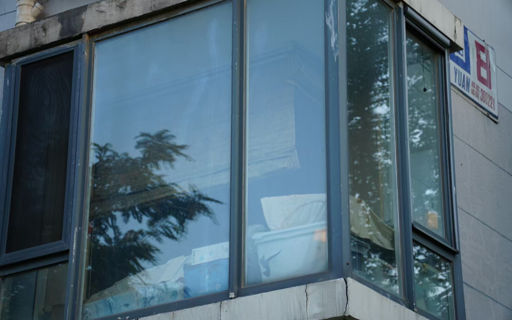}\vspace{2pt}
        \includegraphics[width=1\linewidth,height=65pt]{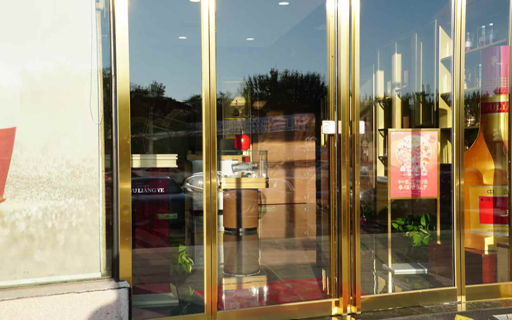}
        \subcaption*{Input}
        \vspace{2pt}
    \end{subfigure}
    \begin{subfigure}{0.19\linewidth}
        \includegraphics[width=1\linewidth,height=65pt]{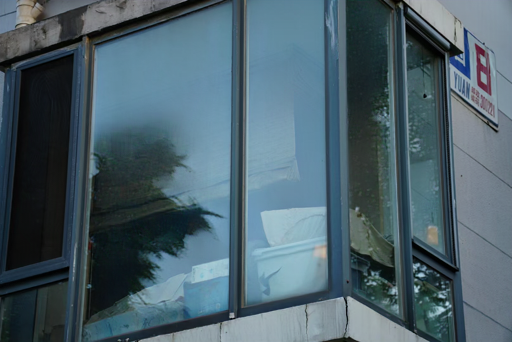}\vspace{2pt}
        \includegraphics[width=1\linewidth,height=65pt]{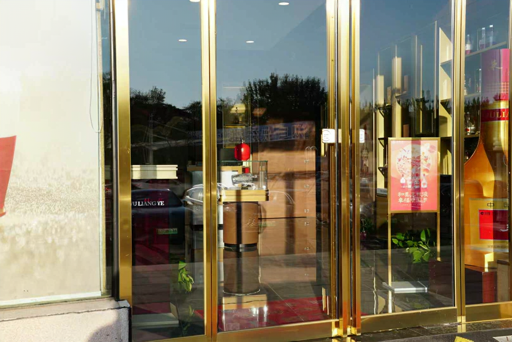}
        \subcaption*{ERRNet~\cite{cvpr/WeiYFW019}}
        \vspace{2pt}
    \end{subfigure}
    \begin{subfigure}{0.19\linewidth}
        \includegraphics[width=1\linewidth,height=65pt]{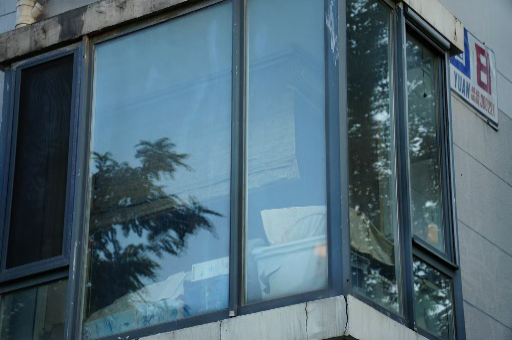}\vspace{2pt}
        \includegraphics[width=1\linewidth,height=65pt]{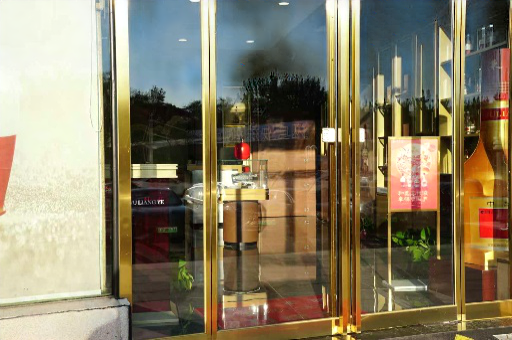}
        \subcaption*{IBCLN~\cite{cvpr/LiY0LH20}}
        \vspace{2pt}
    \end{subfigure}
    \begin{subfigure}{0.19\linewidth}
        \includegraphics[width=1\linewidth,height=65pt]{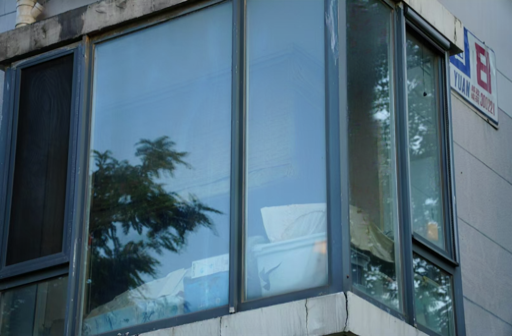}\vspace{2pt}
        \includegraphics[width=1\linewidth,height=65pt]{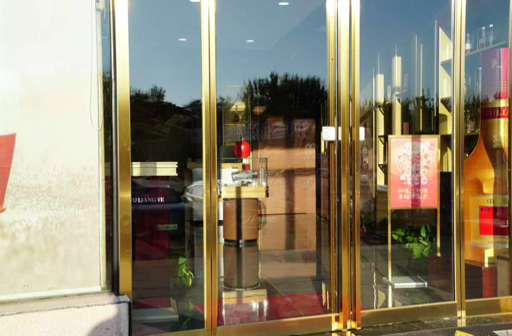}
        \subcaption*{Dong~\emph{et al.} \cite{iccv/Dong00BXL21}}
        \vspace{2pt}
    \end{subfigure}
    \begin{subfigure}{0.19\linewidth}
        \includegraphics[width=1\linewidth,height=65pt]{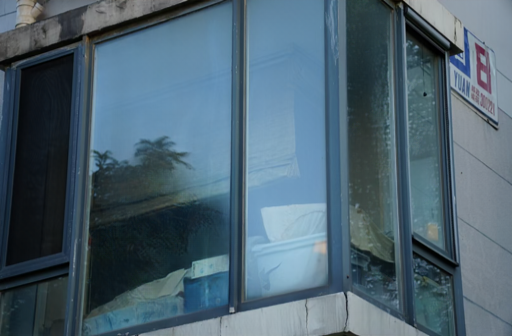}\vspace{2pt}
        \includegraphics[width=1\linewidth,height=65pt]{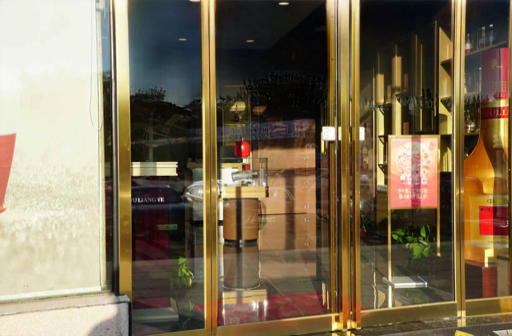}
        \subcaption*{RRW~\cite{cvpr/zhu2024revisiting}}
        \vspace{2pt}
    \end{subfigure}
    \begin{subfigure}{0.19\linewidth}
        \includegraphics[width=1\linewidth,height=65pt]{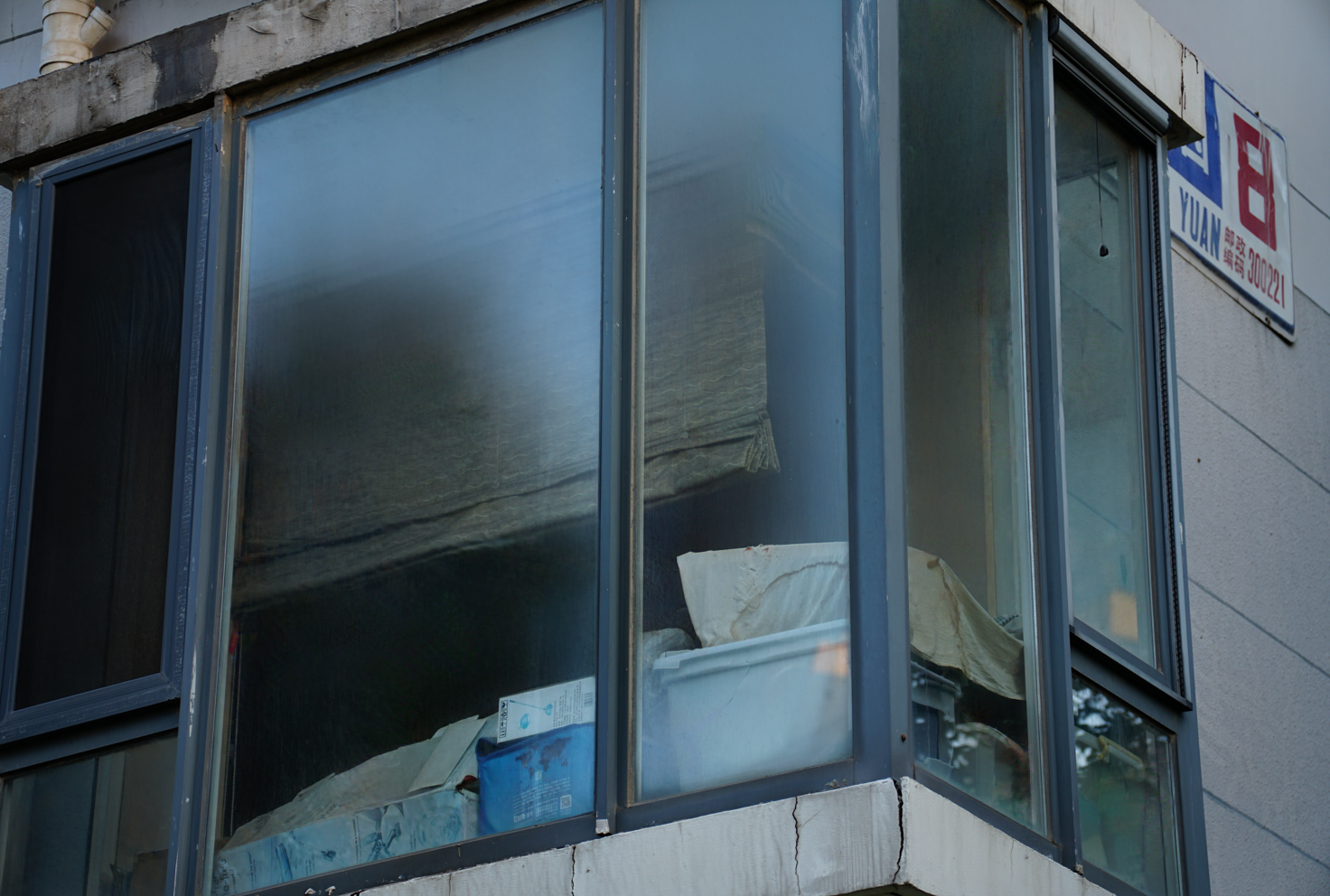}\vspace{2pt}
        \includegraphics[width=1\linewidth,height=65pt]{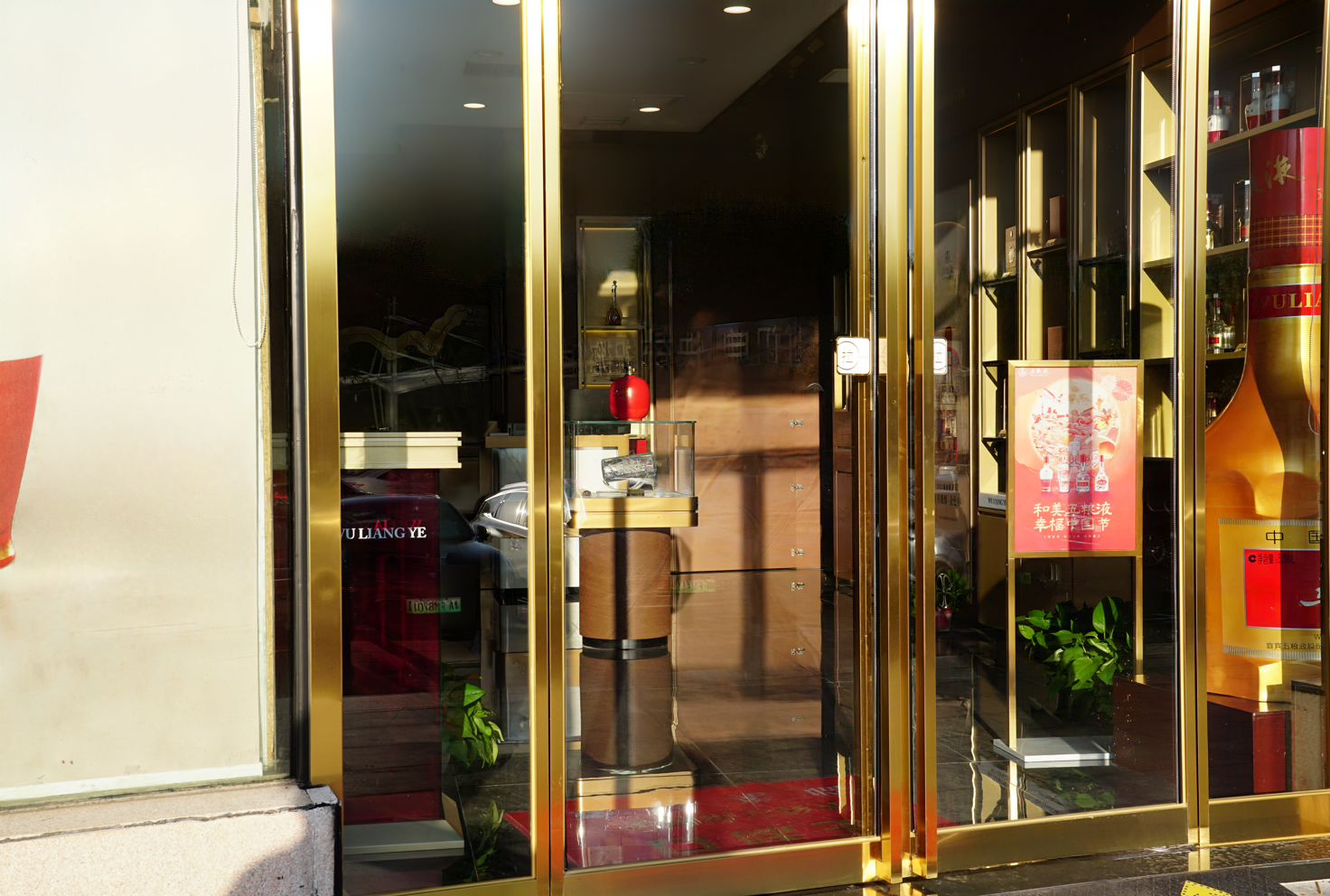}\vspace{2pt}
        \subcaption*{RDNet~\cite{cvpr/zhao2025reversible}}
    \end{subfigure}
    \begin{subfigure}{0.19\linewidth}
        \includegraphics[width=1\linewidth,height=65pt]{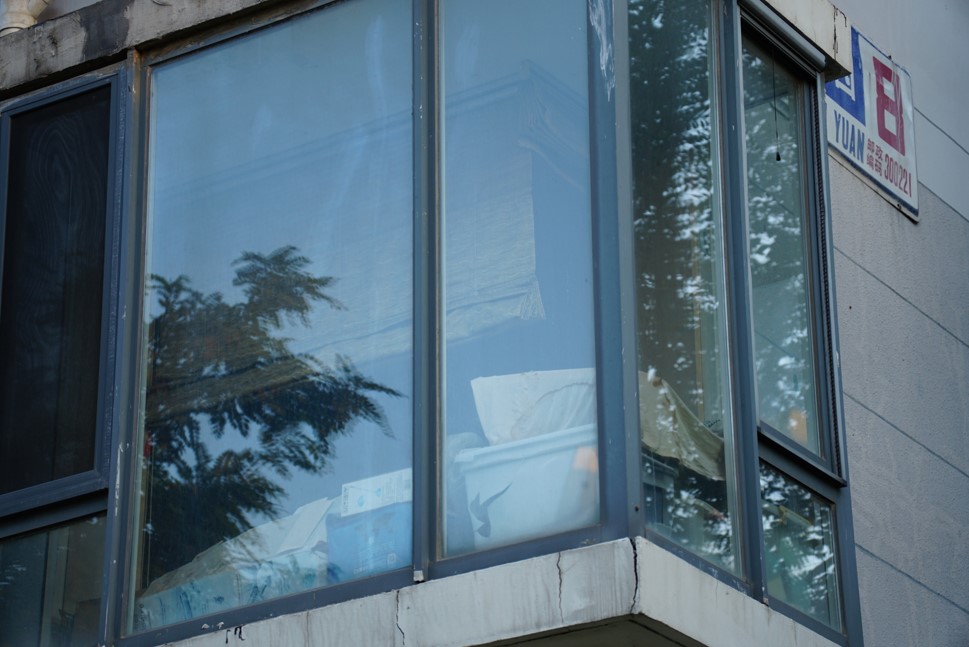}\vspace{2pt}
        \includegraphics[width=1\linewidth,height=65pt]{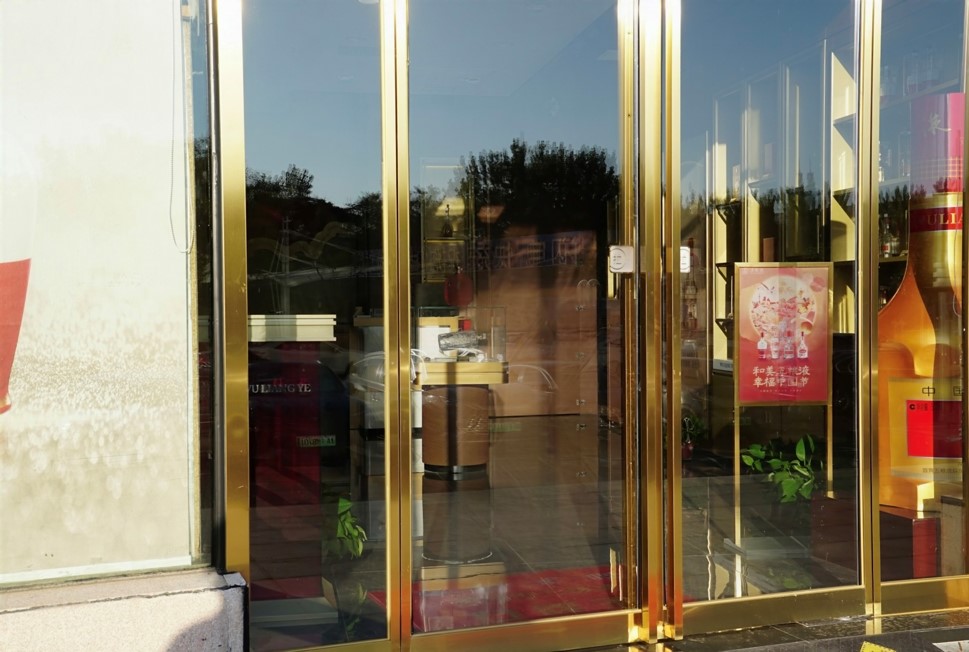}\vspace{2pt}
        \subcaption*{VIVO X300 Ultra}
    \end{subfigure}
        \begin{subfigure}{0.19\linewidth}
        \includegraphics[width=1\linewidth,height=65pt]{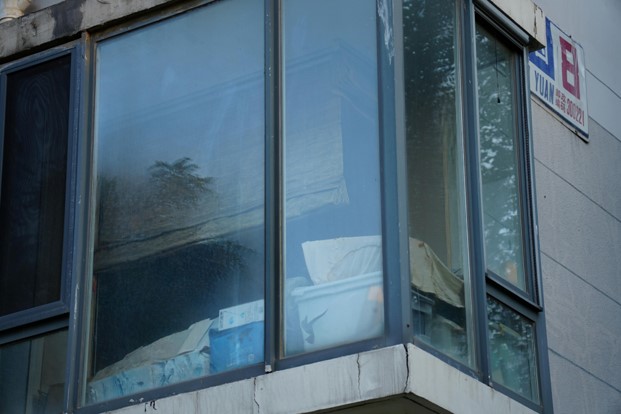}\vspace{2pt}
        \includegraphics[width=1\linewidth,height=65pt]{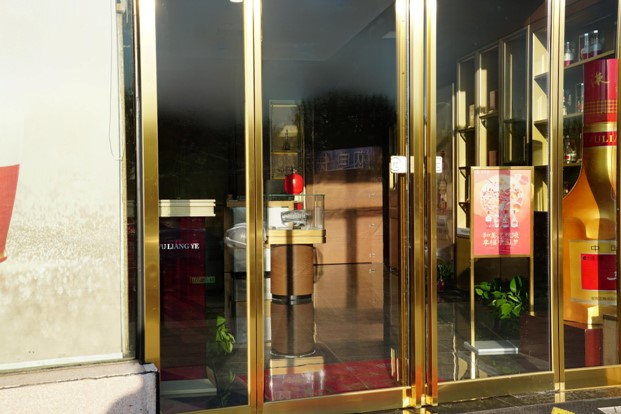}\vspace{2pt}
        \subcaption*{OPPO Find X9 Pro}
    \end{subfigure}
        \begin{subfigure}{0.19\linewidth}
        \includegraphics[width=1\linewidth,height=65pt]{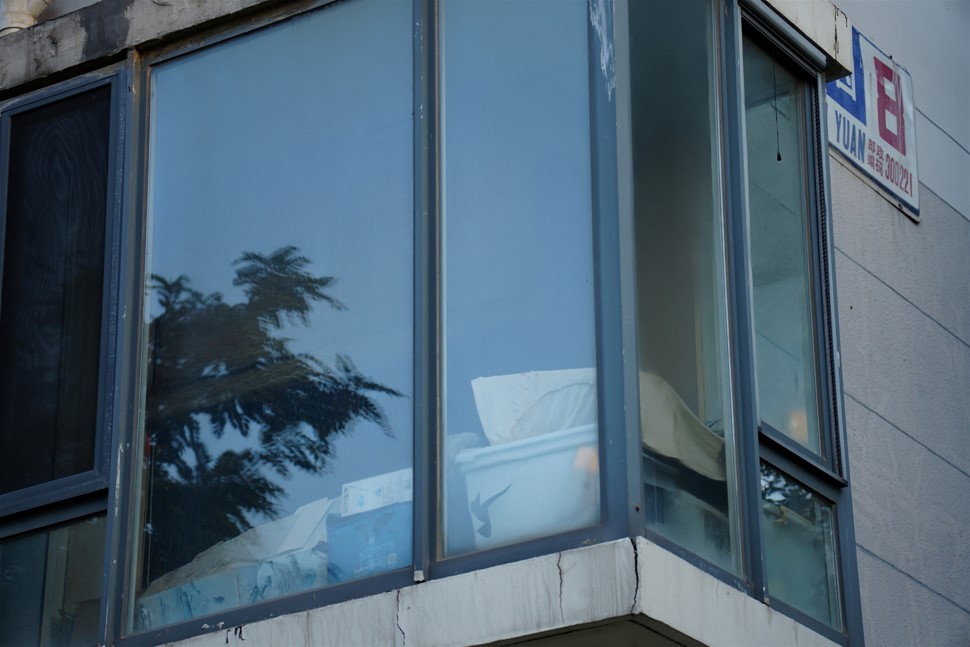}\vspace{2pt}
        \includegraphics[width=1\linewidth,height=65pt]{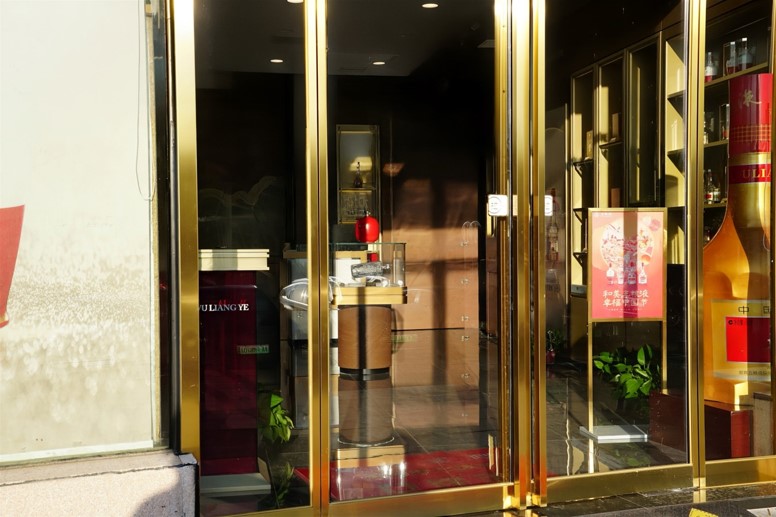}\vspace{2pt}
        \subcaption*{HUAWEI Mate 80 PM}
    \end{subfigure}
    \begin{subfigure}{0.19\linewidth}
        \includegraphics[width=1\linewidth,height=65pt]{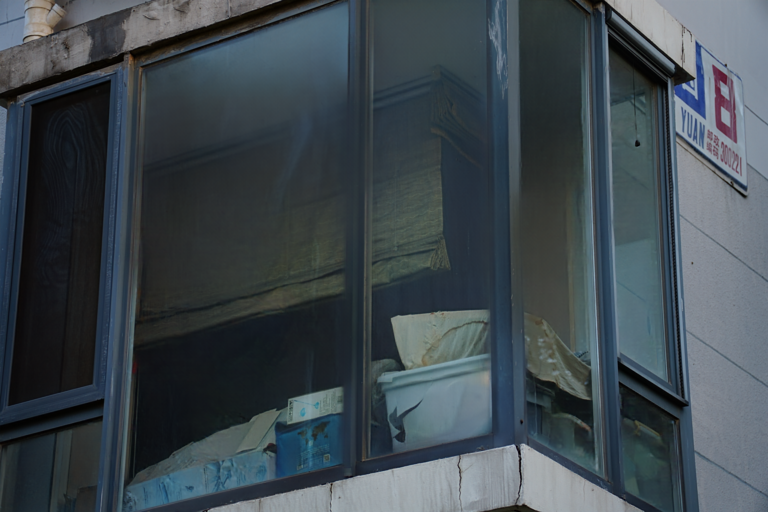}\vspace{2pt}
        \includegraphics[width=1\linewidth,height=65pt]{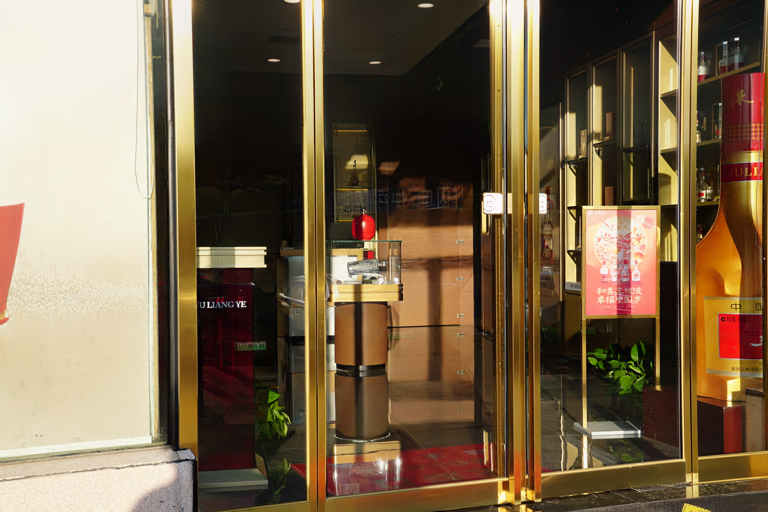}\vspace{2pt}
        \subcaption*{DIRS-PAIR (Ours)}
    \end{subfigure}
     }
     \caption{Visual comparison of transmission predictions in real-world scenarios. Notably, in addition to academic state-of-the-art models, we benchmark against the proprietary reflection removal algorithms built into the latest flagship smartphones.}
     \label{fig:real_world_extra}
\end{figure*}

\section{Experimental Validation}

This section evaluates our proposed DIRS framework. Section~\ref{sec:imp_details} provides the implementation details. Section~\ref{sec:performance} compares our best-performing variant, DIRS-PAIR, against state-of-the-art methods, followed by comprehensive ablation studies in Section~\ref{sec:ablation}. Extended applications, including \emph{reflection scene reconstruction} and \emph{polarized multi-image separation}, are provided in the subsequent sections to further demonstrate the generalization capability of our design.

\subsection{Implementation Details} 
\label{sec:imp_details}

\noindent\textbf{Datasets}. Our training datasets comprise both synthetic and real-world image pairs. For fair comparisons with existing state-of-the-art methods, we follow established protocols and employ two different data settings: 
\textit{Setting I:} In each epoch, we utilize 7,643 synthesized pairs randomly sampled from the PASCAL VOC dataset \cite{ijcv/EveringhamGWWZ10}, alongside 90 real-world pairs from Real20 \cite{cvpr/ZhangNC18a}.
\textit{Setting II:} We introduce 200 extra real-world pairs from the ``Nature'' dataset \cite{cvpr/LiY0LH20}, and scale up the synthetic data by sampling 13,700 pairs from the dataset of Zhang \emph{et al.} \cite{cvpr/ZhangNC18a} to establish a larger training pool.

\noindent\textbf{Data Synthesis and Completion}. 
To provide robust dual-stream supervision, we prepare both synthetic and real triplets. For synthetic data, we employ our Independent multiplicative model defined as $\mathbf{I}_{\textrm{syn}} = a\mathbf{T} + b\mathbf{R} + c(\mathbf{T}\circ\mathbf{R}) + z$ to generate physically plausible triplets. For real data, mainstream datasets provide only paired observations consisting of the composite image $\mathbf{I}_\textrm{real}$ and the transmission layer $\mathbf{T}$. Rather than using the flawed $\mathbf{I}_\textrm{real} - \mathbf{T}$ and forcing the network to absorb non-linear biases, we infer structurally decoupled pseudo reflection layers $\hat{\mathbf{R}}_\textrm{comp}$, as detailed in Sec.~\ref{sec:ref_comp}. Utilizing pseudo triplets consisting of $\mathbf{I}_\textrm{real}$, $\mathbf{T}$, and $\hat{\mathbf{R}}_\textrm{comp}$ provides explicit supervision for both branches, facilitating a more stable dual-stream optimization.

\noindent\textbf{Training Strategy}. The proposed DIRS models are implemented in PyTorch and optimized using the Adam optimizer on a single NVIDIA RTX 3090 GPU. The models are trained for 20 or 80 epochs depending on the data setting (Setting I or II, respectively). The learning rate is fixed at $10^{-4}$ with a batch size of 1 for all experiments. During training, patches of size $256 \times 256$ and $384 \times 384$ are randomly cropped from the input images for the CNN-based and Transformer-based DIRS variants, respectively.

\subsection{Performance Evaluation}
\label{sec:performance}

\noindent\textbf{Quantitative Comparison}. As reported in Tables~\ref{tab:qcomp} and \ref{tab:nature}, we conduct a comprehensive comparison between our DIRS-PAIR model and state-of-the-art methods across five real-world testing benchmarks: Real20 \cite{cvpr/ZhangNC18a}, Nature \cite{cvpr/LiY0LH20}, and three subsets of the $\textrm{SIR}^2$ dataset \cite{iccv/WanSDTK17}. Notably, our model trained under both data settings consistently demonstrates a compelling overall advantage across the evaluated data distributions. This holds true even when compared against recent diffusion-based models \cite{eccv/hong2024differ} and methods leveraging additional real-world data or extra language prompts \cite{cvpr/ZhongHWLS24}. This consistent superiority is attributed to the unified DIRS architecture, which enables dynamic information exchange to disentangle the coupled layers. More importantly, by employing Dual-Stream Joint Attention as an explicit similarity-based soft gate, the PAIR block prevents the indiscriminate aggregation of noisy cross-stream features and selectively captures long-range context, thereby safeguarding the information purity of the recovered background. These quantitative advantages are further corroborated by the subsequent visual comparisons, which demonstrate our effectiveness in eliminating large-area reflections.

\noindent\textbf{Qualitative Comparison}. As shown in Fig.~\ref{fig:visual_comp} on a challenging sample from the Real20 dataset, earlier methods such as BDN and ERRNet struggle to suppress strong reflections, while Zhang \emph{et al.} often introduces color distortion and artifacts. Although recent methods like RRW and RDNet improve the results, they still leave reflection residuals under strong coupling. By contrast, DIRS-PAIR combines LORS modeling with attention based decoupling to separate reflections of varying intensities and recover cleaner transmission layers with higher fidelity. Fig.~\ref{fig:visual_comp_r} further shows the reflection layer predictions on the $\text{SIR}^2$ dataset. Unlike previous methods such as IBCLN and Dong \emph{et al.}, which often produce noisy reflection predictions entangled with transmission remnants, DIRS-PAIR yields much more coherent reflection layers through LORS based residual routing and dual stream attention. To further validate the generalization capability in unconstrained environments, we provide visual comparisons on newly captured challenging real world scenarios in Fig.~\ref{fig:real_world_extra}. Notably, this evaluation not only benchmarks against academic state of the art models but also includes the proprietary algorithms built into the latest flagship smartphones (VIVO X300 Ultra, OPPO Find X9 Pro, and HUAWEI Mate 80 PM). Existing academic methods often struggle with severe nonlinear couplings, leaving obvious hazy residuals (\emph{e.g.}, ERRNet and IBCLN) or suffering from global color degradation (\emph{e.g.}, RRW). Regarding commercial algorithms, typical on-device pipelines (\emph{e.g.}, VIVO and OPPO) are highly optimized for general photography but lack the capacity to physically decouple strong spatial overlaps. The HUAWEI Mate 80 PM employs a more advanced cloud-deployed diffusion method, demonstrating powerful suppression in certain scenarios; however, it can occasionally lead to inconsistent generalization across varying lighting conditions, restricted by its transmission-only predictions and linear formulation. In contrast, our DIRS-PAIR demonstrates highly robust and consistent performance across these diverse challenging cases. Benefiting from explicit dual stream feature interaction and the mathematically grounded nonlinear formation prior, it thoroughly eliminates stubborn reflection remnants while faithfully preserving the structural integrity and original color fidelity of the background scene. Overall, these broad advantages support the effectiveness of our proposed unified paradigm. A detailed limitation analysis is provided in Sec.~\ref{sec:limitations}.

\subsection{Analysis of DIRS Variants}
\label{sec:variants_analysis}

To accommodate diverse deployment scenarios from edge devices to high-performance platforms, our proposed DIRS framework is highly configurable. As introduced, we instantiate it into three variants: the activation-based CNN DIRS-YTMT, the gate-based CNN DIRS-MuGI, and the attention-based Transformer DIRS-PAIR. Here, we comprehensively analyze these variants, evaluating both their qualitative visual differences and quantitative performance-complexity trade-offs which we detail below.

\begin{table*}[t]
\centering
\caption{Comparison of performance, complexity, and inference time. FLOPs and Inference Time are measured on a single NVIDIA RTX 3090 GPU with an input resolution of $256 \times 256$. The inference time is reported as the equivalent single-image latency measured at a batch size of 16. PSNR and SSIM are averaged over the Real20 and SIR$^2$ datasets. }
\label{tab:complexity}
\begin{tabular}{lcccccc}
\toprule
Method & Venue & Params (M) $\downarrow$ & FLOPs (G) $\downarrow$ & Time (ms) $\downarrow$ & PSNR $\uparrow$ & SSIM $\uparrow$ \\
\midrule
Zhang \emph{et al.} \cite{cvpr/ZhangNC18a}  & CVPR,18 & 0.39 & 50.27 & 8.88 & 20.08 & 0.835 \\
ERRNet \cite{cvpr/WeiYFW019}                & CVPR, 19 & 18.95 & 413.01 & 34.39 & 23.53 & 0.879 \\
IBCLN \cite{cvpr/LiY0LH20}                  & CVPR, 20 & 21.61 & 262.03 & 27.66 & 24.10 & 0.879 \\
Dong \emph{et al.} \cite{iccv/Dong00BXL21}  & ICCV, 21 & 10.93 & 295.53 & 37.86 & 24.21 & 0.897 \\
RobustSIRR \cite{cvpr/song2023robust}       & CVPR, 23 & 17.71 & 29.82 & 11.18 & 22.59 & 0.889 \\
RDNet \cite{cvpr/zhao2025reversible}        & CVPR, 25 & 264.96 & 243.83 & 58.07 & 25.95 & 0.909 \\
DExNet \cite{pami/huang2025lightweight}     & TPAMI, 25 & 9.66 & 240.88 & 103.52 & 25.96 & 0.912 \\
\midrule
DIRS-YTMT (Ours) & - & 32.42 & 102.91 & 31.35 & 24.94 & 0.902 \\
DIRS-MuGI (Ours) & - & 84.47 & 153.98 & 49.95 & 25.63 & 0.913 \\
DIRS-PAIR (Ours) & - & 48.80 & 200.22 & 75.36 & {26.37} & {0.918} \\
\bottomrule
\end{tabular}
\end{table*}

\begin{figure*}[t]
    \newcommand{\imgH}{4.8cm} 
    
    \begin{subfigure}[t]{0.40\linewidth}
        \centering
        \includegraphics[width=\linewidth, height=\imgH]{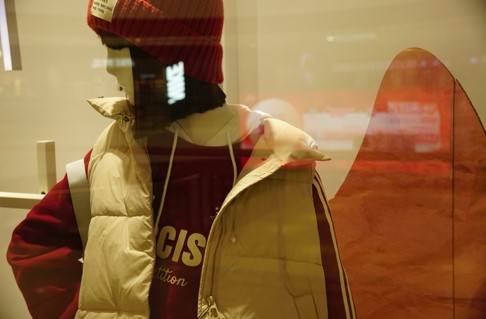}
        \\[3pt]
        \includegraphics[width=\linewidth, height=\imgH]{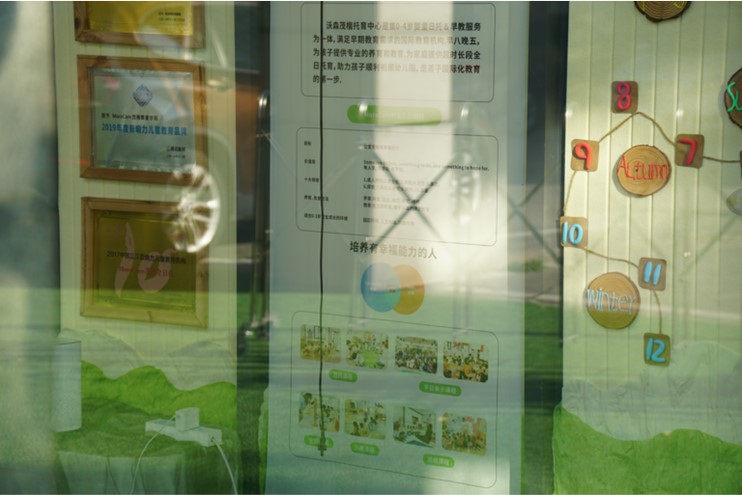}
        \caption{Input $(\textbf{I})$}
    \end{subfigure}
    \hfill
    \begin{subfigure}[t]{0.193\linewidth}
        \centering
        \includegraphics[width=\linewidth, height=\imgH]{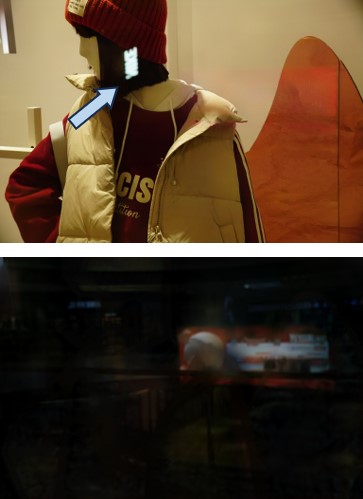}
        \\[3pt]
        \includegraphics[width=\linewidth, height=\imgH]{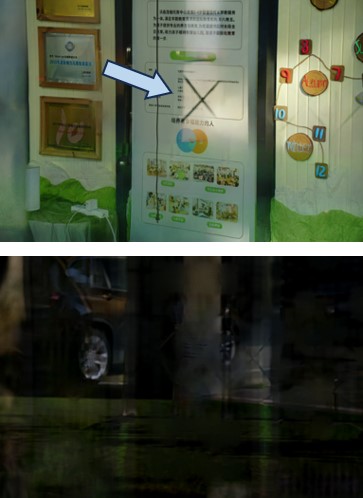}
        \caption{DIRS-YTMT $(\hat{\mathbf{T}}, \hat{\mathbf{R}})$}
    \end{subfigure}
    \hfill
    \begin{subfigure}[t]{0.193\linewidth}
        \centering
        \includegraphics[width=\linewidth, height=\imgH]{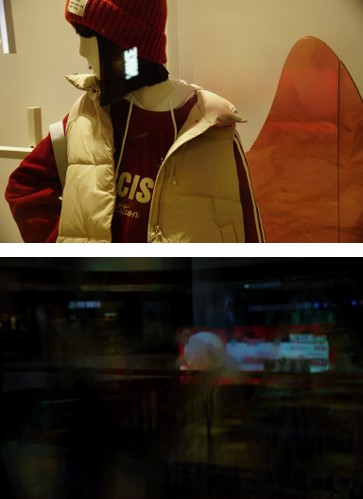}
        \\[3pt]
        \includegraphics[width=\linewidth, height=\imgH]{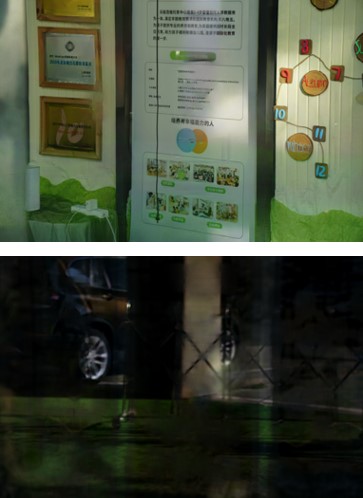}
        \caption{DIRS-MuGI $(\hat{\mathbf{T}}, \hat{\mathbf{R}})$}
    \end{subfigure}
    \hfill
    \begin{subfigure}[t]{0.193\linewidth}
        \centering
        \includegraphics[width=\linewidth, height=\imgH]{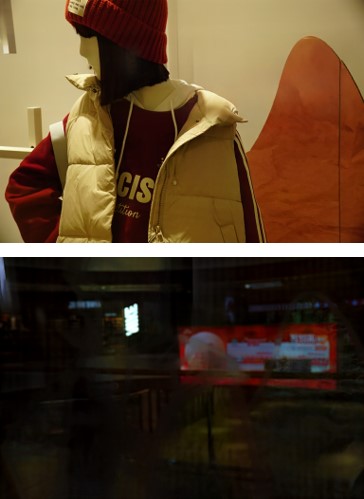}
        \\[3pt]
        \includegraphics[width=\linewidth, height=\imgH]{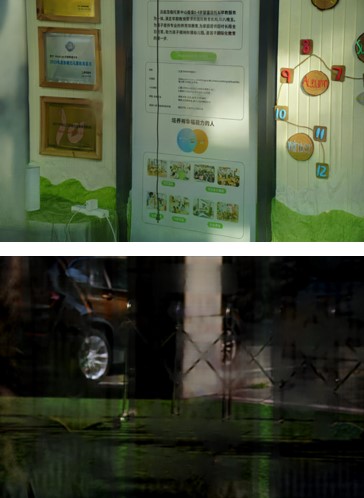}
        \caption{DIRS-PAIR $(\hat{\mathbf{T}}, \hat{\mathbf{R}})$}
    \end{subfigure}
  \caption{Visual comparison of layer predictions among our three variants of the DIRS architecture. }
  \label{fig:ourown}
  \vspace{-10pt}
\end{figure*}

\subsubsection{Visual Comparison of DIRS Variants}
Figure~\ref{fig:ourown} visually compares the layer separation results of the three variants on challenging real-world scenes. As highlighted by the blue arrows in Fig.~\ref{fig:ourown}~(b), the lightweight DIRS-YTMT removes most reflections but leaves subtle ghosting in severely coupled regions, as its dichotomous activation lacks dynamic modulation for extreme overlaps. DIRS-MuGI (Fig.~\ref{fig:ourown}~(c)) mitigates these artifacts via mutual gating, allowing active cross-stream semantic suppression to handle spatially varying nonlinearities effectively. Ultimately, the Transformer-based DIRS-PAIR (Fig.~\ref{fig:ourown}~(d)) achieves the cleanest separation. Leveraging Dual-Stream Joint-Attention (DS-JA), it explicitly assesses cross-stream similarities with a nonlocal receptive field, thoroughly eliminating stubborn remnants while perfectly preserving the structural fidelity of the transmission layer.

\subsubsection{Performance vs. Complexity Trade-offs}
To evaluate the practicability of our framework, we conduct a holistic comparison of quantitative performance (PSNR, SSIM) and computational complexity (Params, FLOPs, Inference Time). The results, benchmarked against recent state-of-the-art methods, are summarized in Table~\ref{tab:complexity}. As reported, DIRS-YTMT serves as a highly efficient baseline. It maintains a low computational footprint (102.91 GFLOPs) and fast inference speed (31.35 ms latency), yet still outperforms several earlier heavy models (\emph{e.g.}, ERRNet and IBCLN) in both quality and efficiency, making it well-suited for real-time or resource-constrained applications. DIRS-MuGI offers a balanced compromise, achieving competitive performance (25.63 dB) with a moderate computational budget. DIRS-PAIR pushes the absolute upper bound of separation quality (26.37 dB), establishing a new state-of-the-art. While it requires the most computational resources among our variants due to its Transformer backbone and non-local attention operations, it is worth noting that DIRS-PAIR remains significantly more computationally efficient (200.22 GFLOPs) than many recent top-performing networks (\emph{e.g.}, RDNet at 243.83 GFLOPs and DExNet at 240.88 GFLOPs). Overall, these variants demonstrate the inherent efficiency and adaptability of our proposed methods. By offering flexible configurations tailored to specific deployment environments, the DIRS paradigm serves as a scalable framework that effectively facilitates dual-stream feature interactions across diverse complexity levels.

\subsection{Ablation Study}
\label{sec:ablation}
As shown in Table~\ref{tab:ablation}, we conduct a comprehensive ablation study across the three DIRS variants. The experiments investigate the overall architecture, block designs, physical modeling constraints, and data synthesis strategies to validate our proposed components and architectural choices.

\vspace{1mm}\noindent\textbf{Ablation Study on DIRS-YTMT.} We first evaluate the activation-based DIRS-YTMT. Reducing it to a single-stream architecture causes a noticeable performance drop, confirming that a dedicated dual-branch structure is essential to constrain the reflection layer and regularize the separation process. We then ablate the DSI Block (DSIB) via two variants: (1) \textit{w/o DSI}, which retains negative-domain features locally instead of exchanging them, and (2) \textit{w/o PAF}, which replaces the paired activation strategy with standard ReLUs. Both suffer performance reductions, verifying that our YTMT strategy's dichotomous feature separation maximizes utilization efficiency and suits dual-stream complementarity. Furthermore, removing the LORS constraints (\textit{w/o LORS}) causes a severe drop (from 24.94 to 24.16 dB). Without LORS bypassing high-order residuals and biases, the intermediate convolutions are forced to absorb these non-linear semantics, disrupting feature separation and degrading fidelity. Finally, replacing our multiplicative model with a simplistic \textit{Linear Model} for data synthesis yields the lowest performance (24.09 dB). This corroborates that explicit physical non-linearities are indispensable for bridging the synthetic-to-real domain gap during model training.

\begin{table}[t]
\caption{Ablation on design factors across three DIRS variants. All metrics are averaged over Real20 and SIR$^2$.}
\label{tab:ablation}
\centering
\begin{tabular}{c|clcc}
 \toprule[1pt] 
 \bigstrut[b] Model & Designs & Variants & PSNR  & SSIM \\
 \hline
 \bigstrut \multirow{6}{*}{DIRS-YTMT} & Arch. & Single Stream & 24.76 & 0.897 \\
 \cline{2-5}
 \bigstrut[t] & \multirow{2}{*}{DSIB} & w/o DSI & 24.73 & 0.899 \\
 \bigstrut & & w/o PAF & 24.82 & 0.899  \\
 \cline{2-5}
 \bigstrut & \multirow{1}{*}{LORS} & w/o LORS & 24.16 & 0.894 \\
 \cline{2-5}
 \bigstrut[t] & \multirow{1}{*}{Data} & Linear Model & 24.09 & 0.893 \\
 \cline{2-5}
 \bigstrut & \multicolumn{2}{c}{Full Setting} &  \textbf{24.94} &  \textbf{0.902} \\ 
 \hline
 \multirow{9}{*}{DIRS-MuGI} 
 \bigstrut & \multirow{1}{*}{Arch.} & Single Stream & 24.79 & 0.899 \\
 \cline{2-5}
 \bigstrut & \multirow{1}{*}{DSIB} & w/o DSI & 24.99 & 0.903 \\
 \cline{2-5}
 \bigstrut & \multirow{4}{*}{LORS} & w/o LORS & 24.78 & 0.899 \\
 & & w/o LZO ($\boldsymbol{\Psi}$) & 25.17 & 0.902 \\
 \bigstrut & & w/o LNR ($\boldsymbol{\Phi}$) & 25.39 & 0.905 \\
 \cline{2-5}
 \bigstrut[t] & \multirow{1}{*}{Data} & Linear Model & 25.06 & 0.905 \\
 \cline{2-5}
 \bigstrut & \multicolumn{2}{c}{Full Setting} &  \textbf{25.63} &  \textbf{0.913} \\
 \hline
 \multirow{9}{*}{DIRS-PAIR} 
 \bigstrut & \multirow{1}{*}{Arch.} & Single Stream & 25.67 & 0.909 \\
 \cline{2-5}
 \bigstrut & \multirow{4}{*}{DSIB} & Standard FFN & 25.31 & 0.905  \\
 & & w/o DS-JA & 25.83 & 0.915 \\
 \bigstrut & & w/o DS-SA & 25.05 & 0.916\\
 \cline{2-5}
 \bigstrut & \multirow{1}{*}{LORS} & w/o LORS & 24.78 & 0.899 \\
 \cline{2-5}
 \bigstrut[t] & \multirow{1}{*}{Data} & Linear Model & 25.21 & 0.903 \\
 \cline{2-5}
 \bigstrut[t] & \multicolumn{2}{c}{Full Setting} &  \textbf{26.37} &  \textbf{0.918} \\
\bottomrule[1pt]
\end{tabular}
\end{table}

\vspace{1mm}\noindent\textbf{Ablation Study on DIRS-MuGI.} We next evaluate the gate-based DIRS-MuGI. Reverting to a single-stream architecture causes a pronounced performance drop. Similarly, the dual-branch configuration without explicit interaction (\textit{w/o DSI}) severely underperforms. This verifies that isolated processing cannot handle the spatially non-uniform coupling of real-world reflections. Instead, the network relies on MuGI's multiplicative modeling to modulate the features and capture the mutually inhibitive nature of the transmission and reflection streams. Additionally, as analyzed in the YTMT variant, removing the entire LORS model or employing the \textit{Linear Model} for data synthesis severely degrades performance. Separately ablating its sub-components, specifically the non-linear residual (\textit{w/o LNR, $\boldsymbol{\Phi}$}) and zero-order offset (\textit{w/o LZO, $\boldsymbol{\Psi}$}), causes measurable declines, empirically proving that both high-order physical interactions and ambient biases are non-negligible in sRGB superposition.

\vspace{1mm}\noindent\textbf{Ablation Study on DIRS-PAIR.} This part evaluates DIRS-PAIR. Replacing the modulated feed-forward network with a Standard FFN~\cite{corr/BahdanauCB14} causes a performance drop, further confirming the importance of explicit multiplicative modeling. We ablate the two core components of the PAIR module: the Dual-Stream Joint Attention (\textit{w/o DS-JA}) and the Dual-Stream Self-Attention (\textit{w/o DS-SA}). Removing either component leads to decreased metrics. DS-JA functions as a soft gate during the selective inter-stream information exchange, whereas DS-SA maintains intra-stream structural coherence, confirming that explicit cross-stream modulation and spatial self-attention act effectively together. Moreover, consistent with previous variants, removing the LORS constraints (\textit{w/o LORS}) degrades performance to 24.78 dB, and employing the \textit{Linear Model} for data synthesis lowers it to 25.21 dB. The consistent performance drops observed across all three DIRS variants under these settings summarize a core finding of our framework: explicitly modeling physical non-linearities, both internally through the LORS architecture and externally via the data synthesis pipeline, is essential for robust real-world reflection separation.

\section{Extended Experiments}
\label{sec:extend_exps}
While the preceding evaluations establish the efficacy of DIRS-PAIR in standard single-image reflection separation, the inherent flexibility of our unified dual-stream paradigm allows for broader physical applications. In this section, we extend our framework to two advanced scenarios. First, in \textit{Reflection Scene Reconstruction}, we leverage the explicitly separated reflection components and our non-linear synthesis model to recover the severely degraded reflection scene. Second, we adapt our architecture to \textit{Polarized Image Reflection Separation}, integrating multi-view polarization cues to resolve intrinsic ill-posedness and severe overexposure.

\begin{figure*}[t]
    \centering{
    \begin{subfigure}[t]{0.195\linewidth}
        \includegraphics[width=1\linewidth,height=65pt]{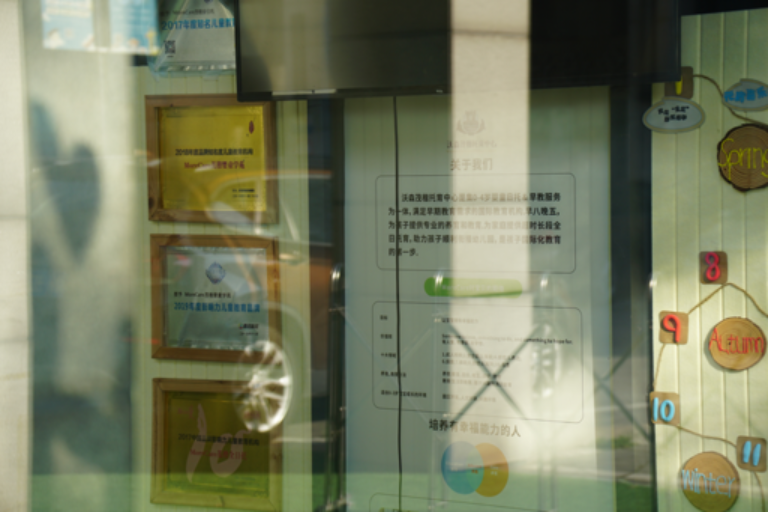}\vspace{2pt}
        \includegraphics[width=1\linewidth,height=65pt]{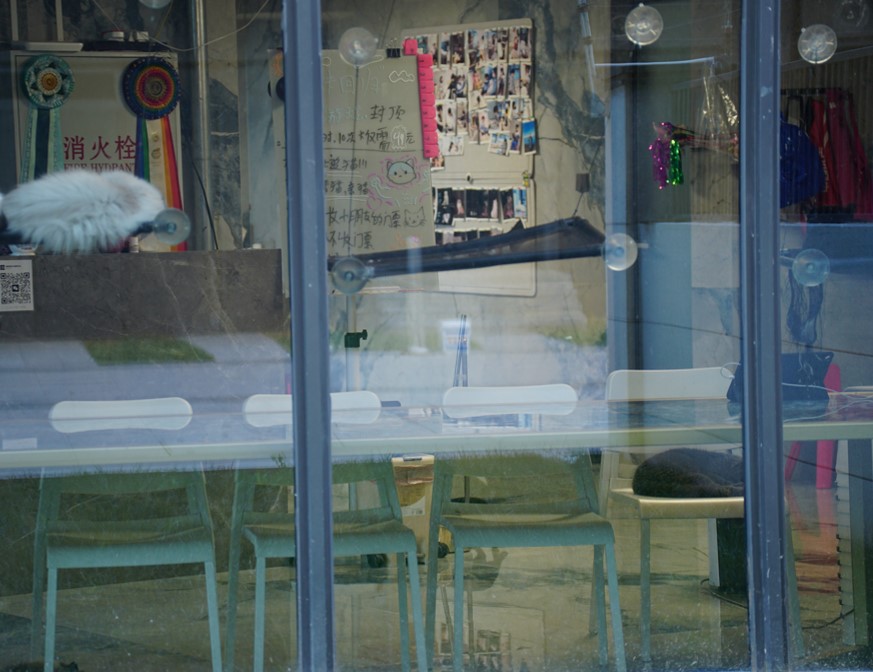}\vspace{2pt}
        \includegraphics[width=1\linewidth,height=65pt]{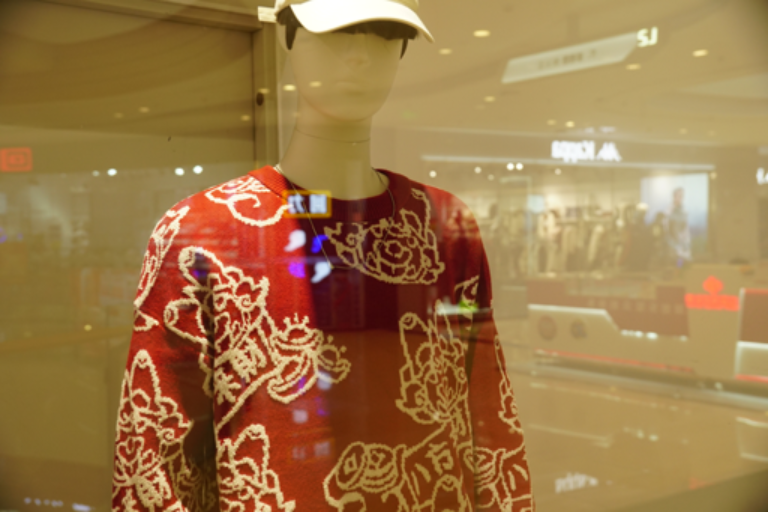}
        \subcaption*{Input}
    \end{subfigure}
    \begin{subfigure}[t]{0.195\linewidth}
        \includegraphics[width=1\linewidth,height=65pt]{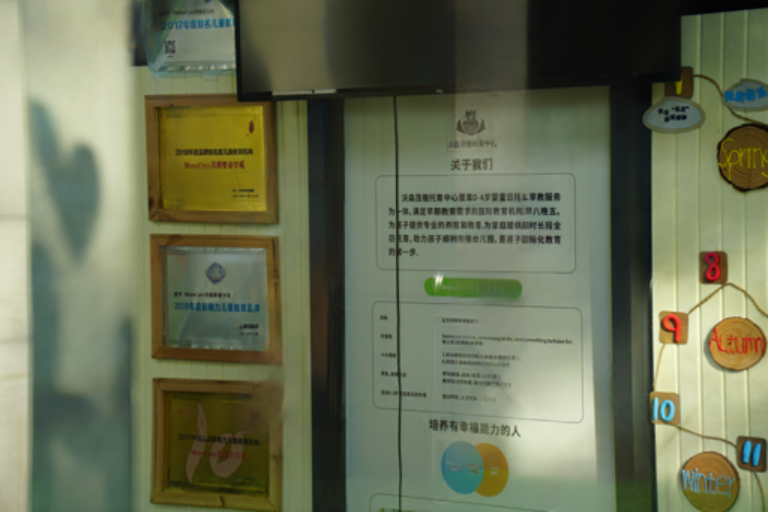}\vspace{2pt}
        \includegraphics[width=1\linewidth,height=65pt]{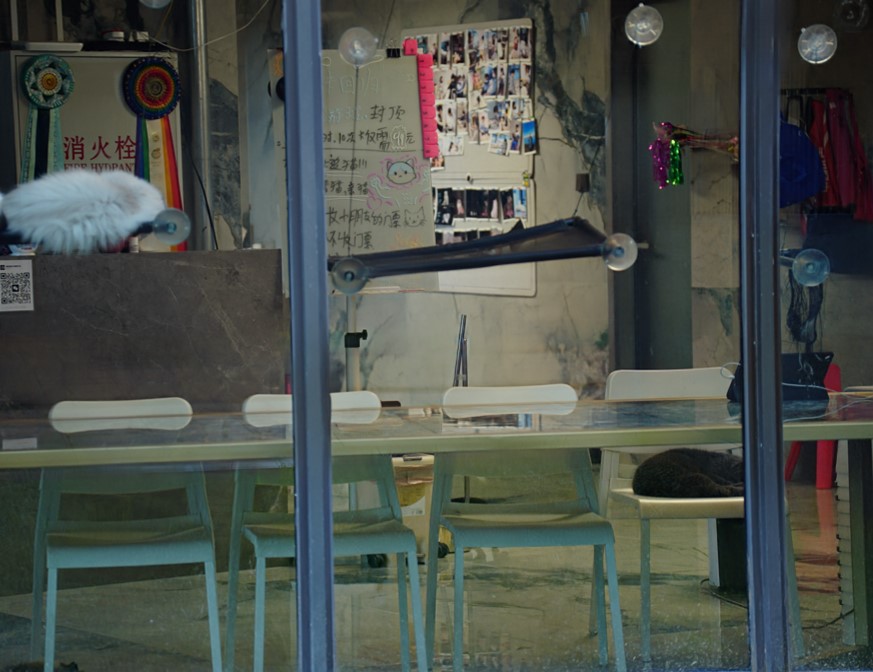}\vspace{2pt}
        \includegraphics[width=1\linewidth,height=65pt]{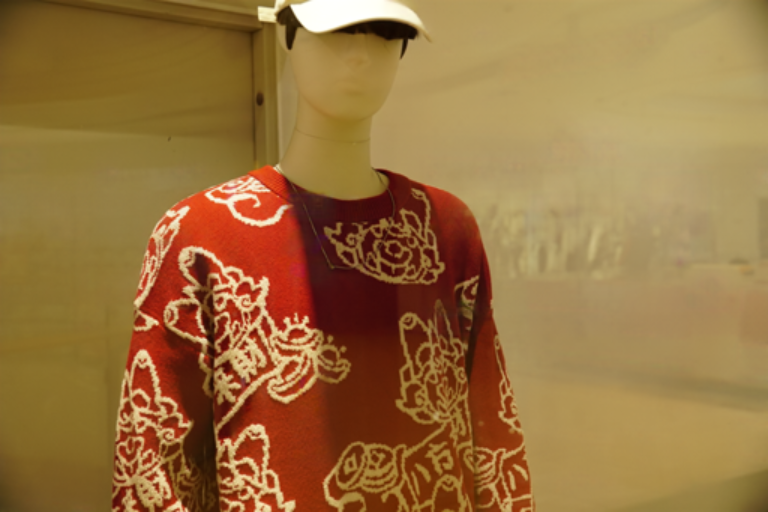}
        \subcaption*{Transmission}
    \end{subfigure}
    \begin{subfigure}[t]{0.195\linewidth}
        \includegraphics[width=1\linewidth,height=65pt]{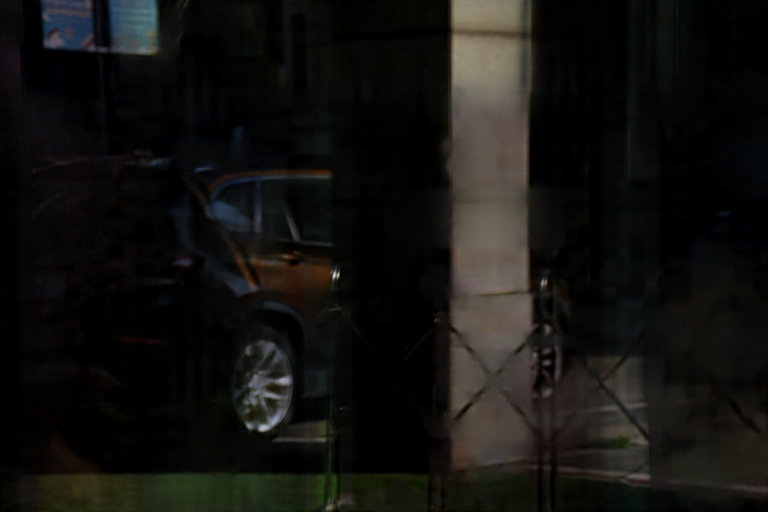}\vspace{2pt}
        \includegraphics[width=1\linewidth,height=65pt]{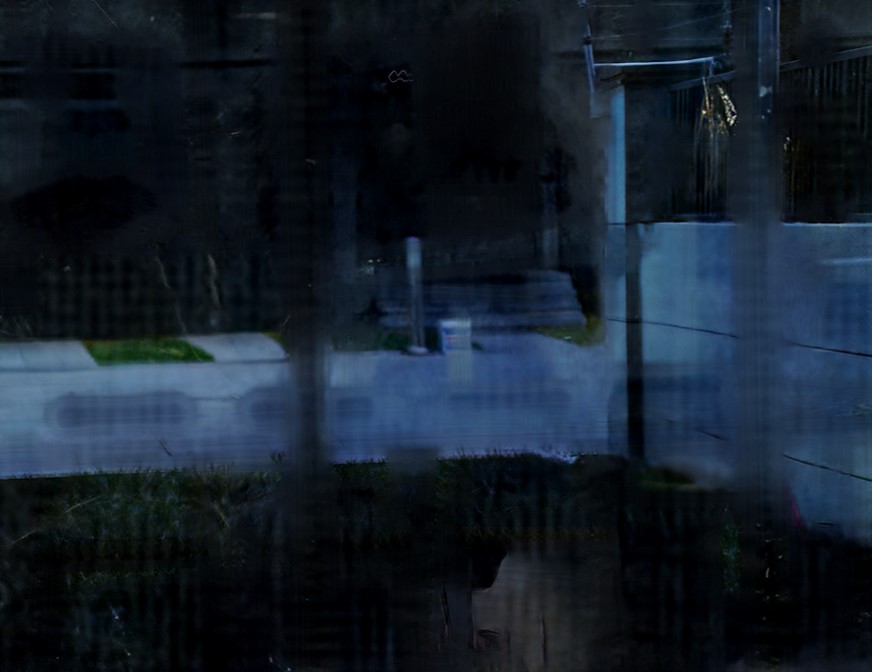}\vspace{2pt}
        \includegraphics[width=1\linewidth,height=65pt]{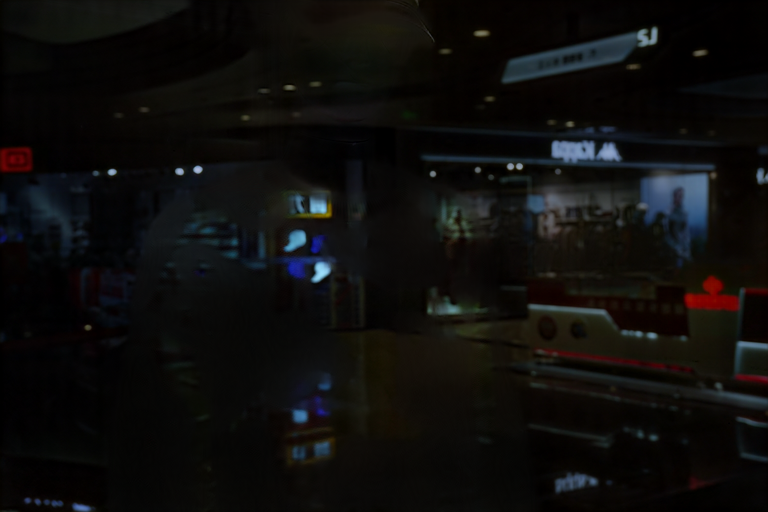}
        \subcaption*{Reflection}
    \end{subfigure}
    \begin{subfigure}[t]{0.195\linewidth}
        \includegraphics[width=1\linewidth,height=65pt]{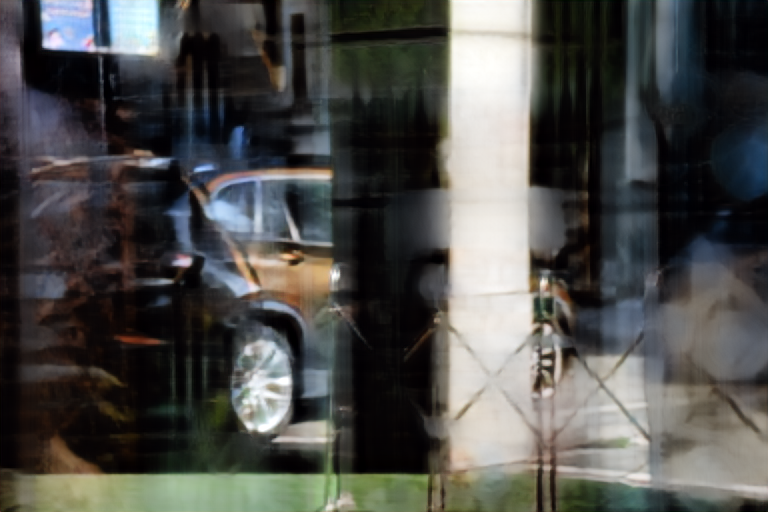}\vspace{2pt}
        \includegraphics[width=1\linewidth,height=65pt]{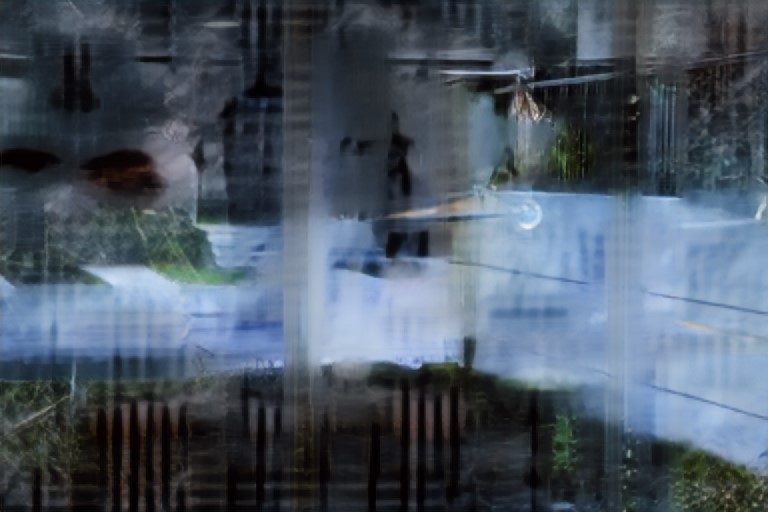}\vspace{2pt}
        \includegraphics[width=1\linewidth,height=65pt]{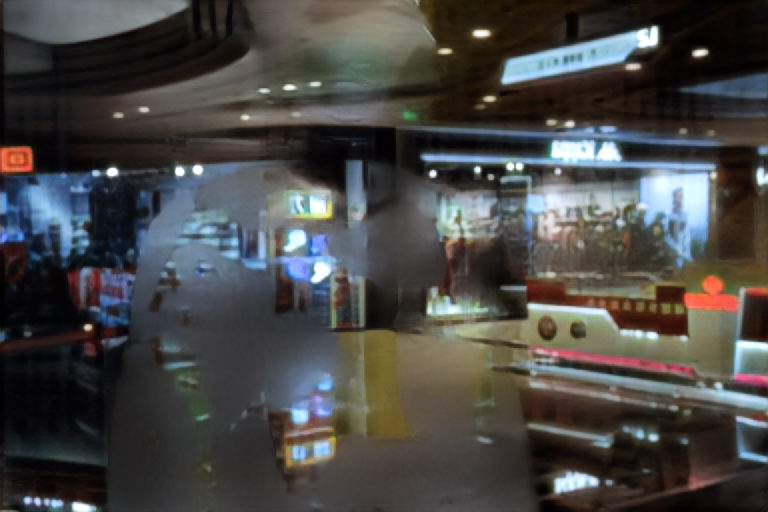}
        \subcaption*{Enhancement Baseline}
    \end{subfigure}
    \begin{subfigure}[t]{0.195\linewidth}
        \includegraphics[width=1\linewidth,height=65pt]{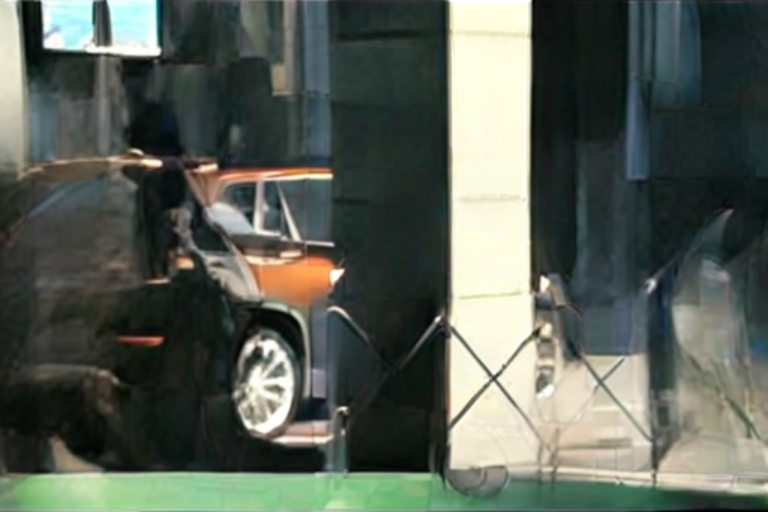}\vspace{2pt}
        \includegraphics[width=1\linewidth,height=65pt]{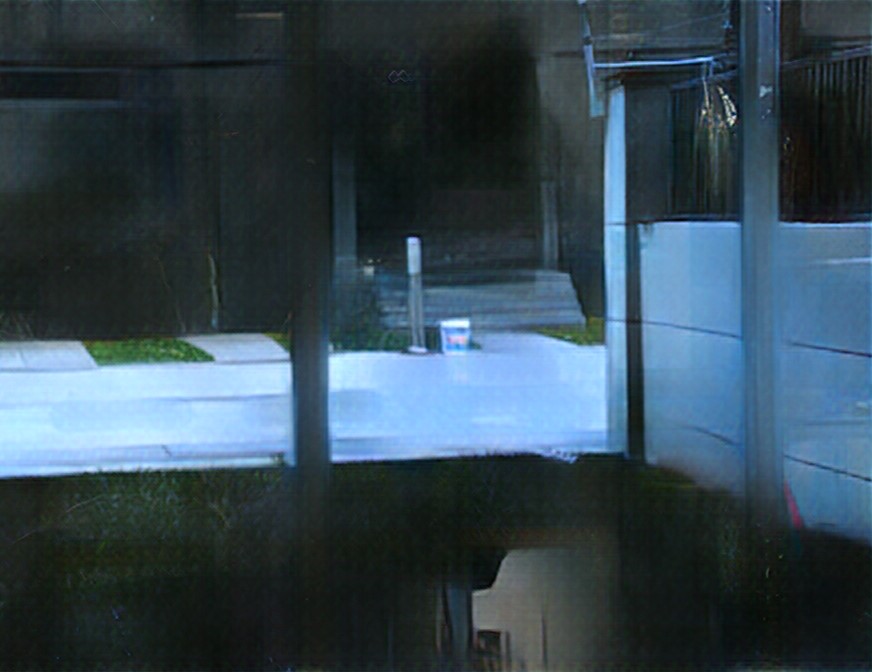}\vspace{2pt}
        \includegraphics[width=1\linewidth,height=65pt]{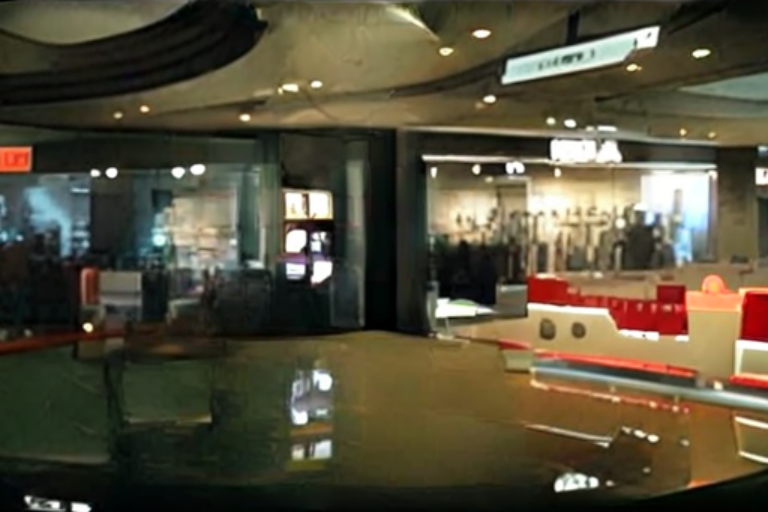}
        \subcaption*{Ours ($\hat{\mathbf{R}}_\text{scene}$)}
    \end{subfigure}
     }
     \caption{Reflection separation (Transmission and Reflection) and scene reconstruction (Reflection Scene) on real-world samples (Input). Reconstruction enhances reflection visibility and restores content lost during decomposition.}
     \label{fig:reflection_scene}
\end{figure*}

\subsection{Reflection Scene Reconstruction}
Following the transparent-surface optical path discussed in the main paper, the captured reflection $\mathbf{R}$ is essentially a degraded version of the original reflection scene $\mathbf{R}_\text{scene}$ due to the low reflectance and scattering of the glass. Defining this degradation process as $\mathbf{R} = \mathcal{F}_{\text{deg}}(\mathbf{R}_\text{scene})$, we can substitute it into our unified non-linear synthesis model formulated earlier, yielding:
\begin{equation}
    \mathbf{I}_{\text{syn}} = a\mathbf{T} + b \mathcal{F}_{\text{deg}}(\mathbf{R}_\text{scene}) + c(\mathbf{T} \circ \mathcal{F}_{\text{deg}}(\mathbf{R}_\text{scene})) + z.
    \label{eq:scene_syn}
\end{equation}

\begin{table}[t]
  \centering
  \caption{Quantitative comparison of Reflection Scene Reconstruction using advanced no-reference image quality assessment metrics on real-world samples. $\uparrow$ indicates higher is better. The best results are in \textbf{bold}.}
  \renewcommand{\arraystretch}{1.2}
    \begin{tabular}{lccc}
    \toprule
    Method & MUSIQ $\uparrow$ & TOPIQ-NR $\uparrow$ & CLIP-IQA $\uparrow$ \\
    \midrule
     Original $\hat{\mathbf{R}}$ & 25.14 & 0.188 & 0.186 \\
     Baseline $\mathcal{F}_{\textrm{enh}}(\hat{\mathbf{R}})$& 26.75 & 0.186 & 0.148  \\
    Ours $\hat{\mathbf{R}}_\text{scene}$ & \textbf{28.55} &  \textbf{0.204} & \textbf{0.289}  \\
    \bottomrule
    \end{tabular}
  \label{tab:nr_iqa}
\end{table}

Training a network to invert this complex physical process requires complete $(\mathbf{T}, \mathbf{R}_\text{scene}, \mathbf{I})$ triplets. However, real-world datasets typically only provide $(\mathbf{R}_\text{scene}, \mathbf{R})$ pairs~\cite{cvpr/WanSLDK20} or $(\mathbf{T}, \mathbf{I})$ pairs~\cite{cvpr/ZhangNC18a,cvpr/LiY0LH20,cvpr/zhu2024revisiting}, necessitating the use of synthetic data to close the training loop. One might attempt to learn a direct mapping $\mathcal{F}_{\text{enh}}: \mathbf{R} \to \mathbf{R}_\text{scene}$ for reflection enhancement. Yet, the ground-truth $\mathbf{R}$ in existing datasets is captured by placing a black cloth behind the glass, which avoids the structural truncation caused by ISP saturation in highly coupled scenarios. Consequently, an $\mathcal{F}_{\text{enh}}$ trained on such data functions more akin to a low-light enhancer, exhibiting limited capability for structural completion.

To achieve physically interpretable scene reconstruction, we repurpose the module $\mathcal{F}_C$ from Sec.~\ref{sec:ref_comp} to formulate as:
\begin{equation}
    \hat{\mathbf{R}}_\text{scene} = \hat{\mathbf{R}} + \mathcal{F}_C(\mathbf{I}, \hat{\mathbf{T}}).
\end{equation}
We constrain this process using a joint objective:
\begin{equation}
    \mathcal{L} = \|\hat{\mathbf{R}}_\text{scene} - \mathbf{R}_\text{scene}\| + \|\hat{\mathbf{I}}_{\text{syn}} - \mathbf{I}_{\text{syn}} \|,
\end{equation}
where $\hat{\mathbf{I}}_{\text{syn}} = a\mathbf{T} + b \mathcal{F}_{\text{deg}}(\hat{\mathbf{R}}_\text{scene}) + c(\mathbf{T} \circ \mathcal{F}_{\text{deg}}(\hat{\mathbf{R}}_\text{scene})) + z$. The mapping $\mathcal{F}_{\text{deg}}$ is pre-trained on the $(\mathbf{R}_\text{scene}, \mathbf{R})$ pairs~\cite{cvpr/WanSLDK20} and kept frozen during the training of $\mathcal{F}_C$. By shifting the mapping from a naive $\mathbf{R} \rightarrow \mathbf{R}_\text{scene}$ to a conditional $(\hat{\mathbf{R}}, \mathbf{I}, \hat{\mathbf{T}}) \rightarrow \hat{\mathbf{R}}_\text{scene}$, the model effectively learns to repair the structural damage inherent in the initial reflection prediction. As shown in Fig.~\ref{fig:reflection_scene}, our method achieves robust scene reconstruction, successfully recovering both visibility and lost structural content.
To quantitatively validate this, we report modern no-reference image quality assessment metrics in Tab.~\ref{tab:nr_iqa}, including MUSIQ~\cite{iccv/ke2021musiq}, TOPIQ-NR~\cite{tip/chen2024topiq}, and CLIP-IQA~\cite{aaai/wang2023exploring}. Compared to the separated raw reflection and the enhancement baseline, our physically guided reconstruction explicitly restores structural integrity and avoids unnatural distortions, consistently achieving the highest quantitative scores across all evaluators. The absolute scores remain relatively modest across all methods. This primarily stems from the inherent sparsity of reflection layers. Since the separated reflections often contain limited semantic and structural content, they naturally receive lower evaluations from general-purpose quality assessors trained on complete natural scenes. Furthermore, while aggressive processing could artificially inflate these metrics, prioritizing numerical gains at the expense of physical fidelity to the original reflection layer is fundamentally inappropriate. As demonstrated in our visual results, our method avoids unnatural hallucination, striking a careful balance between perceptual reconstruction and adherence to the physical formation.

\begin{figure*}[t]
    \centering
    \begin{subfigure}[t]{0.16\linewidth}
        \includegraphics[width=1\linewidth,height=230pt]{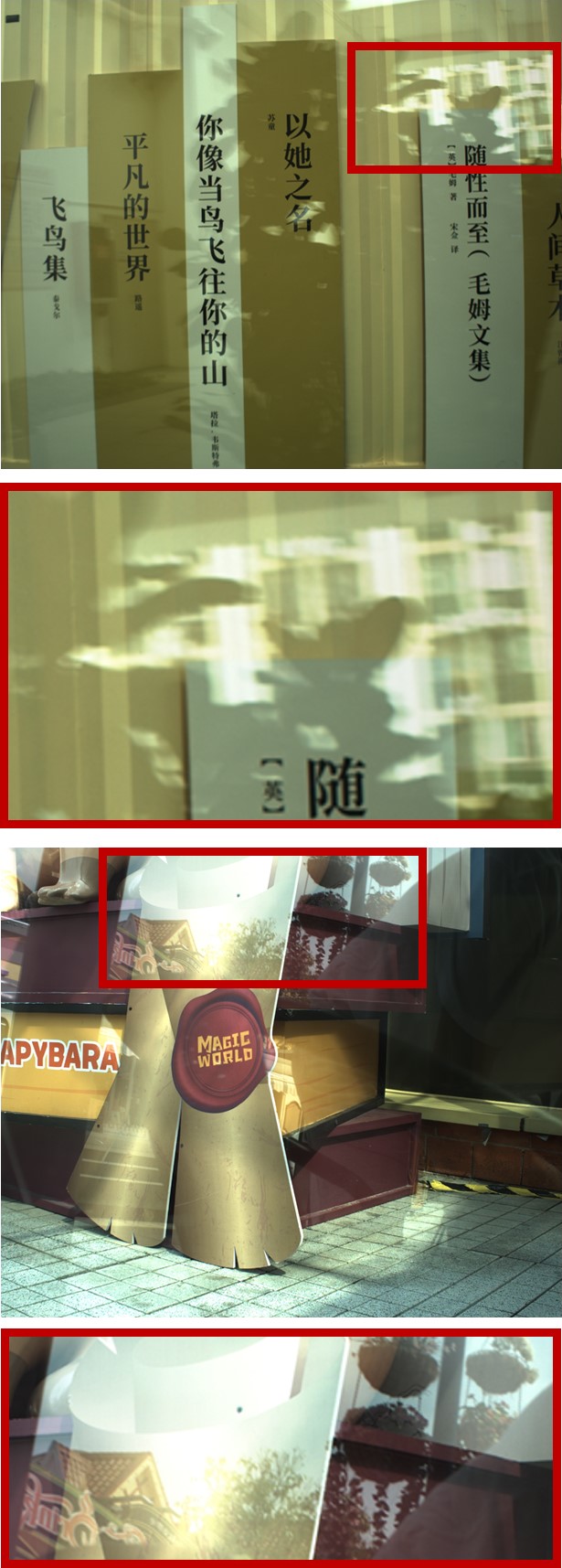}
        \subcaption*{RGB Input}
    \end{subfigure}
    \begin{subfigure}[t]{0.16\linewidth}
        \includegraphics[width=1\linewidth,height=230pt]{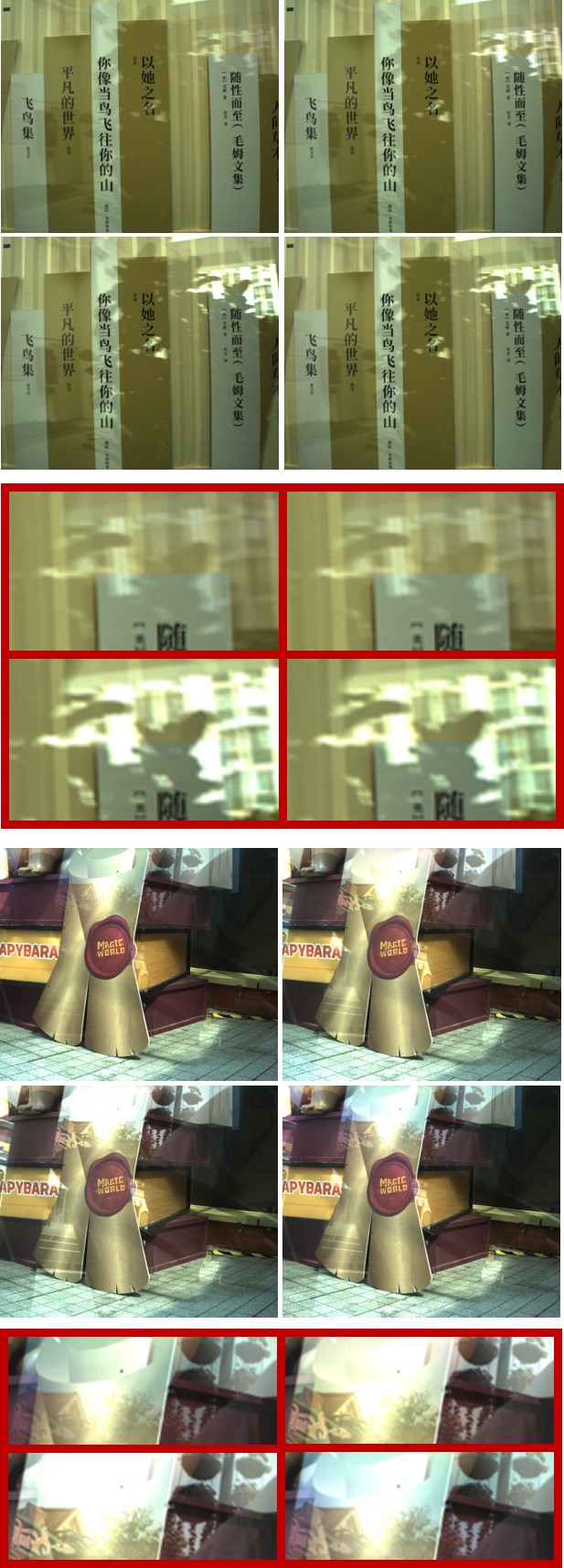}
        \subcaption*{Polarized Inputs}
    \end{subfigure}
    \begin{subfigure}[t]{0.16\linewidth}
        \includegraphics[width=1\linewidth,height=230pt]{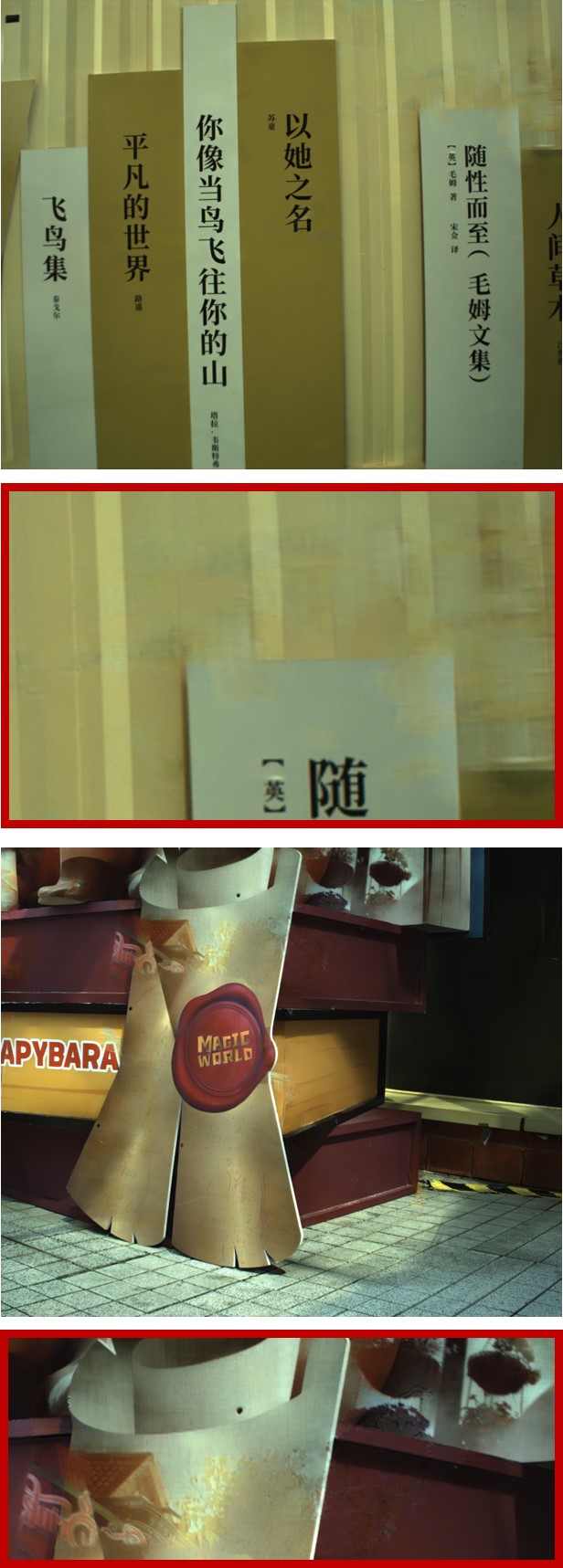}
        \subcaption*{PolarFree~\cite{cvpr/yao2025polarfree}}
    \end{subfigure}
    \begin{subfigure}[t]{0.16\linewidth}
        \includegraphics[width=1\linewidth,height=230pt]{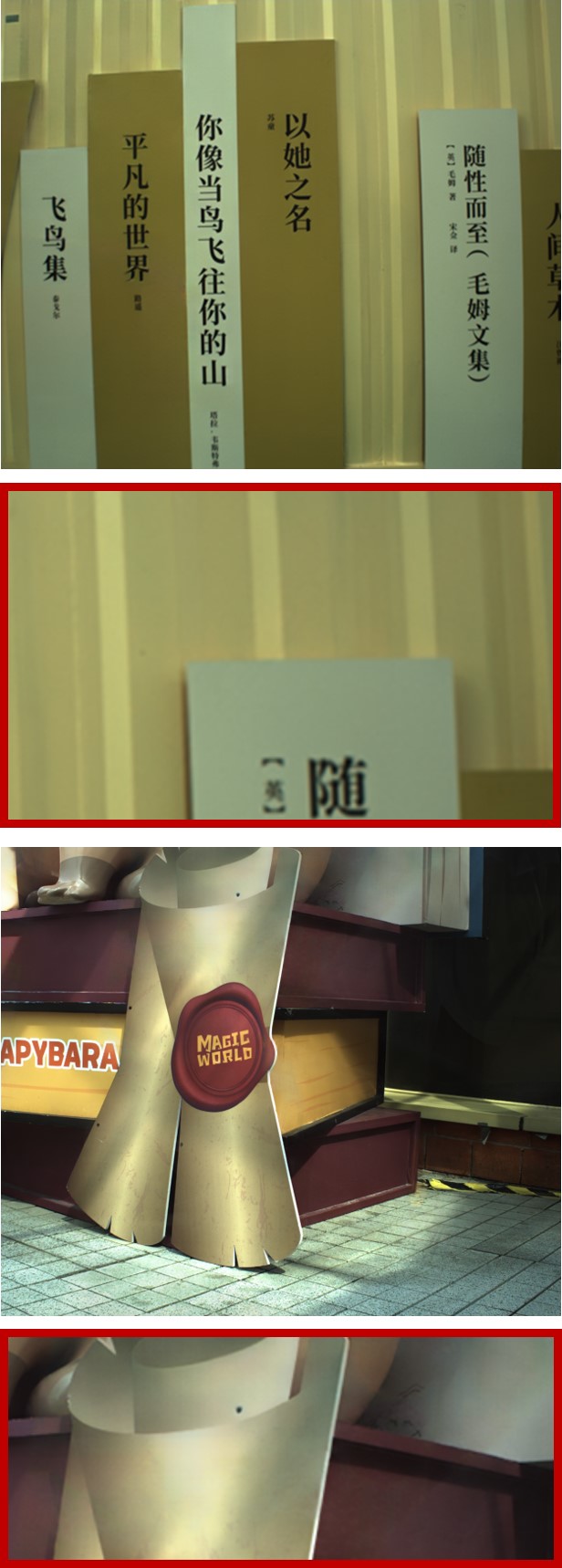}
        \subcaption*{Ours~($\hat{\textbf{T}}$)}
    \end{subfigure}
    \begin{subfigure}[t]{0.16\linewidth}
        \includegraphics[width=1\linewidth,height=230pt]{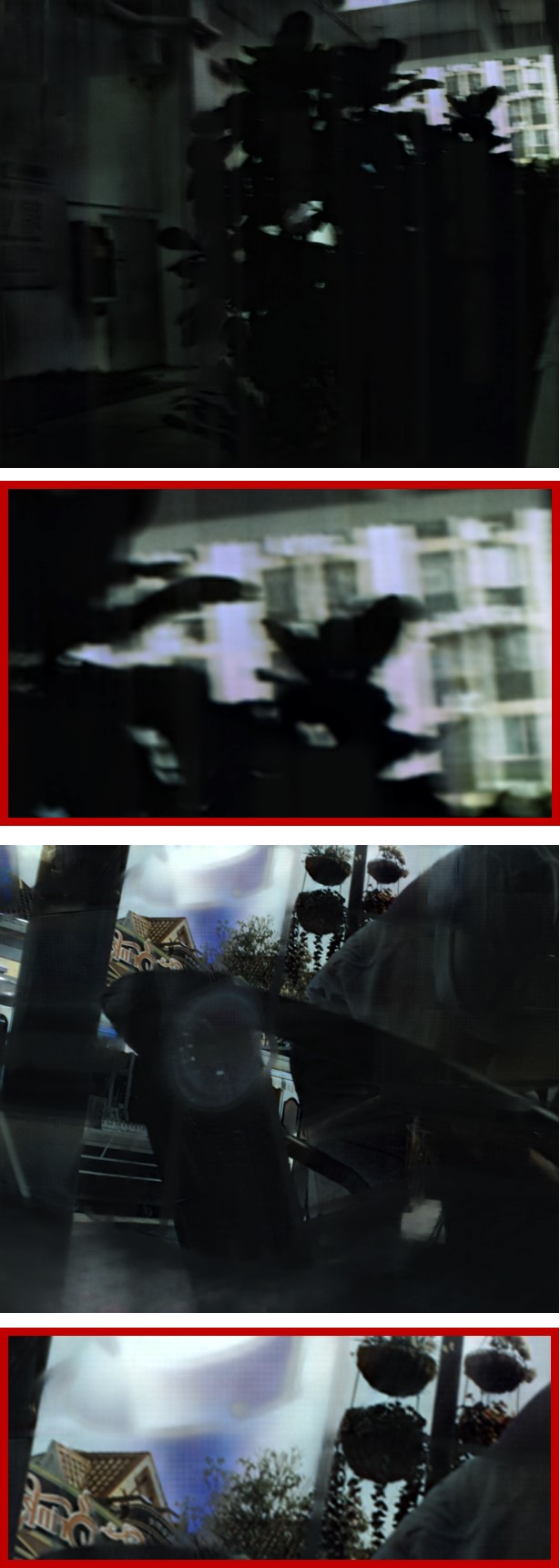}
        \subcaption*{Ours~($\hat{\textbf{R}}$)}
    \end{subfigure}
    \begin{subfigure}[t]{0.16\linewidth}
        \includegraphics[width=1\linewidth,height=230pt]{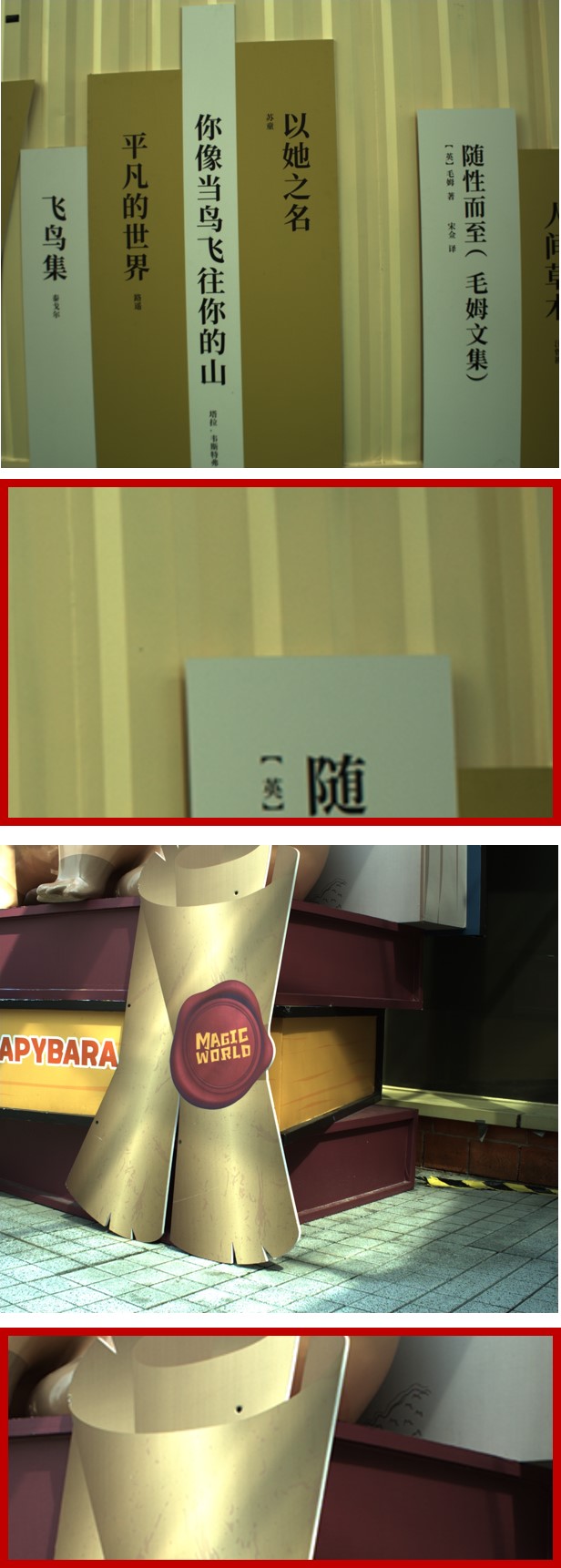}
        \subcaption*{GT}
    \end{subfigure}
    
    \caption{Visual comparison of polarization-based reflection separation on the PolaRGB dataset~\cite{cvpr/yao2025polarfree}. Compared to the recent method PolarFree~\cite{cvpr/yao2025polarfree}, our adapted DIRS framework effectively restores structural details in severely overexposed areas (highlighted in red boxes), yielding cleaner transmission predictions with fewer artifacts.}
    \label{fig:polar_comp}
    \vspace{-5pt}
\end{figure*}

\subsection{Polarized Image Reflection Separation}
\label{sec:polarized}
Although dual-stream interactive models effectively advance single-image reflection separation, they remain constrained by the limited scale of real-world datasets and the intrinsic ill-posedness of single-view observations. Specifically, when strong reflections overlap with transmission layers, the resulting overexposure causes irreversible information loss in a single image. To faithfully restore these regions, we adapt our DIRS design for polarization-based multiple-image reflection separation (PMIRS).

\begin{table}[t]
  \centering
  \caption{Quantitative comparison on the PolaRGB test set~\cite{cvpr/yao2025polarfree} for the PMIRS task. $\uparrow$ indicates higher is better, while $\downarrow$ indicates lower is better. The best results are in \textbf{bold}.}
  \renewcommand{\arraystretch}{1.2}
    \begin{tabular}{lccc}
    \toprule
    Method & PSNR $\uparrow$ & SSIM $\uparrow$ & LPIPS $\downarrow$ \\
    \midrule
    PolarFree~\cite{cvpr/yao2025polarfree} & 22.90 & 0.878 & 0.103 \\
    DIRS-PAIR (Ours) & \textbf{25.73} & \textbf{0.925} & \textbf{0.065} \\
    \bottomrule
    \end{tabular}%
  \label{tab:polar_metrics}%
\end{table}%

A recent dataset, PolaRGB~\cite{cvpr/yao2025polarfree}, utilizes a division-of-focal-plane polarization camera with a color Bayer pattern to capture aligned images at four polarization angles ($0^\circ, 45^\circ, 90^\circ$, and $135^\circ$), denoted as $\mathbf{I}_{0^\circ}, \mathbf{I}_{45^\circ}, \mathbf{I}_{90^\circ}$, and $\mathbf{I}_{135^\circ}$. According to Malus's law~\cite{born1965principles}, reflection intensity varies significantly across different polarization angles. This physical variation allows structural cues hidden within severely overexposed regions to become visible at certain angles, enabling more accurate decomposition. Standard PMIRS strategies typically rely on Stokes parameters~\cite{cvpr/LeiHZYSC20,cvpr/yao2025polarfree}:
\begin{equation}
\begin{aligned}
    \mathbf{S}_{0} &= (\mathbf{I}_{0^\circ} + \mathbf{I}_{45^\circ} + \mathbf{I}_{90^\circ} + \mathbf{I}_{135^\circ})/2, \\
    \mathbf{S}_{1} &= \mathbf{I}_{0^\circ} - \mathbf{I}_{90^\circ}, \quad \mathbf{S}_{2} = \mathbf{I}_{45^\circ} - \mathbf{I}_{135^\circ},
\end{aligned}
\end{equation}
where $\mathbf{S}_{0}$ represents the total unpolarized intensity, and $\mathbf{S}_{1}$ and $\mathbf{S}_{2}$ provide the linear polarization state based on intensity differences between orthogonal polarization directions. 

Our DIRS architecture naturally aligns with these polarization-based inputs. Following the DIRS design in the main paper, $\mathbf{S}_{0}$ captures the average observation and is ideal for global scene semantic extraction, whereas $\mathbf{S}_{1}$ and $\mathbf{S}_{2}$ serve as primary polarization cues. Therefore, we utilize the single-stream GPE to process $\mathbf{S}_{0}$ and the dual-stream APE to process the polarization components:
\begin{equation}
\begin{aligned}
    \{\mathbf{F}_{\mathbf{S}_0}^l\}_{l=1}^L &= \mathcal{F}_{\text{GPE}}(\mathbf{S}_{0}), \\
    \{\mathbf{F}_{\mathbf{S}_1}^l, \mathbf{F}_{\mathbf{S}_2}^l\}_{l=1}^L &= \mathcal{F}_{\text{APE}}([\mathbf{I}_{0^\circ}, \mathbf{I}_{90^\circ}], [\mathbf{I}_{45^\circ}, \mathbf{I}_{135^\circ}]),
\end{aligned}
\end{equation}
where $[\cdot, \cdot]$ denotes channel-wise concatenation. Notably, instead of explicitly inputting the differential parameters $\mathbf{S}_{1}$ and $\mathbf{S}_{2}$, we feed the concatenated raw image pairs to the APE. This strategy preserves richer original structural details for feature fusion, while the network retains the capacity to naturally learn the required differential mapping. 
As illustrated in Fig.~\ref{fig:polar_comp}, our adapted strategy effectively handles severe highlight overlaps, yielding cleaner reflection separation with minimal artifacts compared to existing methods. To quantitatively validate this extension, we compare our adapted model against the state-of-the-art PMIRS method, PolarFree~\cite{cvpr/yao2025polarfree}, on the PolaRGB test set. As reported in Tab.~\ref{tab:polar_metrics}, our DIRS-based architecture significantly outperforms PolarFree, achieving an improvement of nearly 3 dB in PSNR alongside substantially better structural similarity (SSIM) and perceptual quality (LPIPS). The faithful restoration of previously overexposed regions, both visually and quantitatively, decisively demonstrates the versatility and broad applicability of the proposed DIRS framework beyond standard single-image scenarios.

\begin{figure*}[t]
    \begin{subfigure}[b]{0.19\linewidth}
        \centering
        \includegraphics[width=\linewidth,height=100pt]{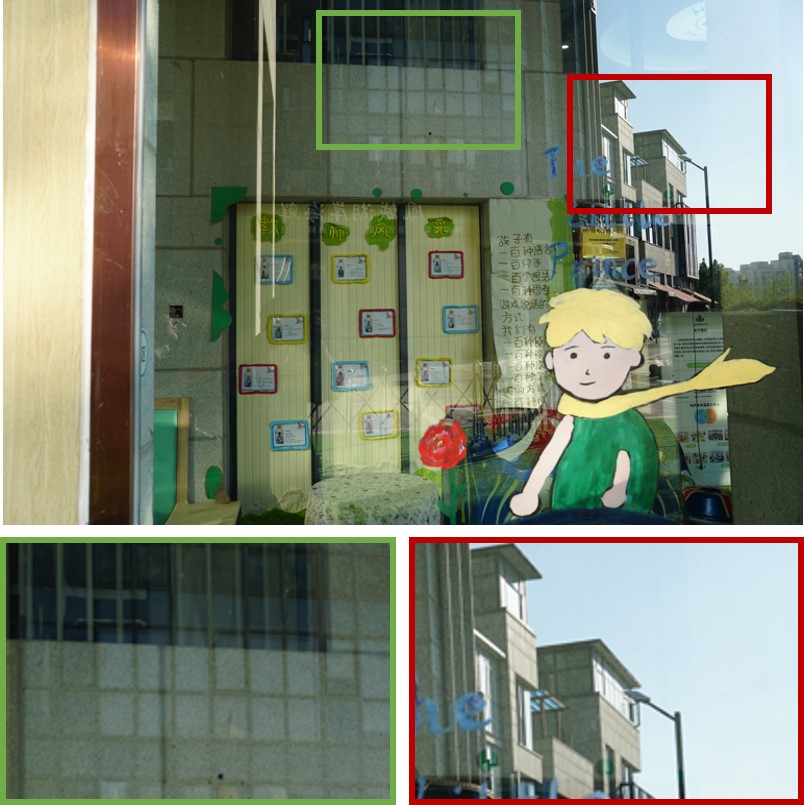}
        \caption*{Input}
    \end{subfigure}
    \begin{subfigure}[b]{0.19\linewidth}
        \centering
        \includegraphics[width=\linewidth,height=100pt]{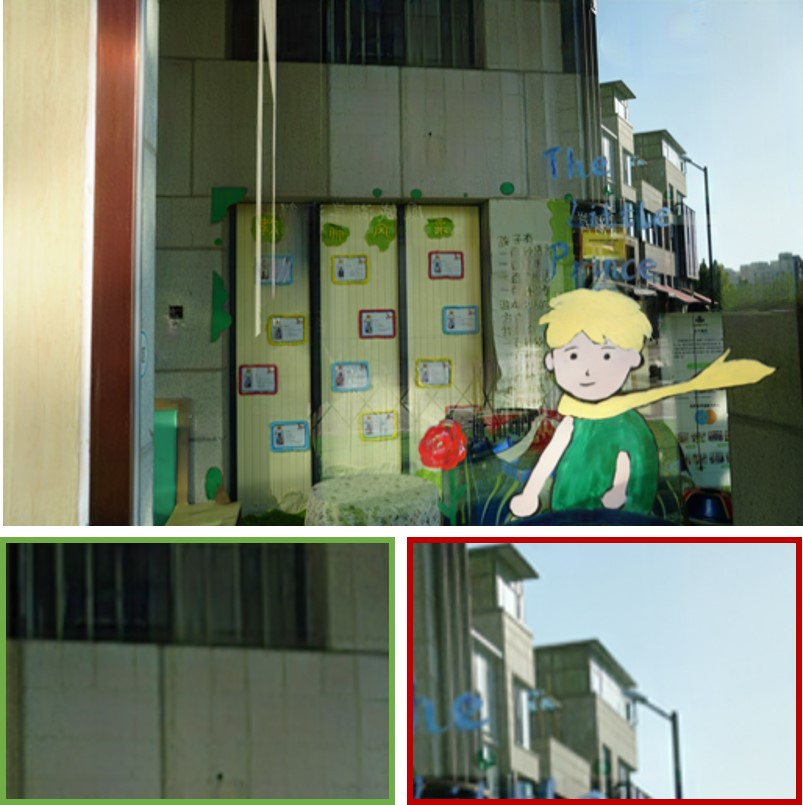}
        \caption*{Dong~\emph{et al.} \cite{iccv/Dong00BXL21}}
    \end{subfigure}
    \begin{subfigure}[b]{0.19\linewidth}
        \centering
        \includegraphics[width=\linewidth,height=100pt]{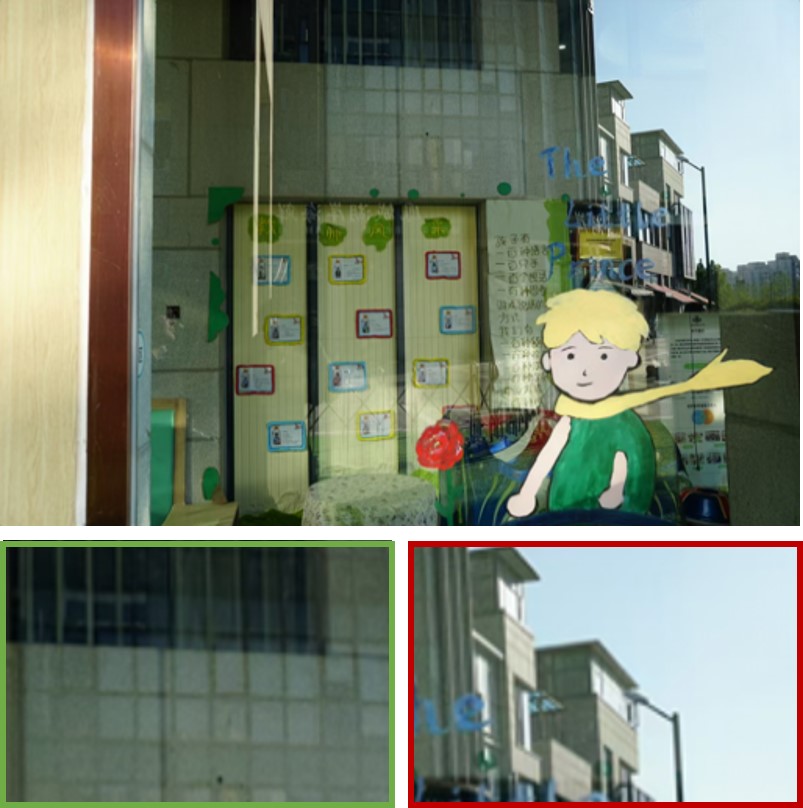}
        \caption*{RRW~\cite{cvpr/zhu2024revisiting}}
    \end{subfigure}
    \begin{subfigure}[b]{0.19\linewidth}
        \centering
        \includegraphics[width=\linewidth,height=100pt]{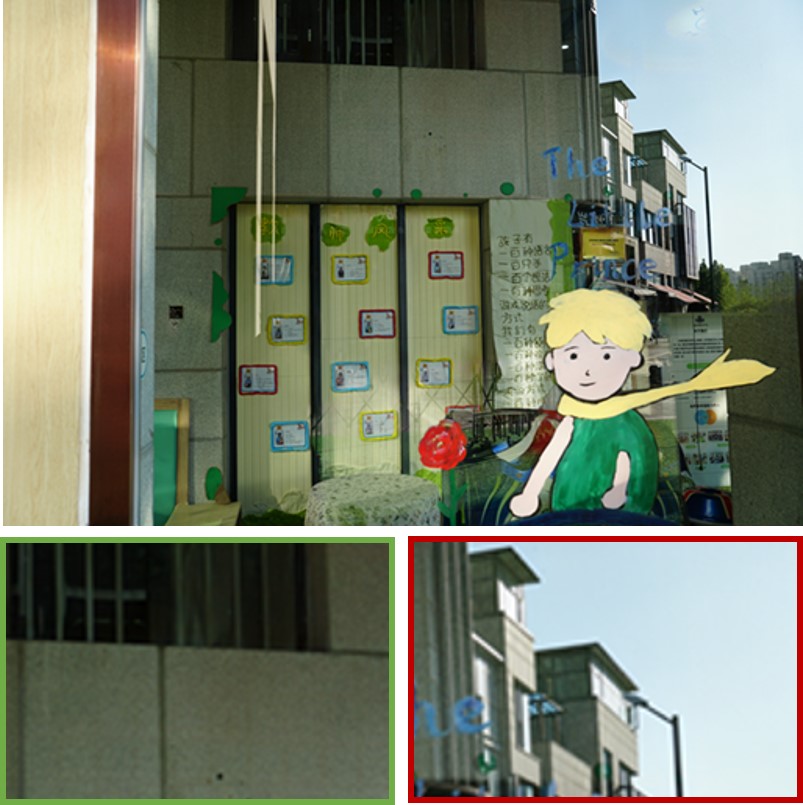}
        \caption*{RDNet~\cite{cvpr/zhao2025reversible}}
    \end{subfigure}
    \begin{subfigure}[b]{0.19\linewidth}
        \centering
        \includegraphics[width=\linewidth,height=100pt]{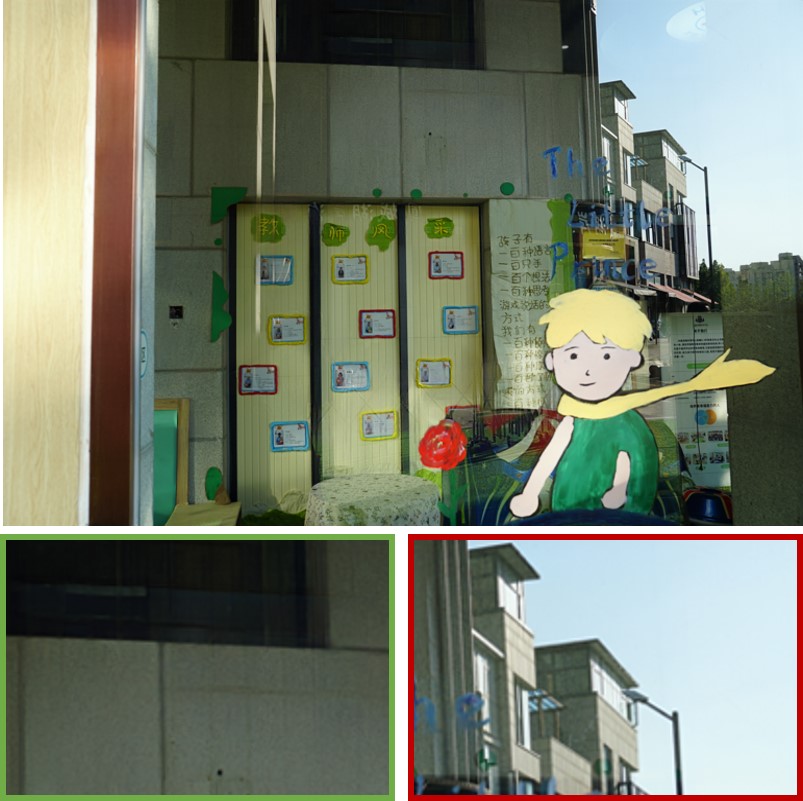}
        \caption*{DIRS-PAIR (Ours)}
    \end{subfigure}
 
\caption{Limitation analysis. While our method effectively disentangles reflection components under relatively mild coupling, all evaluated methods fail under extreme illumination contrast due to irreversible information loss.}
  \label{fig:limit}
  \vspace{-10pt}
\end{figure*}

\section{Limitations}
\label{sec:limitations}

While our proposed DIRS framework significantly advances the state-of-the-art in single-image reflection separation, it still faces inherent physical bottlenecks under extreme real-world conditions. Fig.~\ref{fig:limit} illustrates a challenging failure case characterized by spatially varying coupling intensities. As highlighted by the red boxes, in scenarios where a large illuminance disparity exists between the two sides of the glass, the reflection layer becomes highly prominent. This naturally reduces the transmission information retained in the captured image. Consequently, existing state of the art methods, including ours, struggle to recover the background details in these specific areas. Conversely, in regions with moderate coupling (highlighted by the green boxes), our method remains robust, successfully separating the complex grid reflection structures where other baseline models still fail to achieve satisfactory results.

This contrast underscores the fundamental ill-posedness of the single-image reflection separation task. When the underlying physical signal is truncated or entirely masked, purely algorithmic decoupling is insufficient. This highlights the necessity of exploring alternative strategies to tackle such extreme conditions. Promising future directions include incorporating explicit user guidance (such as manual annotations or language-guided prompts) and exploiting wider spatial context via panoramic-based methods, as discussed in the main manuscript. Furthermore, integrating auxiliary physical observations offers a practical alternative to mitigate this information loss. This encompasses the multiple image strategies surveyed in this document, including the polarization based reflection separation paradigm to which we successfully adapted our framework in Sec.~\ref{sec:extend_exps}.

\section{Conclusion}

In this paper, we present a principled framework that advances both the theoretical understanding and practical performance of reflection separation. We show that the widely adopted linear superposition assumption is insufficient under realistic sRGB image formation, and introduce a learnable nonlinear superposition formulation to capture inter-layer coupling and ISP-induced biases. Building upon this formulation, we develop a generalized dual-stream interactive framework that explicitly models bidirectional dependencies between transmission and reflection through unified feature interaction mechanisms.
Extensive experiments demonstrate that the proposed approach achieves state-of-the-art performance on diverse real-world benchmarks while maintaining strong generalization capability. Beyond empirical results, our study reveals that reflection separation is fundamentally a coupled decomposition problem, where accurate recovery requires the joint modeling of nonlinear image formation and feature interaction. We expect that this perspective will inspire future research toward more principled and physically grounded image decomposition frameworks.

\appendices

\section{Proposition Proofs}
\label{sec:append_proofs}
In this part, we provide the detailed proofs for the two propositions presented in the main manuscript. These proofs justify the intrinsic limitations of linear superposition models in the sRGB space and the necessity of interaction for non-linear reflection formation.

\subsection{Proof of Proposition 1}


\begin{proof}
Since the ISP mapping $\mathcal{F}_{\textup{ISP}}$ operates pixel-wise, we consider an arbitrary spatial location $\mathbf{p}$. Let the continuous scalar values $t = \mathbf{T}_\textup{raw}(\mathbf{p})$ and $r = \mathbf{R}_\textup{raw}(\mathbf{p})$ denote the independent RAW-domain intensities of the transmission and reflection layers at location $\mathbf{p}$, respectively. We prove the proposition by contradiction. 

Assume there exist fixed, scene-independent constants $a, b, z$ such that the linear relationship holds for all admissible continuous values of $t$ and $r$. This yields:
\begin{equation}
\mathcal{F}_{\textup{ISP}}(t + r) = a\mathcal{F}_{\textup{ISP}}(t) + b\mathcal{F}_{\textup{ISP}}(r) + z.
\label{eq:prop1_assumption}
\end{equation}

By applying the chain rule, we take the partial derivative of both sides of Eq.~(\ref{eq:prop1_assumption}) with respect to the transmission intensity $t$:
\begin{equation}
\frac{\partial}{\partial t}\mathcal{F}_{\textup{ISP}}(t + r) = \mathcal{F}_{\textup{ISP}}'(t + r) \cdot 1 = a\mathcal{F}_{\textup{ISP}}'(t).
\label{eq:prop1_dt}
\end{equation}
Similarly, taking the partial derivative with respect to the reflection intensity $r$ yields:
\begin{equation}
\frac{\partial}{\partial r}\mathcal{F}_{\textup{ISP}}(t + r) = \mathcal{F}_{\textup{ISP}}'(t + r) \cdot 1 = b\mathcal{F}_{\textup{ISP}}'(r).
\label{eq:prop1_dr}
\end{equation}

Since the left-hand sides of Eq.~(\ref{eq:prop1_dt}) and Eq.~(\ref{eq:prop1_dr}) are identically $\mathcal{F}_{\textup{ISP}}'(t + r)$, we can directly equate their right-hand sides to obtain:
\begin{equation}
a\mathcal{F}_{\textup{ISP}}'(t) = b\mathcal{F}_{\textup{ISP}}'(r),
\end{equation}
for all independent variables $t$ and $r$. For this equality to hold universally across all independent intensity inputs, the derivative function $\mathcal{F}_{\textup{ISP}}'(\cdot)$ must evaluate to a constant. 

If the first derivative is constant, the function $\mathcal{F}_{\textup{ISP}}$ must strictly be a linear mapping of the form $\mathcal{F}_{\textup{ISP}}(v) = c \cdot v + d$. However, this conclusion directly contradicts the fundamental premise that $\mathcal{F}_{\textup{ISP}}$ is a nonlinear operation.  Since assuming the existence of $a, b, z$ forces the ISP mapping to be linear, the initial assumption must be false. Therefore, no such fixed constants can exist, proving that the standard linear model $\mathbf{I} = a\mathbf{T} + b\mathbf{R} + z$ is fundamentally inadequate for sRGB superimposition.
\end{proof}

 \subsection{Proof of Proposition 2}


\begin{proof}
By definition, a multivariate interaction function $\boldsymbol{\Phi}(\mathbf{T}, \mathbf{R})$ is additively separable if and only if it can be decomposed into independent functions of each variable: $\boldsymbol{\Phi}(\mathbf{T}, \mathbf{R}) = \boldsymbol{\Phi}_\mathbf{T}(\mathbf{T}) + \boldsymbol{\Phi}_\mathbf{R}(\mathbf{R})$. We proceed by contradiction to prove that this separation is impossible under nonlinear real-world ISP constraints.

Let the independently rendered sRGB layers be $\mathbf{T} = \mathcal{F}(\mathbf{T}_{\text{raw}})$ and $\mathbf{R} = \mathcal{F}(\mathbf{R}_{\text{raw}})$. If the sRGB formation model were fully decoupled, we could substitute the separable assumption into the physical formation equation:
\begin{equation}
\begin{aligned}
    \mathcal{F}(\mathbf{T}_{\text{raw}} + \mathbf{R}_{\text{raw}}) &= \mathcal{F}(\mathbf{T}_{\text{raw}}) + \mathcal{F}(\mathbf{R}_{\text{raw}}) \\
    &\quad + \boldsymbol{\Phi}_\mathbf{T}(\mathcal{F}(\mathbf{T}_{\text{raw}})) + \boldsymbol{\Phi}_\mathbf{R}(\mathcal{F}(\mathbf{R}_{\text{raw}})) + \boldsymbol{\Psi}.
\end{aligned}
\label{eq:forward_decouple}
\end{equation}
To test this equality in the continuous and differentiable regions of the ISP pipeline, we take the mixed partial derivative $\frac{\partial^2}{\partial \mathbf{T}_{\text{raw}} \partial \mathbf{R}_{\text{raw}}}$ on both sides of Eq.~(\ref{eq:forward_decouple}), yielding:
\begin{equation}
    \mathcal{F}''(\mathbf{T}_{\text{raw}} + \mathbf{R}_{\text{raw}}) = 0.
\end{equation}
This result implies that the second derivative of the ISP function must be zero everywhere, meaning $\mathcal{F}(\cdot)$ can only be strictly linear. However, real-world ISP functions are inherently non-linear. For example, the standard Gamma compression $\mathcal{F}(x) = x^\gamma$ (typically $\gamma \approx 1/2.2$) yields a non-zero second derivative $\mathcal{F}''(x) = \gamma(\gamma-1)x^{\gamma-2} \neq 0$ for positive irradiances. Similarly, empirically calibrated Camera Response Functions (CRFs) also exhibit intrinsic convexity or concavity. Thus, the assumption leads to a mathematical contradiction, proving that $\boldsymbol{\Phi}(\mathbf{T}, \mathbf{R})$ strictly generates inseparable cross-terms and is therefore non-separable. 

Next, we establish that this physical non-separability mathematically necessitates coupled estimation. Consider our reconstruction objective $\mathcal{L}_{\text{rec}} = \|\mathbf{I}_{\text{obs}} - (\mathbf{T} + \mathbf{R} + \boldsymbol{\Phi}(\mathbf{T}, \mathbf{R}) + \boldsymbol{\Psi})\|_2^2$. The local curvature of the joint optimization landscape is governed by the Hessian matrix $\mathbf{H}$:
\begin{equation}
    \mathbf{H} = \begin{bmatrix} 
    \mathbf{H}_{\mathbf{T}\mathbf{T}} & \mathbf{H}_{\mathbf{T}\mathbf{R}} \\[6pt]
    \mathbf{H}_{\mathbf{R}\mathbf{T}} & \mathbf{H}_{\mathbf{R}\mathbf{R}} 
    \end{bmatrix} = \begin{bmatrix} 
    \frac{\partial^2 \mathcal{L}_{\text{rec}}}{\partial \mathbf{T}^2} & \frac{\partial^2 \mathcal{L}_{\text{rec}}}{\partial \mathbf{T} \partial \mathbf{R}} \\[6pt]
    \frac{\partial^2 \mathcal{L}_{\text{rec}}}{\partial \mathbf{R} \partial \mathbf{T}} & \frac{\partial^2 \mathcal{L}_{\text{rec}}}{\partial \mathbf{R}^2} 
    \end{bmatrix}.
\end{equation}
Because $\boldsymbol{\Phi}$ is non-separable, its mixed partial derivative dictates that the off-diagonal blocks are strictly non-zero ($\mathbf{H}_{\mathbf{T}\mathbf{R}} \neq \mathbf{0}$). In gradient-based optimization, this non-zero cross-curvature indicates that the gradient field of one layer dynamically depends on the state of the other, governed by the total differential $\Delta (\nabla_{\mathbf{T}} \mathcal{L}_{\text{rec}}) \approx \mathbf{H}_{\mathbf{T}\mathbf{R}} \Delta \mathbf{R}$. 
A separated estimator yet implicitly assumes $\mathbf{H}_{\mathbf{T}\mathbf{R}} = \mathbf{0}$, which contradicts the established physical formation. Therefore, converging to the true joint optimum requires the computational graph to evaluate this cross-variable curvature. In conclusion, explicitly modeling these cross-dependencies through dual-stream feature interaction is theoretically motivated for approximating the coupled inverse problem.
\end{proof}

\bibliographystyle{ieee_fullname}
\bibliography{egbib}

@misc{born1965principles,
  title={Principles of Optics},
  author={Born, Max and Wolf, Emil},
  year={1965},
  publisher={Pergamon Press Ltd., London}
}

@inproceedings{eccv/shizawa1992visual,
  title={On visual ambiguities due to transparency in motion and stereo},
  author={Shizawa, Masahiko},
  booktitle={ECCV},
  pages={411--419},
  year={1992}
}

@inproceedings{cvpr/qiu2023looking,
  title={Looking through the glass: Neural surface reconstruction against high specular reflections},
  author={Qiu, Jiaxiong and Jiang, Peng-Tao and Zhu, Yifan and Yin, Ze-Xin and Cheng, Ming-Ming and Ren, Bo},
  booktitle={CVPR},
  pages={20823--20833},
  year={2023}
}

@inproceedings{eccv/BergenBHP90,
  author       = {James R. Bergen and
                  Peter J. Burt and
                  Rajesh Hingorani and
                  Shmuel Peleg},
  title        = {Transparent-Motion Analysis},
  booktitle    = {ECCV},
  volume       = {427},
  pages        = {566--569},
  year         = {1990}
}

@inproceedings{iccv/SchechnerKB98,
  author       = {Yoav Y. Schechner and
                  Nahum Kiryati and
                  Ronen Basri},
  title        = {Separation of Transparent Layers Using Focus},
  booktitle    = {ICCV},
  pages        = {1061--1066},
  year         = {1998},
}

@inproceedings{cvpr/SzeliskiAA00,
  author    = {Richard Szeliski and
               Shai Avidan and
               P. Anandan},
  title     = {Layer Extraction from Multiple Images Containing Reflections and Transparency},
  booktitle = {CVPR},
  pages     = {1246},
  year      = {2000}
}

@article{pami/BergenBHP92,
  author       = {James R. Bergen and
                  Peter J. Burt and
                  Rajesh Hingorani and
                  Shmuel Peleg},
  title        = {A Three-Frame Algorithm for Estimating Two-Component Image Motion},
  journal      = {TPAMI},
  volume       = {14},
  number       = {9},
  pages        = {886--896},
  year         = {1992}
}

@inproceedings{cvpr/TsinKS03,
  author       = {Yanghai Tsin and
                  Sing Bing Kang and
                  Richard Szeliski},
  title        = {Stereo Matching with Reflections and Translucency},
  booktitle    = {CVPR},
  pages        = {702--709},
  year         = {2003}
}

@inproceedings{eccv/SarelI04,
  author    = {Bernard Sarel and
               Michal Irani},
  title     = {Separating Transparent Layers through Layer Information Exchange},
  booktitle = {ECCV},
  pages     = {328--341},
  year      = {2004}
}

@article{tog/AgrawalRNL05,
  author    = {Amit K. Agrawal and
               Ramesh Raskar and
               Shree K. Nayar and
               Yuanzhen Li},
  title     = {Removing photography artifacts using gradient projection and flash-exposure sampling},
  journal   = {TOG},
  volume    = {24},
  number    = {3},
  pages     = {828--835},
  year      = {2005}
}

@inproceedings{iccv/SarelI05,
  author    = {Bernard Sarel and
               Michal Irani},
  title     = {Separating Transparent Layers of Repetitive Dynamic Behaviors},
  booktitle = {ICCV},
  pages     = {26--32},
  year      = {2005}
}

@inproceedings{cvpr/GaiSZ08,
  author    = {Kun Gai and
               Zhenwei Shi and
               Changshui Zhang},
  title     = {Blindly separating mixtures of multiple layers with spatial shifts},
  booktitle = {CVPR},
  year      = {2008}
}

@article{pami/GaiSZ12,
  author    = {Kun Gai and
               Zhenwei Shi and
               Changshui Zhang},
  title     = {Blind Separation of Superimposed Moving Images Using Image Statistics},
  journal   = {TPAMI},
  volume    = {34},
  pages     = {19--32},
  year      = {2012}
}

@article{tog/SinhaKGSS12,
  author    = {Sudipta N. Sinha and
               Johannes Kopf and
               Michael Goesele and
               Daniel Scharstein and
               Richard Szeliski},
  title     = {Image-based rendering for scenes with reflections},
  journal   = {TOG},
  volume    = {31},
  number    = {4},
  pages     = {100:1--100:10},
  year      = {2012}
}

@inproceedings{iccv/LiB13,
  author    = {Yu Li and
               Michael S. Brown},
  title     = {Exploiting Reflection Change for Automatic Reflection Removal},
  booktitle = {ICCV},
  pages     = {2432--2439},
  year      = {2013}
}

@inproceedings{cvpr/GuoCM14,
  author    = {Xiaojie Guo and
               Xiaochun Cao and
               Yi Ma},
  title     = {Robust Separation of Reflection from Multiple Images},
  booktitle = {CVPR},
  pages     = {2195--2202},
  year      = {2014}
}

@article{tog/Freeman15,
  author    = {Tianfan Xue and
               Michael Rubinstein and
               Ce Liu and
               William T. Freeman},
  title     = {A computational approach for obstruction-free photography},
  journal   = {TOG},
  volume    = {34},
  pages     = {79:1--79:11},
  year      = {2015},
}

@inproceedings{cvpr/SimonP15,
  author    = {Christian Simon and
               In Kyu Park},
  title     = {Reflection removal for in-vehicle black box videos},
  booktitle = {CVPR},
  pages     = {4231--4239},
  year      = {2015}
}

@inproceedings{cvpr/YangLDT16,
  author    = {Jiaolong Yang and
               Hongdong Li and
               Yuchao Dai and
               Robby T. Tan},
  title     = {Robust Optical Flow Estimation of Double-Layer Images under Transparency or Reflection},
  booktitle = {CVPR},
  pages     = {1410--1419},
  year      = {2016}
}

@inproceedings{mm/SunLYZWL16,
  author    = {Chao Sun and
               Shuaicheng Liu and
               Taotao Yang and
               Bing Zeng and
               Zhengning Wang and
               Guanghui Liu},
  title     = {Automatic Reflection Removal using Gradient Intensity and Motion Cues},
  booktitle = {ACM MM},
  pages     = {466--470},
  year      = {2016}
}

@inproceedings{cvpr/HanS17,
  author    = {Byeong{-}Ju Han and
               Jae{-}Young Sim},
  title     = {Reflection Removal Using Low-Rank Matrix Completion},
  booktitle = {CVPR},
  pages     = {3872--3880},
  year      = {2017}
}

@article{tip/HanS18,
  author       = {Byeong{-}Ju Han and
                  Jae{-}Young Sim},
  title        = {Glass Reflection Removal Using Co-Saliency-Based Image Alignment and
                  Low-Rank Matrix Completion in Gradient Domain},
  journal      = {TIP},
  volume       = {27},
  number       = {10},
  pages        = {4873--4888},
  year         = {2018}
}

@inproceedings{cvpr/AlayracCZ19,
  author       = {Jean{-}Baptiste Alayrac and
                  Jo{\~{a}}o Carreira and
                  Andrew Zisserman},
  title        = {The Visual Centrifuge: Model-Free Layered Video Representations},
  booktitle    = {CVPR},
  pages        = {2457--2466},
  year         = {2019}
}

@inproceedings{cvpr/PunnappurathB19,
  author       = {Abhijith Punnappurath and
                  Michael S. Brown},
  title        = {Reflection Removal Using a Dual-Pixel Sensor},
  booktitle    = {CVPR},
  pages        = {1556--1565},
  year         = {2019},
}

@inproceedings{cvpr/LiuL0CH20,
  author    = {Yu{-}Lun Liu and
               Wei{-}Sheng Lai and
               Ming{-}Hsuan Yang and
               Yung{-}Yu Chuang and
               Jia{-}Bin Huang},
  title     = {Learning to See Through Obstructions},
  booktitle = {CVPR},
  pages     = {14203--14212},
  year      = {2020}
}

@inproceedings{cvpr/LeiC21,
  author    = {Chenyang Lei and
               Qifeng Chen},
  title     = {Robust Reflection Removal With Reflection-Free Flash-Only Cues},
  booktitle = {CVPR},
  pages     = {14811--14820},
  year      = {2021}
}

@article{pami/LiuLYCH22,
  author    = {Yu{-}Lun Liu and
               Wei{-}Sheng Lai and
               Ming{-}Hsuan Yang and
               Yung{-}Yu Chuang and
               Jia{-}Bin Huang},
  title     = {Learning to See Through Obstructions With Layered Decomposition},
  journal   = {TPAMI},
  volume    = {44},
  number    = {11},
  pages     = {8387--8402},
  year      = {2022}
}

@article{pami/LeiJC23,
  author       = {Chenyang Lei and
                  Xudong Jiang and
                  Qifeng Chen},
  title        = {Robust Reflection Removal With Flash-Only Cues in the Wild},
  journal      = {TPAMI},
  volume       = {45},
  number       = {12},
  pages        = {15530--15545},
  year         = {2023}
}

@inproceedings{cvpr/Wolff89a,
  author       = {Lawrence B. Wolff},
  title        = {Using polarization to separate reflection components},
  booktitle    = {CVPR},
  pages        = {363--369},
  year         = {1989}
}

@article{ijcv/NayarFB97,
  author    = {Shree K. Nayar and
               Xi{-}Sheng Fang and
               Terrance E. Boult},
  title     = {Separation of Reflection Components Using Color and Polarization},
  journal   = {IJCV},
  volume    = {21},
  number    = {3},
  pages     = {163--186},
  year      = {1997}
}

@inproceedings{cvpr/FaridA99,
  author    = {Hany Farid and
               Edward H. Adelson},
  title     = {Separating Reflections and Lighting Using Independent Components Analysis},
  booktitle = {CVPR},
  pages     = {1262--1267},
  year      = {1999},
}

@inproceedings{schechner1999polarization,
  title={Polarization-based decorrelation of transparent layers: The inclination angle of an invisible surface},
  author={Schechner, Yoav Y and Shamir, Joseph and Kiryati, Nahum},
  booktitle={ICCV},
  pages={814--819},
  year={1999}
}

@article{schechner2000polarization,
  title={Polarization and statistical analysis of scenes containing a semireflector},
  author={Schechner, Yoav Y and Shamir, Joseph and Kiryati, Nahum},
  journal={JOSA A},
  volume={17},
  number={2},
  pages={276--284},
  year={2000}
}

@article{tip/KongTS11,
  author    = {Naejin Kong and
               Yu{-}Wing Tai and
               Sung Yong Shin},
  title     = {High-Quality Reflection Separation Using Polarized Images},
  journal   = {TIP},
  volume    = {20},
  number    = {12},
  pages     = {3393--3405},
  year      = {2011}
}

@article{pami/KongTS14,
  author    = {Naejin Kong and
               Yu{-}Wing Tai and
               Joseph S. Shin},
  title     = {A Physically-Based Approach to Reflection Separation: From Physical Modeling to Constrained Optimization},
  journal   = {TPAMI},
  volume    = {36},
  number    = {2},
  pages     = {209--221},
  year      = {2014}
}

@inproceedings{eccv/WieschollekGGK18,
  author    = {Patrick Wieschollek and
               Orazio Gallo and
               Jinwei Gu and
               Jan Kautz},
  title     = {Separating Reflection and Transmission Images in the Wild},
  booktitle = {ECCV},
  pages     = {90--105},
  year      = {2018}
}

@inproceedings{nips/LyuCLPS19,
  author    = {Youwei Lyu and
               Zhaopeng Cui and
               Si Li and
               Marc Pollefeys and
               Boxin Shi},
  title     = {Reflection Separation using a Pair of Unpolarized and Polarized Images},
  booktitle = {NeurIPS},
  pages     = {14532--14542},
  year      = {2019},
}

@inproceedings{eccv/LiQZH20,
  author    = {Rui Li and
               Simeng Qiu and
               Guangming Zang and
               Wolfgang Heidrich},
  title     = {Reflection Separation via Multi-bounce Polarization State Tracing},
  booktitle = {ECCV},
  pages     = {781--796},
  year      = {2020},
}

@inproceedings{cvpr/LeiHZYSC20,
  author    = {Chenyang Lei and
               Xuhua Huang and
               Mengdi Zhang and
               Qiong Yan and
               Wenxiu Sun and
               Qifeng Chen},
  title     = {Polarized Reflection Removal With Perfect Alignment in the Wild},
  booktitle = {CVPR},
  pages     = {1747--1755},
  year      = {2020},
}

@inproceedings{nips/LevinZW02,
  author    = {Anat Levin and
               Assaf Zomet and
               Yair Weiss},
  title     = {Learning to Perceive Transparency from the Statistics of Natural Scenes},
  booktitle = {NeurIPS},
  pages     = {1247--1254},
  year      = {2002},
}

@inproceedings{cvpr/LevinZW04,
  author    = {Anat Levin and
               Assaf Zomet and
               Yair Weiss},
  title     = {Separating Reflections from a Single Image Using Local Features},
  booktitle = {CVPR},
  pages     = {306--313},
  year      = {2004},
}

@inproceedings{eccv/LevinW04,
  title={User assisted separation of reflections from a single image using a sparsity prior},
  author={Levin, Anat and Weiss, Yair},
  booktitle={ECCV},
  pages={602--613},
  year={2004}
}

@article{pami/LevinW07,
  author    = {Anat Levin and
               Yair Weiss},
  title     = {User Assisted Separation of Reflections from a Single Image Using a Sparsity Prior},
  journal   = {TPAMI},
  volume    = {29},
  number    = {9},
  pages     = {1647--1654},
  year      = {2007}
}

@inproceedings{wacv/ChungCWC09,
  author       = {Yun{-}Chung Chung and
                  Shyang{-}Lih Chang and
                  Jung Ming Wang and
                  Sei{-}Wang Chen},
  title        = {Interference reflection separation from a single image},
  booktitle    = {WACV},
  pages        = {1--6},
  year         = {2009}
}

@inproceedings{iscas/YanXY13,
  author       = {Qing Yan and
                  Yi Xu and
                  Xiaokang Yang},
  title        = {Separation of weak reflection from a single superimposed image using gradient profile sharpness},
  booktitle    = {ISCAS},
  pages        = {937--940},
  year         = {2013}
}

@inproceedings{cvpr/LiB14,
  author    = {Yu Li and
               Michael S. Brown},
  title     = {Single Image Layer Separation Using Relative Smoothness},
  booktitle = {CVPR},
  pages     = {2752--2759},
  year      = {2014}
}

@inproceedings{cvpr/ShihKDF15,
  author       = {Yichang Shih and
                  Dilip Krishnan and
                  Fr{\'{e}}do Durand and
                  William T. Freeman},
  title        = {Reflection removal using ghosting cues},
  booktitle    = {CVPR},
  pages        = {3193--3201},
  year         = {2015}
}

@inproceedings{icip/WanSHK16,
  author       = {Renjie Wan and
                  Boxin Shi and
                  Ah{-}Hwee Tan and
                  Alex C. Kot},
  title        = {Depth of field guided reflection removal},
  booktitle    = {ICIP},
  pages        = {21--25},
  year         = {2016}
}

@inproceedings{iccv/WanSDTK17,
  author    = {Renjie Wan and
               Boxin Shi and
               Ling{-}Yu Duan and
               Ah{-}Hwee Tan and
               Alex C. Kot},
  title     = {Benchmarking Single-Image Reflection Removal Algorithms},
  booktitle = {ICCV},
  pages     = {3942--3950},
  year      = {2017}
}

@inproceedings{iccv/FanYHCW17,
  author    = {Qingnan Fan and
               Jiaolong Yang and
               Gang Hua and
               Baoquan Chen and
               David P. Wipf},
  title     = {A Generic Deep Architecture for Single Image Reflection Removal and
               Image Smoothing},
  booktitle = {ICCV},
  pages     = {3258--3267},
  year      = {2017}
}

@inproceedings{cvpr/ZhangNC18a,
  author    = {Xuaner Cecilia Zhang and
               Ren Ng and
               Qifeng Chen},
  title     = {Single Image Reflection Separation With Perceptual Losses},
  booktitle = {CVPR},
  pages     = {4786--4794},
  year      = {2018}
}

@inproceedings{cvpr/WanSDTK18,
  author    = {Renjie Wan and
               Boxin Shi and
               Ling{-}Yu Duan and
               Ah{-}Hwee Tan and
               Alex C. Kot},
  title     = {CRRN: Multi-Scale Guided Concurrent Reflection Removal Network},
  booktitle = {CVPR},
  pages     = {4777--4785},
  year      = {2018}
}

@inproceedings{eccv/YangGLS18,
  author    = {Jie Yang and
               Dong Gong and
               Lingqiao Liu and
               Qinfeng Shi},
  title     = {Seeing Deeply and Bidirectionally: A Deep Learning Approach for Single Image Reflection Removal},
  booktitle = {ECCV},
  pages     = {675--691},
  year      = {2018},
}

@inproceedings{cvpr/WenT0LHH19,
  author    = {Qiang Wen and
               Yinjie Tan and
               Jing Qin and
               Wenxi Liu and
               Guoqiang Han and
               Shengfeng He},
  title     = {Single Image Reflection Removal Beyond Linearity},
  booktitle = {CVPR},
  pages     = {3771--3779},
  year      = {2019}
}

@inproceedings{cvpr/WeiYFW019,
  author    = {Kaixuan Wei and
               Jiaolong Yang and
               Ying Fu and
               David P. Wipf and
               Hua Huang},
  title     = {Single Image Reflection Removal Exploiting Misaligned Training Data and Network Enhancements},
  booktitle = {CVPR},
  pages     = {8178--8187},
  year      = {2019},
}

@inproceedings{cvpr/LiY0LH20,
  author    = {Chao Li and
               Yixiao Yang and
               Kun He and
               Stephen Lin and
               John E. Hopcroft},
  title     = {Single Image Reflection Removal Through Cascaded Refinement},
  booktitle = {CVPR},
  pages     = {3562--3571},
  year      = {2020}
}

@article{li2023two,
  title={Two-stage single image reflection removal with reflection-aware guidance},
  author={Li, Yu and Liu, Ming and Yi, Yaling and Li, Qince and Ren, Dongwei and Zuo, Wangmeng},
  journal={Applied Intelligence},
  volume={53},
  number={16},
  pages={19433--19448},
  year={2023}
}

@inproceedings{cvpr/WanSLDK20,
  author    = {Renjie Wan and
               Boxin Shi and
               Haoliang Li and
               Ling{-}Yu Duan and
               Alex C. Kot},
  title     = {Reflection Scene Separation From a Single Image},
  booktitle = {CVPR},
  pages     = {2395--2403},
  year      = {2020},
}

@article{tip/FengPJCZL21,
  author    = {Xin Feng and
               Wenjie Pei and
               Zihui Jia and
               Fanglin Chen and
               David Zhang and
               Guangming Lu},
  title     = {Deep-Masking Generative Network: A Unified Framework for Background Restoration From Superimposed Images},
  journal   = {TIP},
  volume    = {30},
  pages     = {4867--4882},
  year      = {2021}
}

@inproceedings{cvpr/HongZZJKS21,
  author       = {Yuchen Hong and
                  Qian Zheng and
                  Lingran Zhao and
                  Xudong Jiang and
                  Alex C. Kot and
                  Boxin Shi},
  title        = {Panoramic Image Reflection Removal},
  booktitle    = {CVPR},
  pages        = {7762--7771},
  year         = {2021}
}

@inproceedings{icme/FengJJP0L21,
  author       = {Xin Feng and
                  Haobo Ji and
                  Bo Jiang and
                  Wenjie Pei and
                  Fanglin Chen and
                  Guangming Lu},
  title        = {Contrastive Feature Decomposition for Image Reflection Removal},
  booktitle    = {ICME},
  pages        = {1--6},
  year         = {2021}
}

@inproceedings{mma/shao2021model,
  title={A Model-Guided Unfolding Network for Single Image Reflection Removal},
  author={Shao, Dongliang and Shi, Yunhui and Wang, Jin and Ling, Nam and Yin, Baocai},
  booktitle={ACM MM Asia},
  pages={1--7},
  year={2021}
}

@inproceedings{nips/HuG21,
  author    = {Qiming Hu and
               Xiaojie Guo},
  title     = {Trash or Treasure? An Interactive Dual-Stream Strategy for Single Image Reflection Separation},
  booktitle = {NeurIPS},
  pages     = {24683--24694},
  year      = {2021}
}

@inproceedings{cvpr/ZhengSCJDK21,
  author    = {Qian Zheng and
               Boxin Shi and
               Jinnan Chen and
               Xudong Jiang and
               Ling{-}Yu Duan and
               Alex C. Kot},
  title     = {Single Image Reflection Removal With Absorption Effect},
  booktitle = {CVPR},
  pages     = {13395--13404},
  year      = {2021}
}

@inproceedings{iccv/Dong00BXL21,
  author    = {Zheng Dong and
               Ke Xu and
               Yin Yang and
               Hujun Bao and
               Weiwei Xu and
               Rynson W. H. Lau},
  title     = {Location-aware Single Image Reflection Removal},
  booktitle = {ICCV},
  pages     = {4997--5006},
  year      = {2021}
}

@inproceedings{mm/ZhangSL22,
  author       = {Ya{-}Nan Zhang and
                  Linlin Shen and
                  Qiufu Li},
  title        = {Content and Gradient Model-driven Deep Network for Single Image Reflection Removal},
  booktitle    = {ACM MM},
  pages        = {6802--6812},
  year         = {2022},
}

@inproceedings{cvpr/song2023robust,
  title={Robust single image reflection removal against adversarial attacks},
  author={Song, Zhenbo and Zhang, Zhenyuan and Zhang, Kaihao and Luo, Wenhan and Fan, Zhaoxin and Ren, Wenqi and Lu, Jianfeng},
  booktitle={CVPR},
  pages={24688--24698},
  year={2023}
}

@inproceedings{iccv/Hu23,
  title={Single image reflection separation via component synergy},
  author={Hu, Qiming and Guo, Xiaojie},
  booktitle={ICCV},
  pages={13138--13147},
  year={2023}
}

@article{pami/HongZZJKS23,
  author       = {Yuchen Hong and
                  Qian Zheng and
                  Lingran Zhao and
                  Xudong Jiang and
                  Alex C. Kot and
                  Boxin Shi},
  title        = {PAR$^2$Net: End-to-End Panoramic Image Reflection Removal},
  journal      = {TPAMI},
  volume       = {45},
  number       = {10},
  pages        = {12192--12205},
  year         = {2023}
}

@inproceedings{mm/wang2023personalized,
  title={Personalized single image reflection removal network through adaptive cascade refinement},
  author={Wang, Mengyi and Zhang, Xinxin and Gong, Yongshun and Yin, Yilong},
  booktitle={ACM MM},
  pages={8204--8213},
  year={2023}
}

@inproceedings{cvpr/zhu2024revisiting,
  title={Revisiting Single Image Reflection Removal In the Wild},
  author={Zhu, Yurui and Fu, Xueyang and Jiang, Peng-Tao and Zhang, Hao and Sun, Qibin and Chen, Jinwei and Zha, Zheng-Jun and Li, Bo},
  booktitle={CVPR},
  year={2024}
}

@article{chen2024closer,
  title={A Closer Look at the Reflection Formulation in Single Image Reflection Removal},
  author={Chen, Zhikai and Long, Fuchen and Qiu, Zhaofan and Zhang, Juyong and Zha, Zheng-Jun and Yao, Ting and Luo, Jiebo},
  journal={TIP},
  year={2024}
}

@article{ijcv/hong2024light,
  title={Light flickering guided reflection removal},
  author={Hong, Yuchen and Chang, Yakun and Liang, Jinxiu and Ma, Lei and Huang, Tiejun and Shi, Boxin},
  journal={IJCV},
  volume={132},
  number={9},
  pages={3933--3953},
  year={2024}
}

@inproceedings{cvpr/ZhongHWLS24,
  author       = {Haofeng Zhong and
                  Yuchen Hong and
                  Shuchen Weng and
                  Jinxiu Liang and
                  Boxin Shi},
  title        = {Language-guided Image Reflection Separation},
  booktitle      = {CVPR},
  year         = {2024}
}

@article{nips/hu2024single,
  title={Single image reflection separation via dual-stream interactive transformers},
  author={Hu, Qiming and Wang, Hainuo and Guo, Xiaojie},
  journal={NeurIPS},
  volume={37},
  pages={55228--55248},
  year={2024}
}

@inproceedings{eccv/hong2024differ,
  title={L-DiffER: Single Image Reflection Removal with Language-Based Diffusion Model},
  author={Hong, Yuchen and Zhong, Haofeng and Weng, Shuchen and Liang, Jinxiu and Shi, Boxin},
  booktitle={ECCV},
  pages={58--76},
  year={2024}
}

@article{pami/huang2025lightweight,
  author={Huang, Jun-Jie and Liu, Tianrui and Chen, Zihan and Liu, Xinwang and Wang, Meng and Dragotti, Pier Luigi},
  journal={TPAMI}, 
  title={A Lightweight Deep Exclusion Unfolding Network for Single Image Reflection Removal}, 
  year={2025},
  volume={47},
  number={6},
  pages={4957-4973}
 }

@inproceedings{cvpr/yao2025polarfree,
  title={PolarFree: Polarization-based Reflection-Free Imaging},
  author={Yao, Mingde and Wang, Menglu and Tam, King-Man and Li, Lingen and Xue, Tianfan and Gu, Jinwei},
  booktitle={CVPR},
  pages={10890--10899},
  year={2025}
}

@inproceedings{cvpr/kee2025removing,
  title={Removing reflections from raw photos},
  author={Kee, Eric and Pikielny, Adam and Blackburn-Matzen, Kevin and Levoy, Marc},
  booktitle={CVPR},
  pages={161--171},
  year={2025}
}

@inproceedings{cvpr/zhao2025reversible,
  title={Reversible decoupling network for single image reflection removal},
  author={Zhao, Hao and Li, Mingjia and Hu, Qiming and Guo, Xiaojie},
  booktitle={CVPR},
  pages={26430--26439},
  year={2025}
}

@inproceedings{aaai/chen2025firm,
  title={FIRM: Flexible Interactive Reflection ReMoval},
  author={Chen, Xiao and Jiang, Xudong and Tao, Yunkang and Lei, Zhen and Li, Qing and Lei, Chenyang and Zhang, Zhaoxiang},
  booktitle={AAAI},
  volume={39},
  number={2},
  pages={2230--2238},
  year={2025}
}

@inproceedings{aaai/hu2026dereflection,
  title={Dereflection any image with diffusion priors and diversified data},
  author={Hu, Jichen and Yang, Chen and Zhou, Zanwei and Fang, Jiemin and Tian, Qi and Shen, Wei},
  booktitle={AAAI},
  volume={40},
  number={6},
  pages={4860--4868},
  year={2026}
}

@article{arxiv/howard2017mobilenets,
  title={Mobilenets: Efficient convolutional neural networks for mobile vision applications},
  author={Howard, Andrew G and Zhu, Menglong and Chen, Bo and Kalenichenko, Dmitry and Wang, Weijun and Weyand, Tobias and Andreetto, Marco and Adam, Hartwig},
  journal={arXiv preprint},
  year={2017}
}

@inproceedings{dauphin2017language,
  title={Language modeling with gated convolutional networks},
  author={Dauphin, Yann N and Fan, Angela and Auli, Michael and Grangier, David},
  booktitle={ICML},
  pages={933--941},
  year={2017}
}

@inproceedings{corr/BahdanauCB14,
  author       = {Dzmitry Bahdanau and
                  Kyunghyun Cho and
                  Yoshua Bengio},
  title        = {Neural Machine Translation by Jointly Learning to Align and Translate},
  booktitle    = {ICLR},
  year         = {2015}
}

@inproceedings{cvpr/Liu0LYXWN000WG22,
  author       = {Ze Liu and
                  Han Hu and
                  Yutong Lin and
                  Zhuliang Yao and
                  Zhenda Xie and
                  Yixuan Wei and
                  Jia Ning and
                  Yue Cao and
                  Zheng Zhang and
                  Li Dong and
                  Furu Wei and
                  Baining Guo},
  title        = {Swin Transformer V2: Scaling Up Capacity and Resolution},
  booktitle    = {CVPR},
  pages        = {11999--12009},
  year         = {2022}
}

@inproceedings{iccv/ZongS023,
  author       = {Zhuofan Zong and
                  Guanglu Song and
                  Yu Liu},
  title        = {DETRs with Collaborative Hybrid Assignments Training},
  booktitle    = {ICCV},
  pages        = {6725--6735},
  year         = {2023}
}

@inproceedings{nips/XieWYAAL21,
  author       = {Enze Xie and
                  Wenhai Wang and
                  Zhiding Yu and
                  Anima Anandkumar and
                  Jos{\'{e}} M. {\'{A}}lvarez and
                  Ping Luo},
  title        = {SegFormer: Simple and Efficient Design for Semantic Segmentation with
                  Transformers},
  booktitle    = {NeurIPS},
  pages        = {12077--12090},
  year         = {2021}
}

@article{ijcv/EveringhamGWWZ10,
  author    = {Mark Everingham and
               Luc Van Gool and
               Christopher K. I. Williams and
               John M. Winn and
               Andrew Zisserman},
  title     = {The Pascal Visual Object Classes {(VOC)} Challenge},
  journal   = {IJCV},
  volume    = {88},
  number    = {2},
  pages     = {303--338},
  year      = {2010}
}

@inproceedings{iccv/shao2019objects365,
  title={Objects365: A large-scale, high-quality dataset for object detection},
  author={Shao, Shuai and Li, Zeming and Zhang, Tianyuan and Peng, Chao and Yu, Gang and Zhang, Xiangyu and Li, Jing and Sun, Jian},
  booktitle={ICCV},
  pages={8430--8439},
  year={2019}
}

@inproceedings{icmlw/maas2013rectifier,
  title={Rectifier nonlinearities improve neural network acoustic models},
  author={Maas, Andrew L and Hannun, Awni Y and Ng, Andrew Y and others},
  booktitle={ICML Workshop},
  volume={30},
  number={1},
  pages={3},
  year={2013}
}

@inproceedings{iccv/ke2021musiq,
  title={Musiq: Multi-scale image quality transformer},
  author={Ke, Junjie and Wang, Qifei and Wang, Yilin and Milanfar, Peyman and Yang, Feng},
  booktitle={ICCV},
  pages={5148--5157},
  year={2021}
}

@inproceedings{aaai/wang2023exploring,
  title={Exploring clip for assessing the look and feel of images},
  author={Wang, Jianyi and Chan, Kelvin CK and Loy, Chen Change},
  booktitle={AAAI},
  volume={37},
  number={2},
  pages={2555--2563},
  year={2023}
}

@article{tip/chen2024topiq,
  title={Topiq: A top-down approach from semantics to distortions for image quality assessment},
  author={Chen, Chaofeng and Mo, Jiadi and Hou, Jingwen and Wu, Haoning and Liao, Liang and Sun, Wenxiu and Yan, Qiong and Lin, Weisi},
  journal={TIP},
  volume={33},
  pages={2404--2418},
  year={2024}
}

\end{document}